\documentclass{article}

\usepackage{natbib}
\setcitestyle{authoryear,round,citesep={;},aysep={,},yysep={;}}
\usepackage[letterpaper,top=2cm,bottom=1.5cm,left=1.5cm,right=2cm,marginparwidth=1.75cm]{geometry}
\usepackage{parskip}

\usepackage{xr}
\usepackage{amsmath}
\usepackage{mathtools}
\usepackage{amssymb,mathrsfs} 
\usepackage{amssymb}
\usepackage{amsthm}
\usepackage{amsfonts}
\usepackage{upgreek}
\usepackage{nicefrac}
\usepackage{authblk}
\usepackage{graphicx}
 \usepackage{color}
 \usepackage{hyperref}
\hypersetup{bookmarksdepth=subsection}
\usepackage{stmaryrd}
\DeclareMathAlphabet{\mathpzc}{OT1}{pzc}{m}{it}
\usepackage[inline]{enumitem}
\usepackage{url}
\usepackage{graphicx} 
\usepackage{tikz} 
\usepackage{xcolor}
\usepackage{bbm,bm}

\usepackage{ifthen}
\usepackage{xargs}
\usepackage{subfigure}
\usepackage{caption}

\usepackage{wrapfig}

\usepackage{multirow}

\usepackage{booktabs}       
\usepackage{titlesec}

\usepackage{graphicx}

\usepackage{todonotes}

\usepackage{amsopn}

\usepackage{aliascnt}

\usepackage{autonum}

\usepackage{cleveref}
\usepackage{algorithm}
\usepackage{algpseudocode}
\usepackage{dsfont}

\newtheorem{theorem}{Theorem}[subsection]
\newtheorem{remark}{Remark}[subsection]

\newaliascnt{corollary}{theorem}

\aliascntresetthe{corollary}

\newaliascnt{proposition}{theorem}
\newtheorem{proposition}[proposition]{Proposition}
\aliascntresetthe{proposition}

\newaliascnt{lemma}{theorem}
\newtheorem{lemma}[lemma]{Lemma}
\aliascntresetthe{lemma}

\crefname{theorem}{theorem}{theorems}
\Crefname{theorem}{Theorem}{Theorems}
\crefname{remark}{remark}{remarks}
\Crefname{remark}{Remark}{Remarks}

\crefname{proposition}{proposition}{propositions}
\Crefname{proposition}{Proposition}{Propositions}

\crefname{corollary}{corollary}{corollaries}
\Crefname{corollary}{Corollary}{Corollaries}

\crefname{example}{example}{examples}
\Crefname{example}{Example}{Examples}

\crefname{lemma}{lemma}{lemmas}
\Crefname{lemma}{Lemma}{Lemmas}

\crefname{figure}{figure}{figures}
\Crefname{figure}{Figure}{Figures}

\newtheorem{assumption}{\textbf{H}\hspace{-3pt}}
\Crefname{assumption}{\textbf{H}\hspace{-3pt}}{\textbf{H}\hspace{-3pt}}
\crefname{assumption}{\textbf{H}}{\textbf{H}}

\newtheorem{assumptionRW}{\textbf{RW}\hspace{-3pt}}
\Crefname{assumptionRW}{\textbf{RW}\hspace{-3pt}}{\textbf{RW}\hspace{-3pt}}
\crefname{assumptionRW}{\textbf{RW}}{\textbf{RW}}

\newtheorem{assumptionM}{\textbf{M}\hspace{-3pt}}
\Crefname{assumptionM}{\textbf{M}\hspace{-3pt}}{\textbf{M}\hspace{-3pt}}
\crefname{assumptionM}{\textbf{M}}{\textbf{M}}

\newtheorem{assumptionBRW}{\textbf{BRW}\hspace{-3pt}}
\Crefname{assumptionBRW}{\textbf{BRW}\hspace{-3pt}}{\textbf{BRW}\hspace{-3pt}}
\crefname{assumptionBRW}{\textbf{BRW}}{\textbf{BRW}}

\Crefname{assumptionG}{\textbf{G}\hspace{-3pt}}{\textbf{G}\hspace{-3pt}}
\crefname{assumptionG}{\textbf{G}}{\textbf{G}}

\Crefname{assumptionAp}{\textbf{A'}\hspace{-3pt}}{\textbf{A'}\hspace{-3pt}}
\crefname{assumptionAp}{\textbf{A'}}{\textbf{A'}}



\hypersetup{
  colorlinks = true,
  citecolor = blue,
}

\usepackage{version}
\excludeversion{commenta}

\makeatletter
\newcommand*{\addFileDependency}[1]{
  \typeout{(#1)}
  \@addtofilelist{#1}
  \IfFileExists{#1}{}{\typeout{No file #1.}}
}
\makeatother

\def \N{\mathbb{N}}
\def \R{\mathbb{R}}

\def \Z{\mathbb{Z}}
\def \E{\mathbb{E}}

\def \f{\mathbf{f}}

\def \P{\mathbb{P}}

\def \D{\mathbb{D}}

\def \0{\mathds{O}}

\def \Fc{\mathcal{F}}

\def \Ic{\mathcal{I}}

\def \Mc{\mathcal{M}}

\def \Pc{\mathcal{P}}

\def\0{\mathds{O}}

\def \l{\left}
\def \r{\right}
\def\uBRWtheta{u^{\mathrm{BRW},\theta}}

\newcommand{\eqdef}{:=}

\newcommand{\mustar}{\mu^{\star}}

\def\msa{\mathsf{A}}

\def\msb{\mathsf{B}}

\def\mse{\mathsf{E}}

\def\msm{\mathsf{M}}

\def\msx{\mathsf{X}}

\def\mcs{\mathcal{S}}

\def\mcx{\mathcal{X}}
\def\mce{\mathcal{E}}
\def\mcf{\mathcal{F}}

\def\rset{\mathbb{R}}

\def\zset{\mathbb{Z}}

\def\mrl{\mathrm{L}}

\def\rmd{\mathrm{d}}

\def\rme{\mathrm{e}}
\def\rmn{\mathrm{n}}

\def\rmC{\mathrm{C}}

\newcommand{\1}{\mathbbm{1}}

\newcommand{\LeftEqNo}{\let\veqno\@@leqno}

\def\ie{\textit{i.e.}}

\def\eqsp{\;}
\newcommand{\coint}[1]{\left[#1\right)}

\newcommand{\ccint}[1]{\left[#1\right]}

\def\as{\ensuremath{\text{a.s.}}}

\def\eg{e.g.}

\newcommand{\opnorm}[1]{{\left\vert\kern-0.25ex\left\vert\kern-0.25ex\left\vert #1
    \right\vert\kern-0.25ex\right\vert\kern-0.25ex\right\vert}}

\def\bfe{\mathbf{e}}

\def\Id{\operatorname{Id}}

\def\Id{\operatorname{Id}}

\newcommand\coupling[2]{\Gamma(\mu,\nu)}

\renewcommand{\geq}{\geqslant}
\renewcommand{\leq}{\leqslant}

\def\vareps{\varepsilon}

\def\tZ{\tilde{Z}}

\def\bfX{\mathbf{X}}

\def\Tf{\mathrm{T}_{f}}

\newcommand{\foX}{\overrightarrow{X}}

 \newcommand{\baX}{\overleftarrow{X}}
  \newcommand{\baY}{\overleftarrow{Y}}

 \newcommand{\KL}{\mathrm{KL}}

\def\bbf{\mathbf{b}}

\def\loss{\mathfrak{L}}

\def\loikhi2{\mathbf{\chi^2}}

\def\baq{\overleftarrow{q}}

\def\Tf{T_f}

\def\l{\left}
\def\r{\right}
\def\mE{\mathbb{E}}
\def\foq{\overrightarrow{q}}

\def\l{\left}
\def\r{\right}

\newcommand{\PE}{\mathbb{E}}

\newcommand{\tvnorm}[1]{\| #1 \|_{\mathrm{TV}}}

\def\Id{\mathrm{Id}}

\def\Mc{\mathcal{M}}

\def\mF{\mathcal{F}}
\def\MP{\mathrm{MP}}
\def\D{\mathbb{D}}

\def\mE{\mathcal{E}}

\def\mce{\mathcal{E}}
\def\mse{\mathsf{E}}
\def\msa{\mathsf{a}}

\def\canoX{\mathbf{X}}

\def\canoq{\mathbf{q}}
\def\genericq{\mathfrak{q}}

\def\msa{\mathsf{A}}

\def\msm{\mathsf{M}}

\def\um{\mathsf{um}}
\def\m{\mathsf{m}}

\def\Ic{\mathcal{I}}
\def\PP{\mathbb{P}}
\def\pX{p}

\def\muX{\mu}

\def\plusinfty{+\infty}

\def\tZ{\tilde{\mathbb{Z}}}

\def\foXrw{{X}^{\mathrm{RW}}}
\def\baXrw{\overleftarrow{X}^{\mathrm{RW}}}
\def\baXrwstar{\overleftarrow{X}^{\mathrm{RW},\star}}
\def\foqrw{{q}^{\mathrm{RW}}}
\def\baqrw{\overleftarrow{q}^{\mathrm{RW}}}
\def\rw{\gamma^{\mathrm{RW}}}
\def\foXm{{X}^{\mathrm{M}}}
\def\baXm{\overleftarrow{X}^{\mathrm{M}}}
\def\baXmstar{\overleftarrow{X}^{\mathrm{M},\star}}
\def\baYmstar{\overleftarrow{Y}^{\mathrm{M},\star}}

\def\foqm{{q}^{\mathrm{M}}}
\def\baqm{\overleftarrow{q}^{\mathrm{M}}}

\def\fopm{{p}^{\mathrm{M}}}
\def\fopmm{{p}^{\mathrm{M},1}}

\def\foXbrw{{X}^{\mathrm{BRW}}}
\def\baXbrw{\overleftarrow{X}^{\mathrm{BRW}}}
\def\baXbrwstar{\overleftarrow{X}^{\mathrm{BRW},\star}}
\def\foqbrw{{q}^{\mathrm{BRW}}}
\def\baqbrw{\overleftarrow{q}^{\mathrm{BRW}}}
\def\brw{\gamma^{\mathrm{BRW}}}

\def\uM{u^{\mathrm{M}}}
\def\uMstar{u^{\mathrm{M},\theta^\star}}
\def\uRW{u^{\mathrm{RW}}}
\def\uRWstar{u^{\mathrm{RW},\theta^\star}}
\def\uBRW{u^{\mathrm{BRW}}}
\def\uBRWstar{u^{\mathrm{BRW},\theta^\star}}
\def\tqm{\tilde{q}^{\mathrm{M}}}

\def\tqbrw{\tilde{q}^{\mathrm{BRW}}}

\def\qrw{q^{\mathrm{RW}}}

\def\mum{\mu^{\mathrm{M}}}
\def\murw{\mu^{\mathrm{RW}}}
\def\mubrw{\mu^{\mathrm{BRW}}}

\def\foprw{p^{\mathrm{RW}}}
\def\fopbrw{p^{\mathrm{BRW}}}

\def\Vrw{V^{\mathrm{RW}}}
\def\Vm{V^{\mathrm{M}}}
\def\Vbrw{V^{\mathrm{BRW}}}
\def\m{\mathbf{m}}
\def\baPrw{\overleftarrow{\mathbb{P}}^{\mathrm{RW}}}
\def\baPm{\overleftarrow{\mathbb{P}}^{\mathrm{M}}}
\def\baPbrw{\overleftarrow{\mathbb{P}}^{\mathrm{BRW}}}
\def\baPrwstar{\overleftarrow{\mathbb{P}}^{\mathrm{RW},\star}}
\def\baPmstar{\overleftarrow{\mathbb{P}}^{\mathrm{M},\star}}
\def\baPbrwstar{\overleftarrow{\mathbb{P}}^{\mathrm{BRW},\star}}

\def\tuMstar{\tilde{u}^{\mathrm{M},\theta^\star}}

\def\complement{\mathrm{c}}
\def\tmubrw{\tilde{\mu}^{\mathrm{BRW}}}

\def\tmustar{\tilde{\mu}^\star}

\def\mB{\mathcal{B}}
\def\mF{\mathcal{F}}

\def\Y{\mathrm{Y}}

\newcommand{\QQ}{\mathbb{Q}}
\def\mri{\mathrm{I}}
\def\mrm{\mathsf{m}}

\def\Puq{\P^{u\canoq}}
\def\Pq{\P^{\canoq}}
\def\baqrwstar{\overleftarrow{q}^{\mathrm{RW},\theta^\star}}
\def\baqmstar{\overleftarrow{q}^{\mathrm{M},\theta^\star}}
\def\bapmstar{\overleftarrow{p}^{\mathrm{M},\theta^\star}}
\def\baqbrwstar{\overleftarrow{q}^{\mathrm{BRW},\theta^\star}}

\def\tuM{\tilde{u}^{\mathrm{M}}}
\def\tuMstar{\tilde{u}^{\mathrm{M},\theta^\star}}
\def\Mrw{M^{\mathrm{RW}}}
\def\Mm{M^{\mathrm{M}}}
\def\tMm{\tilde{M}^{\mathrm{M}}}
\def\Mbrw{M^{\mathrm{BRW}}}
\def\bfMbrw{\mathbf{M}^{\mathrm{BRW}}}

\def\Sbrw{S^{\mathrm{BRW}}}
\def\Abrw{A^{\mathrm{BRW}}}
\def\fopbrwb{p^{\mathrm{BRW},1}}
\def\foqbrwb{q^{\mathrm{BRW},1}}

\def\dbf{\mathbf{d}}
\def\Mbf{\mathbf{M}}
\def\mstar{\mathbf{m}^\star}

\crefname{theorem}{theorem}{theorems}
\Crefname{theorem}{Theorem}{Theorems}
\crefname{remark}{remark}{remarks}
\Crefname{remark}{Remark}{Remarks}

\crefname{proposition}{proposition}{propositions}
\Crefname{proposition}{Proposition}{Propositions}

\crefname{corollary}{corollary}{corollaries}
\Crefname{corollary}{Corollary}{Corollaries}

\crefname{example}{example}{examples}
\Crefname{example}{Example}{Examples}

\crefname{lemma}{lemma}{lemmas}
\Crefname{lemma}{Lemma}{Lemmas}

\title{Non-Asymptotic Convergence of Discrete Diffusion Models: Masked and Random Walk dynamics}

\author[1]{{Giovanni Conforti}}

\author[2]{{Alain Durmus}}

\author[3]{{Le-Tuyet-Nhi Pham}\footnote{corresponding author}
\thanks{le-tuyet-nhi.pham@polytechnique.edu}}

\author[4]{{Ga\"el Raoul}}

\affil[1]{Department of Mathematics, University of Padova\\
Padova, Italy}

\affil[2,3,4]{\'Ecole Polytechnique, CMAP, IP Paris\\
Palaiseau, France}

\begin{document}
\maketitle
\begin{abstract}
  Diffusion models for continuous state spaces based on Gaussian noising processes are now relatively well understood from both practical and theoretical perspectives. In contrast, results for diffusion models on discrete state spaces remain far less explored and pose significant challenges, particularly due to their combinatorial structure and their more recent introduction in generative modelling. 
  In this work, we establish new and sharp convergence guarantees for three popular discrete diffusion models (DDMs). Two of these models are designed for finite state spaces and are based respectively on the random walk and the masking process.
The third DDM we consider is defined on the countably infinite space $\N^d$ and uses a drifted random walk as its forward process. For each of these models, the backward process can be characterized by a discrete score function that can, in principle, be estimated. However, even with perfect access to these scores, simulating the exact backward process is infeasible, and one must rely on time discretization. In this work, we study Euler-type approximations and establish convergence bounds in both Kullback–Leibler divergence and total variation distance for the resulting models, under minimal assumptions on the data distribution. 
To the best of our knowledge, this study provides the {optimal non-asymptotic} convergence guarantees for these noising processes that do not rely on boundedness assumptions on the estimated score. In particular, the computational complexity of each method scales only {linearly in the dimension, up to logarithmic factors}.
\end{abstract}

\section{Introduction}
Diffusion Models (DMs) have established themselves as a fundamental tool for the generation of complex, high-dimensional data, including images \cite[see, \textit{e.g.},][]{rombach2022high,ramesh2022hierarchical}, audio \cite{chen2020wavegrad,kong2020diffwave}, and video \cite{ho2022video,villegas2022phenaki,bar2024lumiere}.
In their first formulation, DMs define a forward process governed by stochastic differential equations (SDEs) that progressively corrupt the data with noise until it reaches an easy-to-sample prior distribution, and then learn the corresponding reverse dynamics to reconstruct samples from this prior back to the data distribution.
In their continuous-time formulation, DMs benefit from a rigorous theoretical foundation and an analytically stable learning objective \cite{song2021scorebasedgenerativemodelingstochastic,chen2022sampling,dockhorn2021score,conforti2025kl}.

By contrast, {Discrete Diffusion Models} (DDMs) continue to pose significant challenges. Multiple diffusion-based methods have recently been proposed for discrete spaces \cite{austin2021structured,shi2024simplified,campbell2022continuous,holderrieth2024generator,ren2024discrete}, or spaces of mixed type \cite{bertazzi2024piecewise}, but there is still no consensus on which approach is theoretically sound or most practically efficient. Various formulations rely on complex forward kernels or computationally unstable ratio-based estimators for backward transitions, leading to limited convergence guarantees and high computational costs in high dimensions. Furthermore, recent studies on discrete diffusion models have introduced valuable theoretical tools \cite{campbell2022continuous,holderrieth2024generator,ren2024discrete}, yet most approaches remain either too generic or rely on strong assumptions, limiting their scalability and stability during training.
Recently, \cite{bach2025samplingbinarydatadenoising} and \cite{pham2025discrete} obtained sharp convergence guarantees for data supported on the hypercube under mild assumptions, and \cite{liang2025absorb} established error bounds for masked diffusion processes on $\mathbb{Z}^d_m$. However, these approaches are either tailored to the hypercube or rely on uniform boundedness conditions on the estimated score, and none extend to the countably infinite setting $\mathbb{N}^d$.

This paper bridges these gaps by establishing theoretical guarantees for DDMs driven by masked and random walk dynamics, applicable to discrete data supported on both the finite space $\Z^d_m$ and the countably infinite space $\N^d$, where \cite{anonymous2025countsdiff} have demonstrated the necessity and strong empirical performance of such models.
Specifically, we study DDMs driven by random walks on the cycle $\Z^d_m$, and establish the non-asymptotic error bounds that does not depend on the boundedness of the estimated score by leveraging the score monotonicity.
For the random walk defined on $\N^d$, we further demonstrate that a similar type of bound holds under the mild assumption that the data distribution admits a finite second moment.
Notably, this yields the {first rigorous convergence analysis} for data supported on the {countably infinite state space $\N^d$}.
In addition, we investigate the widely adopted masked diffusion model for discrete data, originally introduced by \cite{austin2021structured}, which has demonstrated strong empirical performance in recent studies \cite{shi2024simplified,chao2025beyond}.
This study provides the first {non-asymptotic} error bound for masked diffusion models under an early-stopping scheme, which does not rely on the boundedness assumption on the estimated score function. The result holds under the mild assumption that the data distribution is fully supported—a practical condition, since smoothing can always be applied in practice.
These results underscore the generality and robustness of our analytical framework. In particular, our analysis crucially relies on the evolution of the score, especially its monotonicity estimates, along the backward dynamics, enabling rigorous error bounds without imposing overly restrictive assumptions on the data. Furthermore, by employing an appropriate sequence of step sizes, we achieve a complexity that scales {linearly} (up to logarithmic factors), rather than exponentially, with the dimension.


\textbf{Notation.} Given a measurable space $(\mse, \mce)$, we denote by $\Pc(\mse)$ the set of probability measures on $\mse$ and by $2^{\mse}$ the power set of $\mse$. Given two probability measures $\mu, \nu \in \Pc(\mse)$, the Kullback--Leibler divergence (also called relative entropy) of $\mu$ with respect to $\nu$ is defined as $\KL(\mu|\nu):=\int \log({\rmd\mu}/{\rmd\nu})\rmd\mu$ if $\mu$ is absolutely continuous with respect to $\nu$, and $\KL(\mu|\nu)=+\infty$ otherwise. The total variation distance between $\mu$ and $\nu$ is defined as $\|\mu-\nu\|_{\text{TV}}= \sup_{\msa \in \mce} |\mu(\msa)-\nu(\msa) |$. Consider a random variable $X$, we denote by $\mathrm{Law}(X)$ the law of $X$. We denote by $\updelta_x$ the Dirac mass at $x$. We use $\1(\cdot)$ to denote an indicator function. We denote the set $\{1,\dots,n\}$ as $[n]$ for a natural number $n \in \N^*$.

\section{Discrete diffusion models}

\subsection{Continuous-Time Markov Chains and their time-reversal}\label{sec:CTMC_timereversal}

\subsubsection{Continuous-Time Markov Chain} We begin with a brief overview of Continuous-Time Markov Chains (CTMCs). Throughout this paper, we consider a discrete countable state space $\msx$.
A CTMC on $\msx$ is a
time-indexed right-continuous stochastic process $(X_t)_{t \geq 0}$ on a probability space $(\Omega,\mcf,\PP)$ that Markov,
\ie,
for any $0 \leq s < t$, almost surely, $\P(X_{t_n} = x_{t_n} | \mcf_s)= \P(X_{t_n} = x_{t_n} | X_{s})$,
where $(\mcf_t)_{t\geq 0}$ is the natural filtration associated with $(X_t)_{t\geq0}$.

To design CTMCs, one central object is a rate matrix, \ie, function $(x,t,\msb)\in \msx \times \R_+ \times 2^{\msx} \mapsto q_t(x,\msb)$; also referred to $Q$-function or generator. In particular, we consider the following assumption on $(q_t)_{t\geq 0}$.
\begin{assumption}
	\label{ass:ctcm}
	$(q_t)_{t\geq 0}$ is a stable conservative rate matrix, \ie, it satisfies the following properties:
	\begin{itemize}
		\item for all $(x,t) \in \msx \times \rset_+$, the function $q_t(x,\cdot)$ is a signed (discrete) measure on $\msx$ such that $q_t(x,\msx) = 0$ and $0 \leq q_t(x,\mathsf{B}\setminus \{x\}) <\infty $ for all $\mathsf{B} \subset \msx$;
		\item for all $\mathsf{B} \subset \msx$, the function $(x,t) \mapsto q_t(x,\mathsf{B})$ is measurable;
		\item for all $x \in \msx$, the singleton $\{x\}$ is $q$-bounded, \ie,  $\sup_{t\in \rset_+} (-q_t(x, \{x\}))  <\infty$. 
	\end{itemize}
\end{assumption}
We will show that~\Cref{ass:ctcm} is satisfied by the generators of the forward processes considered in~\Cref{sec:considered_models}.

Under~\Cref{ass:ctcm}, the generator $(q_t)_{t\geq 0}$ allows to define a sub-Markov semigroup $\{\pX_{s,t}\,:\, 0\leq s <t \}$ on $\msx$ and  corresponding CTMCs,
\ie, for $0\leq s <t$, $\pX_{s,t}$ is a  transition sub-probability density and  the Chapman-Kolmogorov equation holds, \ie, for any
$0\leq s< u < t$ and $x_s,x_u,x_t\in\msx$,
\begin{equation}
	\label{eq:chapman_kolmogorov}
	\pX_{s,t}(x_s,x_t) = (\pX_{s,u}\pX_{u,t})(x_s,x_t) \eqdef \sum_{x_u \in\msx} \pX_{s,u}(x_s,x_u) \pX_{u,t}(x_u,x_t) \eqsp.
\end{equation}
In the case where for any $0<s<t$, $p_{s,t} = p_{0,t-s}$, then $\{\pX_{s,t}\,:\, 0\leq s <t \}$ is said to be a homogeneous semigroup. Otherwise it is said
to be inhomogeneous.

Here we suppose that there is no explosion which is equivalent to the fact that the family of semigroup $\{\pX_{s,t}\,:\, 0\leq s <t\}$ is in fact  Markov.
\begin{assumption}
	\label{ass:non_explo}
	For any $0\leq s <t$ and $x \in\msx$, $\pX_{s,t}(x,\msx) = 1$.
\end{assumption}
\begin{remark}
	\label{rem:non_explo}
	Note that \cite[Theorem 6]{feller1940integro} ensure  that~\Cref{ass:non_explo} holds if  there exists a measurable function $t\mapsto \phi_t$ such that $\sup_{x\in\msx} (-q_t(x, \{x\})) \leq \phi_t$ for any $t \geq 0$, and for a fixed $p >1$, $\int_{s}^t (\phi_u)^p \rmd u < \plusinfty$ for any $0\leq s < t$. We also refer to \cite{zhang2018nonexplosion} for conditions implying~\Cref{ass:non_explo} for unbounded generators.
\end{remark}
Under \Cref{ass:ctcm} and \Cref{ass:non_explo}, for any initialization $X_0\sim \mu_0$, there exists a CTMC $(X_t)_{t \geq 0}$ starting from $X_0$ and associated with the family of transitions $\{\pX_{s,t}\,:\, 0 \leq s <t \}$:
\begin{equation}
	\label{eq:density_transition_X}
	\pX_{s,t}(x_s,x_t) =
	\begin{cases}
		\PP(X_{t} = x_{t} | X_s = x_s) & \text{ if $\PP(X_s=x_s) \neq 0$} \eqsp, \\
		\PP(X_{t} = x_{t}) & \text{ otherwise } \eqsp.
	\end{cases}
\end{equation}
In addition, \cite[Theorem 4.3.]{feinberg2014solutions} shows that for all $(x,s) \in \msx \times \R_+$ and $\msb \subset \msx$ such that $\sup_{t \in \R_+, x\in \msb} (-q_t(x,\{x\}))< \infty$, the function $\pX_{s,t}(x, \msb)$ satisfies for almost $t>s$ the forward Kolmogorov equation:
\begin{equation}
	\label{eq:foward_kolm}
	\frac{\partial \pX_{s,t}}{\partial t} (x,\msb) = \int_{\msb} q_t(y,\{y\}) \pX_{s,t}(x,\rmd y)  + \int_{\msx} q_t(y,\msb \setminus \{y\}) \pX_{s,t}(x,\rmd y)  \eqsp.
\end{equation}
Given the generator $(q_t)_{t\geq 0}$, we define
\begin{align}
	(qf)(t,x) = \sum_{y \neq x} q_t(x,y) [f(t,y)- f(t,x)]\eqsp, \quad \text{for any function $f$} \eqsp.
\end{align}
We note that while we restrict our presentation to a family of
inhomogeneous generators $(q_t)_{t\geq 0}$ indexed by $\rset_+$, the
same results and construction apply when  it is restricted to a finite interval $\ccint{0,\Tf}$ for an horizon $\Tf >0$.

{\color{black}
	\subsubsection{Time-reversal process}  Starting from the CTMC $(X_t)_{t\geq 0}$ with initial distribution $\mu_0$ and associated with $(q_t)_{t\geq 0}$, we define the corresponding time-reversal process  $(\overleftarrow{X}_t)_{t \in \ccint{0,\Tf}}$ for a horizon $\Tf$, for any $t \in\ \ccint{0,\Tf}$ as  $\overleftarrow{X}_t={X}_{\Tf-t}$.
	Under~\Cref{ass:ctcm} and~\Cref{ass:non_explo}, \cite[Theorem 2.8]{conforti2022time} applies (see Appendix \ref{app:application_reversal} for completeness}):  $(\overleftarrow{X}_t)_{t\in\ccint{0,\Tf}}$ is also an {inhomogeneous} CTMC, associated with a family of generator matrices $(\overleftarrow{q}_t)_{t \in\ccint{0,\Tf}}$ satisfying the time-reversal formula:
	\begin{equation}
		\label{eq:time_rev_formula_X}
		\muX_{\Tf-t}(x)\overleftarrow{q}_t(x,y)=\muX_{\Tf-t}(y)q_{\Tf -t}(y,x) \eqsp,
	\end{equation}
	for any $0 \leq t \leq \Tf$ and  $x \neq y \in \msx$, where for any $t\in\ccint{0,\Tf}$, we denote by $\muX_t$ the forward marginal distribution:
	\begin{equation}
		\muX_t(x) = \PP(X_t = x) \eqsp.
	\end{equation}
	Equation~\eqref{eq:time_rev_formula_X} serves as the key tool to derive the backward generator $(\baq_t)_{t\in [0,\Tf]}$, which inherits the stability of $(q_t)_{t\in [0,\Tf]}$.
	Starting from~\eqref{eq:time_rev_formula_X}, under mild assumptions on the target $\mustar$, we show in~\Cref{sec:considered_models} that for all the models we consider, the associated  backward generator $(\baq_t)_{t\in [0,\Tf)}$ can be written for any $0 \leq t \leq \Tf$ and $x \neq y \in \msx$ as
	\begin{equation}\label{eq:backward_generator_ctmc}
		\baq_t (x,y) = u_t (x,y) \tilde{q}_t(x,y) \eqsp.
	\end{equation}
	Here $(\tilde{q}_t)_{t \in [0,\Tf]}$ is an auxiliary generator derived straightforwardly from the forward generator $(q_t)_{t \in [0,\Tf]}$, and $(u_t)_{t\in [0,\Tf]}$ is a family of non-negative function from $\msx^2$ to $\rset_+$, which plays a similar role as the score function in continuous generative models.
	The explicit expression of $(\tilde{q}_t)_{t \in [0,\Tf]}$ and $(u_t)_{t\in [0,\Tf]}$ are both provided in~\Cref{sec:considered_models}, and note that for all considered models except masked diffusion, $\tilde{q} = q$. 
	
	
	\subsubsection{Approximation of the time-reversal process}
	Similar to standard continuous diffusion models, simulating the reverse process exactly is infeasible due to three main bottlenecks:
	(i) the starting distribution $\baX_0\sim \muX_{\Tf}$ is intractable;
	(ii) the dynamics of the backward process rely on the backward generator $(\baq_t)_{t\in \ccint{0,\Tf}}$ that we do not have access;  (iii) finally, even in the absence of the two previous limitations, inhomogeneous CTMCs still require time discretization. In fact, time discretization can in principle be avoided via the uniformization algorithm (see \cite{wan2025error}), yet this approach scales poorly in high-dimensional settings.
	
	Regarding (i), instead of sampling exactly from $\mu_{\Tf}$, we can start the DDMs from an easy-to-sample distribution $\gamma$ that approximates $\mu_{\Tf}$, typically chosen as the invariant measure of the forward dynamics.
	
	As for (ii), starting from~\eqref{eq:backward_generator_ctmc}, we will see that in all the models that we consider, the main unknown is the family of function  $(u_t)_{t\in [0,\Tf]}$ which depends implicitly on the data distribution, while $(\tilde{q}_t)_{t \in [0,\Tf]}$ can easily be derived from the forward matrix rate $(q_t)_{t\in\ccint{0,\Tf}}$.   Typically, $(u_t)_{t\in [0,\Tf]}$ is approximated using a parameterized family  $\{(u_t)_{t\in [0,\Tf]} \, : \, \theta \in \Theta \}$, which is trained using an appropriate loss function.
	In particular, following \cite{lou2023discrete}\footnote{the score function considered in \cite{lou2023discrete} corresponds to $u_t$ in our notations.} we consider the loss function
	\begin{equation}
		\label{eq:loss_entropy_function}
		\loss_{\mathrm{e}}(\theta) \eqdef \int_0^{\Tf} \E \l[ \sum_{y\in \msx} \l(u_{t}\log \frac{u_{t}}{u^{\theta}_{t}}-u_{t}+u^{\theta}_{t} \r)\tilde{q}_{t}(X_{\Tf-t},y)\1_{y \neq X_{\Tf-t}} \r] \rmd t \eqsp.
	\end{equation}
	Furthermore, leveraging the conditional expectation formulation of the discrete score, we can incorporate a stable $\mathrm{L}^2$ loss into our objective function:
	\begin{equation}
		\loss_2(\theta) \eqdef \int_0^{\Tf} \E \l[ \sum_{y \in \msx} \l\| (u_t - u^{\theta}_t) \tilde{q}_t(X_{\Tf-t},y)  \r\|^2  \1_{y \neq X_{\Tf-t}} \r] \rmd t \eqsp.
	\end{equation}
	Based on an approximate minimizer of this function, from~\eqref{eq:backward_generator_ctmc},  we could in principle  consider the resulting backward generator defined for any $t \in [0,\Tf)$ and $x,y \in \msx$ as
	\begin{equation}
		u^{\theta^\star}_t (x,y) \tilde{q}_t(x,y) \eqsp.
	\end{equation}
	However, exact simulation of the CTMC associated with  this rate matrix is infeasible in practice, so we discretize time and approximate the backward rate using piecewise constant functions.
	
	Let $\{t_k\}_{k=0}^K$ be a time grid with step sizes $h_k = t_k - t_{k-1}$.
	Overall, given an approximation $\gamma$ from $\muX_{\Tf}$, $(u^{\theta^\star}_{t})_{t\in\coint{0,\Tf}}$ and  $\{t_k\}_{k=0}^K$, the generative process $(\baX^\star_{t})_{\in\ccint{0,\Tf}}$ that we consider is defined as follows.
	First set $\baX^\star_0 \sim \gamma$. Given $\baX^\star_{t_k}$, for $t \in [t_k,t_{k+1})$ and $x \neq y \in \msx$, we set
	\begin{align}\label{eq:generator_generative_ctmc}
		\baq_{t}^{\theta^\star}(x,y) &=   u^{\theta^{\star}}_{t_k}(x,y|\baX^\star_{t_k})\tilde q_{t} (x,y) \eqsp,
	\end{align}
	where $(x,y) \mapsto u^{\theta^{\star}}_{t_k}(x,y|\baX^\star_{t_k})$ is a function that is straightforwardly designed from  the estimate $u^{\theta^\star}_{t_k}$ and the current state $\baX^\star_{t_k}$. 
	For simplicity, we restrict attention to the case where $\tilde{q}_t$ is time-independent. Equipped with the generator $\baq_t^{\theta^\star}$, the trajectory can be sampled as follows. Set $t= t_k$ and let $\lambda_{t}(x) = \sum_{y \neq x} \baq^{ \theta^\star}_{t} (x, y)$ denote the total rate at the state $x$, then draw the holding time $\tau \sim  \mathrm{Exp}(\lambda_{t_k}(\baX_{t_k}^\star))$.
	\begin{enumerate}[wide,label=(\arabic*)]
		\item If $t+\tau > t_{k+1}$: set $\baX^\star_{\tilde{t}} = \baX^\star_{t}$ for all $\tilde{t} \in [t, t_{k+1}]$ and update $t = t_{k+1}$.
		\item Otherwise, set $\baX^\star_{\tilde{t}} = \baX^\star_{t}$ for all $\tilde{t} \in [t, t+\tau)$.  Select $y \neq x \in \msx$ with probability $\baq^{ \theta^\star}_t(\baX^\star,y)/ \lambda_t(\baX^\star_t)$, then update $\baX^\star_{t+\tau} = y $ , $t = t+ \tau$ and repeat from the step computing the total rate.
	\end{enumerate}
	Notably, at each iteration we avoid solving the Kolmogorov equation to obtain the next state; instead, we simply use the transition rates together with the exponential sampler.
	In practice, Poisson sampler can be used to simulate a discrete-time Markov chain $(\baY^\star_k)_{k=0}^K$ whose law matches that of $(\baX^\star_t)_{t\in [0,\Tf]}$ evaluated on the time grid, that is, $\baY^\star_k \overset{\text{Law}}{=} \baX^\star_{t_k}$.
	For each $\sigma \in \Mc$, we draw $n_\sigma \sim \mathrm{Poisson}(h_{k+1}\baq^{ \theta^\star}_{t_k}(\baY^\star_{t_k}, \sigma(\baY^\star_{t_k})))$ where $n_\sigma$ denotes the number of times the move $\sigma$ is applied. We then update the state according to $\baY^\star_{t_{k+1}} = (\Pi_{\sigma} \sigma^{n_\sigma}) (\baY_{t_k}^\star)$.

	In what follows, we present the CTMC that we consider as forward
	process and specify the two functions $(x,y) \mapsto u^{\theta^\star}_{t_k}(x,y|\baX^\star_{t_k})$ and $(x,y) \mapsto \tilde q_{t} (x,y)$ in the generator~\eqref{eq:generator_generative_ctmc}  of the resulting generative process. Next, we show that explicit convergence bounds between the associated generative distribution and the data distribution $\mustar$.
	
{\color{black}
	
	\subsection{Discrete Diffusion Models}
	\label{sec:considered_models}
	
	We first focus on the finite state space $\Z^d_m=\{0,\dots,m-1\}^d$, where we consider two types of noising processes: the random walk process and the mask diffusion process originally proposed by \cite{austin2021structured}. Subsequently, we investigate the biased random walk on the countably infinite state space $\N^d$.

	
	
	
	\subsubsection{Random walk on \protect{$\Z^d_m$}}
	
	We define the forward process $(\foXrw_t)_{t\in [0,\Tf]}$ on $\Z^d_m$ over a fixed time horizon $\Tf>0$ as a homogeneous CTMC initialized from the data distribution $\mustar$, as a distribution on $\Z^d_m$, and associated with the generator $\foqrw$ specified for $x,y \in \Z^d_m$ as follows:
	\begin{align}\label{eq:forward_generator_rw}
		\foqrw (x,y) = \begin{cases}
			1/2 \quad &\text{if } y = \sigma(x) \text{ for } \sigma \in \mcs \eqsp,\\
			- d \quad &\text{if } y=x \eqsp, \\
			0  \quad &\text{otherwise} \eqsp,
		\end{cases}
	\end{align}
	where
	\begin{equation}
		\label{eq:def_mc}
		\mcs \eqdef \{\sigma^\ell_+, \sigma^\ell_- \, : \, \ell \in [d] \}\eqsp,
	\end{equation}
	and
	the operators $\sigma^{\ell}_+, \sigma^\ell_-$ correspond to the forward and backward jump on the $\ell$-th component, respectively, defined as
	\begin{equation}
		\label{eq:def_sigma_ell}
		\sigma^\ell_+(x) = x+\bfe^\ell \pmod{m} \quad \text{ and } \quad \sigma^\ell_-(x) = x-\bfe^\ell \pmod{m} \eqsp,
	\end{equation}
	for $x \in \Z^d_m$, where $\{\bfe^\ell\}_{\ell=1}^d$ are the basic vectors of $\R^d$ and $\pmod{m}$ denotes the modulo operation by $m$.
	It follows directly from~\eqref{eq:forward_generator_rw} that $\foqrw$ is a non-explosive stable conservative rate matrix, \ie, satisfies~\Cref{ass:ctcm} and~\Cref{ass:non_explo}, ensuring that the process $(\foXrw_t)_{t\in [0,\Tf]}$ is well-defined.
	
	In addition, it is well known that $\rw = \mathrm{Uniform}(\Z^d_m)$ is an invariant distribution for $\foqrw$ by \cite[Section 3.5]{norris1998markov} since
	for any $x \in \Z^d_m$, the following holds:
	\begin{align}
		(\rw \foqrw)(x) = \sum_{y \in \Z^d_m} \rw(y) \foqrw(y,x) = \frac{1}{m^d} \sum_{y \in \Z^d_m} \foqrw (y,x) = 0 \eqsp.
	\end{align}
	Moreover, $(\foXrw_t)_{t\in [0,\Tf]}$ converges geometrically fast to $\rw$ in various metrics. Here we will exploit that it converges in Kullback-Leibler divergence. Indeed, $\rw$ satisfies a discrete  Logarithm Sobolev inequality (LSI) which implies by  \cite[Lemma 3.2, Section 4.2]{diaconis1996logarithmic} and \cite[Theorem 2.4]{bobkov2003modified} that
	\begin{align}\label{eq:entropy_decay_rw}
		\KL(\murw_t|\rw) \leq \rme^{-\frac{16\pi^2}{25m^2}t}\KL(\murw_0|\rw)  \eqsp.
	\end{align}
	Furthermore,~\eqref{eq:backward_generator_ctmc} holds for this choice of CTMC with $\tilde{q}^{\mathrm{RW}}_t(x,y) = \foqrw(x,y)$ and
	\begin{align}
		\uRW_t(x,y) = \begin{cases}
			\quad{\murw_{\Tf-t}(\sigma(x))}/{\murw_{\Tf-t}(x)}  \quad &\text{if } y = \sigma(x) \text{ for } \sigma \in \mcs \eqsp,\\
			{\sum_{\sigma \in \mcs} \uRW_t(x,\sigma(x)) }/{2d} \quad &\text{if } x=y \eqsp,\\
			\qquad \qquad 1  \quad &\text{otherwise}\eqsp,
		\end{cases}
	\end{align}
	for $x,y \in \Z^d_m$ and $t \in [0,\Tf)$; see~\Cref{lem:formula_u_brw} for detailed justification. In addition, similar to diffusion models in continuous settings, the discrete score above can be expressed as a conditional expectation; see Appendix \ref{sec:further_rw} for completeness.
	
	The generative process is then simulated as described in~\Cref{sec:CTMC_timereversal} using the generator given in~\eqref{eq:generator_generative_ctmc} with
	\begin{align}\label{eq:control_generative_rw}
		\uRWstar_{t_k}(x,y|\baXrwstar_{t_k}) = \begin{cases}
			\uRWstar_{t_k}(\baXrwstar_{t_k}, \sigma(\baXrwstar_{t_k})) \quad &\text{if $y = \sigma(x)$ for $\sigma \in \mcs$} \eqsp,\\
			\qquad \qquad 1 \quad &\text{otherwise} \eqsp,
		\end{cases}
	\end{align}
	where $\mcs$ is defined in~\eqref{eq:def_mc}.
	The pseudo-code for simulating $(\baXrwstar_t)_{t\in [0,\Tf]}$ is provided in \Cref{alg:rw}, Appendix \ref{sec:ddms-algorithms} for completeness.

	
	\subsubsection{Masked diffusion on \protect{$\Z^d_m$}}
	Consider the state space $\Z_m^d$ and we augment it with an additional mask state on each component, which is assigned the index $m$, to obtain the extended state space $\tZ_m^d = \{0,\dots,m \}^d$.
	In addition, we set here
	\begin{equation}
		\label{eq:def_masked_set}
		\text{  $\msm_x:= \{i  \in [d]: \eqsp x^i =m \}$ and $\msm_x^{\complement}:= \{ i  \in [d]: \eqsp x^i \neq m \}$}
	\end{equation}
	denoting the set of masked and non-masked (maskable) coordinates of $x$, respectively.
	
	The forward masking process $(\foXm_t)_{t\in [0,\Tf]}$ is then defined as an inhomogeneous CTMC on $\tZ^d_m$, starting from $\mustar$ distributed on $\Z^d_m$, and associated with the generator $(\foqm_t)_{t\in [0,\Tf]}$ specified as follows: for $x,y \in \tZ_m^d$ and $t \in [0,\Tf]$,
	\begin{align} \label{eq:forward_generator_masked_d}
		\foqm_t(x, y ):&=\begin{cases}
			\quad\beta(t) \quad &\text{if } \exists i \in \msm_x^{\complement}: \, y = \mrm^{(i)}(x)  \eqsp,\\
			-|\msm_x^{\complement} | \beta(t) \quad &\text{if } x=y \eqsp,\\
			\quad 0  \quad &\text{otherwise}\eqsp,
		\end{cases}
	\end{align}
	where $t\mapsto \beta(t)$ is a function satisfying~\Cref{ass:condition_on_beta} below, and $\mrm^{(i)}(x)$ is the vector obtained from $x$ by setting the $i$-th coordinate to the mask value $m$ and leaving all other coordinates unchanged:
	\begin{equation}
		\label{eq:def_m_i_x}
		(\mrm^{(i)}(x))^j = \begin{cases}
			m \quad &\text{if } j=i \eqsp,\\
			x^j \quad &\text{if } j \neq i \eqsp,
		\end{cases} \quad \text{for } i,j  \in [d] \eqsp.
	\end{equation}
	We impose the following assumptions on the function $\beta$:
	
	\begin{assumptionM}
		\label{ass:condition_on_beta}
		$t\mapsto \beta(t)$ is continuous, non-decreasing from $\rset_+$ to $\ccint{0,1}$
		and $\int_{0}^{\plusinfty}\beta(t) = \plusinfty$.
	\end{assumptionM}
	\Cref{ass:condition_on_beta}  ensures that $    \sup_{t\in [0,\Tf]}  \foqm_t(x)=|\msm_x^{\complement}| \sup_{t \in [0,\Tf]} \beta(t) <\infty$ for $x \in \tZ_m^d$, namely, $(\foqm_t)_{t\in [0,\Tf]}$ is stable and conservative. Furthermore, it also implies that $        \sup_{x\in \tZ^d_m}\foqm_t(x) = d \beta(t) \leq d$, and therefore $(\foqm_t)_{t\in [0,\Tf]}$ satisfies~\Cref{ass:ctcm} and~\Cref{ass:non_explo} following \Cref{rem:non_explo}.
	
	In addition, based on the generator $(\foqm_t)_{t\in [0,\Tf]}$, we can  determine the transition probabilities $\fopm_{s,t}$ for all $0\leq s \leq t \leq \Tf$; see \eg, \Cref{sec:CTMC_timereversal}.
	Specifically, for $x,y \in \tZ^d_m$, denoting by $x^i$ and $y^i$ the $i$-th component respectively, and for $0 \leq s \leq t \leq \Tf$, we have
	\begin{align}\label{eq:forward_transition_matrix_mask_d}
		\fopm_{s,t}(x,y) = \prod_{i=1}^d \fopmm_{s,t} (x^i,y^i)  \eqsp, \quad \text{where } \fopmm_{s,t} (x^i,y^i) = \begin{cases}
			\alpha_t/\alpha_s  \quad &\text{if } x^i = y^i \eqsp, \\
			1- \alpha_t/\alpha_s  \quad &\text{if } y^i = m \eqsp, \\
			0   \quad &\text{otherwise} \eqsp,
		\end{cases}
	\end{align}
	where $\alpha_t \eqdef \exp ( - \int_0^t \beta(s) \rmd s )$. We can interpret~\eqref{eq:forward_transition_matrix_mask_d}:  on each component, at time $s \in [0,\Tf]$, if the process is at the normal state, it will jump to mask at time $t \in [s,\Tf]$ with probability $1-\alpha_t/\alpha_s$ and remains staying the same state with probability $\alpha_t/\alpha_s$. Once it is masked, it will stay masked forever.
	Note that under~\Cref{ass:condition_on_beta},  $\alpha_0 = 1$ and $\lim_{\Tf \to \infty} \alpha_{\Tf} =0$, meaning that every state is unmasked initially but at the end, almost states are masked.
	
	For this choice of generator~\eqref{eq:forward_generator_masked_d}, ~\eqref{eq:backward_generator_ctmc} holds with $\tqm_t(x,y) = \foqm_{\Tf-t}(y,x)$ and
	\begin{align}
		\uM_t(x,y) = \begin{cases}
			\quad{\mum_{\Tf-t}(y)}/{\mum_{\Tf-t}(x)} \quad &\text{if } \exists i \in \msm_y^{\complement}: \eqsp x=\mrm^{(i)}(y) \eqsp,\\
			-{\sum_{y \neq x} \uM_t(x,y) \tqm_t (x,y)}/{\tqm_t(x,x)} \quad &\text{if } x=y \eqsp,\\
			\qquad \qquad  1  \quad &\text{otherwise}\eqsp,
		\end{cases}
	\end{align}
	for $x,y \in \tZ^d_m$ and $t\in [0,\Tf)$; see~\Cref{lem:formula_u_masked_d} for the detailed proof. 
	Similar to the random walk on $\Z^d_m$ case, the discrete score above can also be expressed as a conditional expectation, which enables efficient training; see Appendix \ref{sec:further_mask}.
	
	The backward dynamic can be approximated by using the generative process presented in~\Cref{sec:CTMC_timereversal}, which is associated with the trained generator given in~\eqref{eq:generator_generative_ctmc} with
	\begin{align}\label{eq:control_generative_masked_d}
		\uMstar_{t_k}(x,y|\baXmstar_{t_k}) =
		\begin{cases}
			\uMstar_{t_k}(\baXmstar_{t_k}, \um^{(i)}_j (\baXmstar_{t_k})) &\text{if $y = \um^{(i)}_j(x)$ for $(i,j) \in \msm_x \times \Z_m$} \eqsp, \\
			\qquad \qquad 1 &\text{otherwise} \eqsp,
		\end{cases}
	\end{align}
	where $\um^{(i)}_j(x)$ is the vector obtained from $x \in \tZ^d_m$ by setting the $i$-th coordinate to the original value $j$ and keeping all other coordinated unchanged for fixed $i\in [d]$ and $j \in \Z_m$:
	\begin{align}\label{eq:unmask_operator_d}
		(\um^{(i)}_j(x))^\ell =\begin{cases}
			j \quad &\text{if $\ell=i$} \eqsp,\\
			x^\ell \quad &\text{if $\ell \neq i$} \eqsp,
		\end{cases} \quad \text{for $\ell \in [d]$}\eqsp.
	\end{align}
	The pseudo-code for sampling $(\baXmstar_t)_{t\in [0,\Tf]}$ in the case $\beta(t)=1$ for all $t \in [0,\Tf]$ is provided in \Cref{alg:masked_d}, Appendix \ref{sec:ddms-algorithms} for completeness.
	
	Regarding the initialization of our process, while we derive bounds starting our process from   $\updelta_{\text{MASK}}^{\otimes}$, these results are based on considering our generative process starting the other initial distribution  $\mathrm{Uniform}(\Z^d_m) \fopm_{0,\Tf}$ for which bound in KL can be derived. Indeed, convergence in KL of the noising process to its stationary distribution $\updelta_{\text{MASK}}^{\otimes}$ does not hold here which poses several challenges in the analysis of the initialization error of the process starting from $\updelta_{\text{MASK}}^{\otimes}$ in KL. On the other hand, taking $\mathrm{Uniform}(\Z^d_m) \fopm_{0,\Tf}$ which put positive probability on any element of $\tZ_m^d$ circumvent this issue.


	\subsubsection{Biased random walk on {$\N^d$}}\label{sec:brw}
	We consider the state space $\mathbb{N}^d$ and take as forward process a biased random walk on $\mathbb{N}^d$, defined by constant forward jump rates and backward jump rates proportional to the current state. Denoting this process by $(\foX^{\mathrm{BRW}}_t)_{t\in[0,T_f]}$, it is a homogeneous continuous-time Markov chain initialized according to the data distribution $\mu^\star$ and governed by the generator $\foq^{\mathrm{BRW}}$, defined for $t\in[0,\Tf]$ and $x,y\in\mathbb{N}^d$ as
	\begin{align} \label{eq:forward_generator_brw}
		\foqbrw(x, y ):&=\begin{cases}
			\quad 1 \quad &\text{if } y = \sigma^\ell_+ (x) \text{ for } \ell  \in [d] \eqsp,\\
			\quad x^{\ell} \quad &\text{if } y = \sigma^\ell_-(x) \text{ for } \ell  \in [d] \eqsp,\\
			-d-\sum_{\ell = 1}^d x^\ell \quad &\text{if } y=x \eqsp, \\
			\quad 0  \quad &\text{otherwise}\eqsp.
		\end{cases}
	\end{align}
	Here $x^{\ell}$ denotes the $\ell^{\text{th}}$-component of $x$ and  the operators $\sigma^{\ell}_+$ and $\sigma^{\ell}_-$ are defined as in \eqref{eq:def_sigma_ell} but without the modulo $m$ operation, noting here that we make a slight abuse of notation.
	
	In addition, similarly we also consider the set $\mcs$ as defined in~\eqref{eq:def_mc} with these new functions, using the same notation.
	Note that at the boundary $x^\ell = 0$, backward jumps are not allowed, as the corresponding transition rate is zero, so the process can only move to the right from that state.
	In addition, since the generator $\foqbrw$ is time-independent, the stability condition is automatically satisfied, therefore~\Cref{ass:ctcm} holds.
	~\Cref{ass:non_explo} is also easily satisfied;  see Appendix \ref{sec:further_brw} for completeness.
	Thus,  $(\foXbrw_t)_{t\in [0,\Tf]}$ is well-defined.
	
	Concerning the long-term behavior, it is well-known that $(\foXbrw_t)_{t \in [0,\Tf]}$ converges geometrically in KL divergence to the multidimensional Poisson distribution $\brw = \mathrm{Poisson}(1)^{\otimes d} \in \mathcal{P}(\N^d)$. Furthermore, \cite[Theorem 1.1]{Wu2000ANM} and \cite[Section 3]{conforti2022probabilistic} showed that $\brw$ satisfies the modified log-Sobolev inequality with constant $1$, which implies an exponential entropy decay:
	\begin{align}\label{eq:entropy_decay_brw}
		\KL(\mubrw_t|\brw) &\leq \rme^{-t}\KL(\mustar|\brw) \eqsp.
	\end{align}
	In this setting,~\eqref{eq:backward_generator_ctmc} holds with $\tilde{q}^{\mathrm{BRW}}_t = \foqbrw $ and
	\begin{align}
		\uBRW_t(x,y) = \begin{cases}
			\quad{\tmubrw_{\Tf-t}(\sigma(x))}/{\tmubrw_{\Tf-t}(x)} \quad &\text{if } y = \sigma(x) \text{ for } \sigma \in \mcs \eqsp,\\
			-{\sum_{\sigma \in \mcs} \uBRW_t \foqbrw(x,\sigma(x)) }/{\foqbrw(x,x)} \quad  &\text{if } y =x \eqsp,\\
			\qquad \qquad  1 \quad &\text{otherwise}\eqsp,
		\end{cases}
	\end{align}
	for $x,y \in \N^d$ and $t\in [0,\Tf]$, where $\tmubrw \eqdef \mubrw/\brw$; see~\Cref{lem:formula_u_brw} for completeness. Similar to the two aforementioned cases, the discrete score above can also be expressed as a conditional expectation, which enables stable training; see \Cref{sec:further_brw}.
	
	The generative process described in~\Cref{sec:CTMC_timereversal} is associated with the trained generator specified in~\eqref{eq:generator_generative_ctmc} with
	\begin{align}\label{eq:control_generative_brw}
		\uBRWstar_{t_k}(x,y|\baXbrwstar_{t_k}) =
		\begin{cases}
			\uBRWstar_{t_k}(\baXbrwstar_{t_k}, \sigma(\baXbrwstar_{t_k})) \quad &\text{if $y = \sigma(x)$ for $\sigma \in \mcs$} \eqsp, \\
			\qquad \qquad  1 &\text{otherwise} \eqsp. 
		\end{cases}
	\end{align}
	The pseudo-code for sampling $(\baXbrwstar_t)_{t\in [0,\Tf]}$ is provided in \Cref{alg:brw}, Appendix \ref{sec:ddms-algorithms} for completeness.

	
	\section{Main results}\label{sec:main_results}
	This section provides quantitative error estimates for using the generative process to approximate the backward evolution and recover the data distribution in the aforementioned settings. To this end, we establish conditions ensuring the performance of the approximate score and mild regularity assumptions on the data distribution. While classical diffusion models typically rely on an $\mrl^2$-type approximation error condition, our analysis naturally leads to an entropic-type condition, reflecting the discrete nature of the state space. Moreover, the regularity assumptions we impose on the data distribution are minimal and, in some cases, can be further relaxed by employing an early stopping strategy. This bridges a gap in previous theoretical works, which often required significantly stronger assumptions.
	
	\subsection{Masked diffusion on \texorpdfstring{$\Z^d_m$}{Zdm}}
	
	We first impose a regularity condition on the data distribution $\mustar$ to ensure the validity of the subsequent computations:
	
	\begin{assumptionM}\label{ass:finite_fisher_masked_d}
		The data distribution $\mustar$ has full support on $\Z^d_m$, \ie, $\mustar(x) \in (0,1)$ for any $x\in \Z^d_m$ and $\mustar(x) = 0$ for $x \in \tZ^d_m \setminus \Z^d_m$.
	\end{assumptionM}
	Note that this assumption is crucial yet reasonable in practice, as we can always consider a smoothed version of the data distribution.
	Under~\Cref{ass:condition_on_beta} and~\Cref{ass:finite_fisher_masked_d}, $\mum_t$ is fully supported on $\tZ^d_m$ for any $t \in (0,\Tf]$; see \Cref{lem:formula_u_masked_d} for completeness.
	We now state an assumption on the approximation performance of approximated score. Define for any $a \geq 0$:
	\begin{equation}
		\label{eq:def_h}
		\mathbf{h}(a) \eqdef a\log a -a +1\eqsp,
	\end{equation}
	with the convention $0 \times \log 0 = 0$.
	\begin{assumptionM}\label{ass:approx_score_masked_d}
		There exists $\varepsilon \geq 0$ such that
		\begin{align}
			\sum_{k=0}^{K-1} &h_{k+1} \E \Bigg[ \sum_{\substack{i \in \msm_{\foXm_{\Tf-t_k}} }}\sum_{j\in \Z_m} \Bigg\{ \uMstar_{t_k} \mathbf{h} \l(\frac{\uM_{t_k}}{\uMstar_{t_k}} \r) (\foXm_{\Tf-t_k},\um^{(i)}_j(\foXm_{\Tf-t_k})) \notag \\
			& + \l\| \uM_{t_k} (\foXm_{\Tf-t_k},\um^{(i)}_j(\foXm_{\Tf-t_k})) - \uMstar_{t_k} (\foXm_{\Tf-t_k},\um^{(i)}_j(\foXm_{\Tf-t_k})) \r\|^2 \Bigg\} \Bigg] \leq \varepsilon^2 \Tf \eqsp. \label{assumptiononscore_masked_d}
		\end{align}
	\end{assumptionM}
	
	
	\begin{theorem}\label{theo:5_masked_d}
		Assume~\Cref{ass:condition_on_beta},~\Cref{ass:finite_fisher_masked_d} and~\Cref{ass:approx_score_masked_d} hold. For any $\eta \in (0,\Tf)$ and any time discretization $\{t_k\}_{k=0}^K$ such that $t_0=0$ and $t_K = \Tf-\eta$ and start the generative process from $\baXmstar_0 \sim \mathrm{Uniform}(\Z^d_m)\fopm_{0,\Tf}$. Then the following holds
		\begin{align}
			\KL(\mum_\eta|\mathrm{Law}(\baXmstar_{\Tf-\eta}))
			&\lesssim \underbrace{d \alpha_{\Tf}(1+ \log (m/{\alpha_{\Tf}}))}_{\text{initial error}}  + \underbrace{\rme^h \varepsilon^2 \Tf}_{\text{approx. error}} \notag \\
			&\quad + \underbrace{hd \l[\alpha_\eta/(1-\alpha_\eta) + m \r] + (\rme^h-1) dm\Tf}_{\text{discretization error}}\eqsp,\label{eq:bound_convergence_masked_d}
		\end{align}
		where $h= \max_k\{ t_{k+1}-t_k\}$ and $\alpha_t = \exp ( - \int_0^t \beta(s) \rmd s )$.
	\end{theorem}
	\begin{proof}[Proof of \Cref{theo:5_masked_d}]
		It is deferred to \Cref{sec:masked_proof_conv}.
	\end{proof}
	
	Notably,  $\uM_{t}$ is not defined at $t=\Tf$ (see~\eqref{eq:discrete score_masked_d}), as the data is only supported on the original state space $\Z^d_m$. Consequently, the early-stopping scheme is essential in the masking setting. 
	
	\begin{theorem}\label{theo:main_masked_d}
		Under conditions and notations of~\Cref{theo:5_masked_d}, it holds
		\begin{align}
			&\tvnorm{\mustar - \mathrm{Law}(\baXmstar_{\Tf-\eta})} \lesssim \underbrace{1-\alpha_\eta^d}_{\text{early stopping error}}\\
			&\quad +\sqrt{d \alpha_{\Tf}(1+ \log (m/{\alpha_{\Tf}})) + \rme^h \varepsilon^2\Tf + hd [\alpha_\eta/(1-\alpha_\eta) + m] +  (\rme^h-1)dm\Tf } \eqsp.
		\end{align}
	\end{theorem}
	\begin{proof}[Proof of \Cref{theo:main_masked_d}]
		Note that by definition of the total variation norm \begin{align}
			\tvnorm{\mum_{\eta}-\mustar} \leq \P( \foXm_\eta \neq \foXm_0) = 1 - \P(\foXm_\eta = \foXm_0) = 1-\alpha_\eta^d \eqsp.
		\end{align}
		This together with \Cref{theo:5_masked_d} and the triangle and Pinsker inequalities yield the desired conclusion.
	\end{proof}
	
	Following \cite[Theorem 3]{conforti2025kl}, a tighter bound on the discretization error can be obtained by choosing an appropriate sequence of step sizes, resulting in a logarithmic rather than linear dependence on the discrete Fisher information.
	
	\begin{theorem}\label{theo:exp_stepsize_masked}
		Assume~\Cref{ass:finite_fisher_masked_d} and~\Cref{ass:approx_score_masked_d} hold and consider the constant generator $\beta(t) = 1$ for all $t \in [0,\Tf]$. For any $\eta \in (0,\Tf)$, let $c \in (0,1/2]$, $\Tf-\eta \geq 1+2c$,
		and
		$L_\eta =  d^{-1} \Ic (\mum_\eta) \geq 2$.
		Choose
		$h_{k+1} = c \min \l\{ \max \l\{ \Tf-\eta-t_k, 1/L_\eta \r\}, 1 \r\}$
		for $k < K-1$ and $h_K = \Tf-\eta-t_{K-1}$, draw $\baXmstar_0 \sim \mathrm{Uniform}(\Z^d_m)\fopm_{0,\Tf}$, we then have
		\begin{align}
			&\tvnorm{\mustar - \mathrm{Law}(\baXmstar_{\Tf-\eta})} \notag \\
			&\qquad \lesssim {d\eta}  + \sqrt{  {d\log(m)\rme^{-\Tf/2}} + {\varepsilon^2\Tf} + {(\rme^c-1)dm [\log(\eta^{-1}+m) +\Tf]}} \eqsp. \label{eq:bound_convergence_scaled_masked}
		\end{align}
		In particular, choosing
		\begin{align}\label{eq:cor_T_eta_masked_d}
			\Tf = 2 \log \frac{d\log(m)}{\varepsilon^2} \quad  \text{and} \quad \eta = \frac{\varepsilon}{d}\eqsp,
		\end{align}
		then setting the maximum step-size as
		\begin{align}\label{eq:cor_c_masked_d}
			c =  \log \l( \frac{\varepsilon^2}{dm \l[ 2\log(d\log(m)/\varepsilon^2)+\log(m+d/\varepsilon) \r]}
			+1 \r) \eqsp,
		\end{align}
		imply the error $\tilde{O}(\varepsilon)$ 
		and is associated with the number of iterations $\tilde O(dm/\varepsilon^2)$, where the notation $\tilde O$ means that logarithmic factors of $d,m, \varepsilon$ have been dropped.
	\end{theorem}
	\begin{proof}[Proof of~\Cref{theo:exp_stepsize_masked}]
		The proof benefits from the choice of the step-size's scheme and is postponed to~\Cref{sec:masked_proof_conv}.
	\end{proof}
	
	\begin{remark}
		The tighter bound above is obtainable not only for the constant generator but also for any function $\beta$ satisfying \Cref{ass:condition_on_beta}, since the bound on the Fisher information holds universally (see \Cref{lem:bound_fisher_masked}), and the choice of step sizes can be adapted to each specific case.
		Notably, \Cref{theo:exp_stepsize_masked} provides the {first non-asymptotic error bound} for masked diffusion models employing an early-stopping strategy, which does not rely on the boundedness of the estimated score as in \cite{liang2025absorb}.
	\end{remark}
	
In the previous results, we assumed that the generative process is initialized from $\mathrm{Uniform}(\Z^d_m)\fopm_{0,\Tf}$. In practice, however, it is typically started from $\updelta_{\text{MASK}}^{\otimes d}$. We can also derive bounds for this initialization, as stated in the next result.
	
	\begin{theorem}\label{theo:start_delta_mask}
		Assume~\Cref{ass:condition_on_beta},~\Cref{ass:finite_fisher_masked_d} and~\Cref{ass:approx_score_masked_d} hold. Let $\eta \in (0,\Tf)$ and consider any time discretization $\{t_k\}_{k=0}^K$ with $t_0 = 0$ and $t_K = \Tf - \eta$. Let $(\baYmstar_t)_{t \in [0,\Tf-\eta]}$ denote the generative process driven by the same generator as $(\baXmstar_t)_{t \in [0,\Tf-\eta]}$, but initialized from $\baYmstar_0 \sim \updelta_{\text{MASK}}^{\otimes d}$. Then the following holds:
		\begin{align}
			&\tvnorm{\mustar - \mathrm{Law}(\baYmstar_{\Tf-\eta})} \lesssim 1-\alpha_\eta^d +d\alpha_{\Tf}\\
			&+\sqrt{d \alpha_{\Tf}(1+ \log (m/{\alpha_{\Tf}}))+ hd [\alpha_\eta/(1-\alpha_\eta) + m ] + \rme^h \varepsilon^2\Tf + (\rme^h-1)dm \Tf } \eqsp,
		\end{align}
		where $h = \max_k \{t_{k+1}-t_k \}$. Additionally, considering the sequence of step-sizes as in \Cref{theo:exp_stepsize_masked} yields
		\begin{align}
			&\tvnorm{\mustar - \mathrm{Law}(\baYmstar_{\Tf-\eta})} \notag \\
			&\quad \lesssim {d\eta} + {d\rme^{-\Tf}} + \sqrt{  {d\log(m)\rme^{-\Tf/2}} + {\varepsilon^2\Tf} + {(\rme^c-1)dm [\log(\eta^{-1}+m) +\Tf]}} \eqsp. 
		\end{align}
		In particular, choosing
		\begin{align}
			\Tf = \max \l\{ \log \frac{d}{\varepsilon} \eqsp, \eqsp 2 \log \frac{d\log(m)}{\varepsilon^2}\r\} \quad  \text{and} \quad \eta = \frac{\varepsilon}{d} \eqsp,
		\end{align}
		then setting the maximum step-size
		\begin{align}
			c =  \log \l( \frac{\varepsilon^2}{dm[\Tf+\log(m+d/\varepsilon)] } +1 \r) \eqsp,
		\end{align}
		imply the error $\tilde{O}(\varepsilon)$
		and is associated with the number of iterations $\tilde O(dm/\varepsilon^2)$.
	\end{theorem}
	\begin{proof}[Proof of \Cref{theo:start_delta_mask}]
        It is given in Appendix \ref{sec:proof_theo:start_delta_mask}.
	\end{proof}
	
	\subsection{Random walk on \texorpdfstring{$\Z^d_m$}{Zdm}}
	
	To bound the error of our DDM, we introduce assumptions concerning the accuracy of the training phase and the regularity properties of the data distribution:
	
	\begin{assumptionRW}  \label{ass:approx_score_randomwalk}
		There exists $\varepsilon \geq 0$ such that
		\begin{align}
			\sum_{k=0}^{K-1} h_{k+1} &\E \Bigg[ \sum_{\sigma \in \mcs}  \uRWstar_{t_k} \mathbf{h}\l( \frac{\uRW_{t_k}}{\uRWstar_{t_k}}\r)  (\foXrw_{\Tf-t_k},\sigma(\foXrw_{\Tf-t_k})) \Bigg]  \leq \varepsilon^2 \Tf \eqsp, \label{assumptiononscore_randomwalk}
		\end{align}
		where $\mcs$ is given in \eqref{eq:def_mc} and $\mathbf{h}$ is defined in \eqref{eq:def_h}.
	\end{assumptionRW}
	This condition naturally appears as we bound the $\KL$ divergence of path probability measures corresponding to the approximate discrete score $(\uRWstar_t)_{t\in [0,\Tf)}$ and the true one $(\uRW_t)_{t\in [0,\Tf)}$ respectively. While standard Girsanov theorem for diffusion implies an $\mrl^2$-type approximation error condition for generative models \cite{chen2022sampling}, our results naturally involve the entropic-type condition due to the discrete structure of our noising process.
	
	\begin{assumptionRW}\label{ass:finite_fisher_randomwalk}
		The data distribution has full support on $\Z^d_m$, \ie, $\mustar(x) \in (0,1) $ for any $x\in \Z^d_m$.
	\end{assumptionRW}
	
	In fact, \Cref{ass:finite_fisher_randomwalk} is equivalent to requiring finiteness of the discrete Fisher information:
	\begin{align}\label{eq:finite_fisher_randomwalk}
		\Ic(\mustar):=\E\l[\sum_{\sigma \in \mcs} \mathbf{h}\l( \frac{\mustar (\sigma(\foXrw_0))}{\mustar(\foXrw_0)}\r)\r]< \infty \eqsp.
	\end{align}
	Here the discrete score $\uRW_t$ is well-defined at $\Tf$ thanks to~\Cref{ass:finite_fisher_randomwalk}; see~\eqref{eq:formula_u} for completeness.
	
	Note that~\Cref{ass:finite_fisher_randomwalk} parallels the finite relative Fisher information condition introduced in \cite{conforti2025kl}. Nevertheless,~\Cref{ass:finite_fisher_randomwalk} is considerably simpler: the state space is finite, and the function $\mathbf{h}$ becomes infinite only when $\mustar$ lacks full support. Moreover, this assumption can later be further relaxed via the early stopping strategy, so that only  \Cref{ass:approx_score_randomwalk}  is needed to ensure the convergence of DDMs.
	
	
	\begin{theorem}\label{theo:5_randomwalk}
		Assume~\Cref{ass:approx_score_randomwalk} and~\Cref{ass:finite_fisher_randomwalk}, choose any time discretization $\{t_k\}_{k=0}^K$ such that $t_0=0$ and $t_K=\Tf$. Then the following holds
		\begin{align}\label{eq:bound_convergence_randomwalk}
			\KL(\mustar| \mathrm{Law}(\baXrwstar_{\Tf}) ) \lesssim \underbrace{\rme^{-\frac{16\pi^2}{25m^2}\Tf}\KL(\mustar|\rw)}_{\text{initial error}}+ \underbrace{h \Ic(\mustar)}_{\text{discretization error}}+ \underbrace{ \varepsilon^2\Tf }_{\text{approximation error}} \eqsp,
		\end{align}
		where $h:=\max_k \{t_{k+1}-t_k\}$. Furthermore, \eqref{eq:bound_convergence_randomwalk} still holds if we replace $\KL(\mustar|\rw)$ by $d\log (m)$.
	\end{theorem}
	
	\begin{proof}[Proof of~\Cref{theo:5_randomwalk}]
		The proof is postponed to~\Cref{sec:rw_brw_proof_conv}.
	\end{proof}
	
	\begin{remark}
		When the data distribution coincides with the invariant measure $\mathrm{Uniform}(\Z^d_m)$, we obtain $\Ic(\mustar)=0$, which in turn cancels the discretization error.
	\end{remark}

	Next, we employ an exponentially decreasing sequence of step sizes to obtain a tighter bound on the discretization error.
	\begin{theorem}\label{theo:scale_rw}
		Let $c \in (0,1/2]$ and $\Tf \geq 1+2c$. Suppose~\Cref{ass:approx_score_randomwalk},~\Cref{ass:finite_fisher_randomwalk} hold and let
		$L =  d^{-1} \Ic (\mustar) \geq 2$.
		Choose
		$h_{k+1} = c \min \l\{ \max \l\{ \Tf-t_k, 1/L \r\}, 1 \r\}$ for $k < K-1$ and $h_K = \Tf-t_{K-1}$, we then have that
		\begin{align}
			\label{eq:bound_convergence_scaled_rw}
			\KL(\mustar|\mathrm{Law}(\baXrwstar_{\Tf}))
			&\lesssim  \rme^{-\frac{16\pi^2}{25m^2}\Tf}d\log(m)+  \varepsilon^2\Tf + {cd \log(m)} \log (L)  \eqsp.
		\end{align}
		In particular, choosing the time horizon $\Tf$ and the maximum step-size $c$ as
		\begin{align}\label{eq:cor_randomwalk_1}
			\Tf = \frac{25m^2}{16\pi^2} \log \frac{d\log(m)}{\varepsilon^2} \quad \text{and} \quad  c = \frac{\varepsilon^2}{d \log (m)\log (L)} \eqsp,
		\end{align}
		implies the number of iterations
		\begin{align}
			K \lesssim {d \log (m)\log (L)[m^2\log(d\log(m)/\varepsilon^2)+\log (L)]}/{\varepsilon^2} \eqsp,
		\end{align}
		and makes the approximation error $\tilde O(m^2 \varepsilon^2)$.
	\end{theorem}
	\begin{proof}[Proof of \Cref{theo:scale_rw}]
		The proof of~\Cref{theo:scale_rw} benefits from the choice of the step-size's scheme and is postponed to~\Cref{sec:rw_brw_proof_conv}.
	\end{proof}

	
	In the next result, we drop \Cref{ass:finite_fisher_randomwalk} by employing an early stopping strategy and using that $\mu_t$ has support $\zset^d_m$ for any $t \in (0,\Tf]$.
	
	
	
	
	\begin{theorem}\label{theo:early_stopping_rw}
		Assume~\Cref{ass:approx_score_randomwalk} holds. For $\eta \in (0,\Tf)$, let $c \in (0, 1/2]$, $\Tf -\eta \geq 1+2c$. Set $L_\eta = d^{-1} \Ic(\murw_\eta) \geq 2$. Choose
		$h_{k+1} = c \min \l\{ \max \l\{ \Tf-\eta-t_k, 1/L_\eta \r\}, 1 \r\}$ for $k < K-1$ and $h_K = \Tf-\eta-t_{K-1}$, then the following holds
		\begin{align}
			\KL(\murw_\eta|\mathrm{Law}(\baXrwstar_{\Tf-\eta}))
			&\lesssim  \rme^{-\frac{16\pi^2}{25m^2}\Tf}d\log(m) +  \varepsilon^2\Tf + {cd \log(m)} \log(\eta^{-1}\log(m)) \eqsp.
		\end{align}
		As a consequence, the total variation distance between $\mustar$ and the generated distribution is given by
		\begin{align}
			&\tvnorm{\mustar - \mathrm{Law}(\baXrwstar_{\Tf-\eta})} \\
			&\quad \lesssim \underbrace{d\eta }_{\text{early stopping error}} + \sqrt{{\rme^{-\frac{16\pi^2}{25m^2}\Tf} d\log(m)} +{ \varepsilon^2\Tf}+ { cd\log (m) \log(\eta^{-1}\log(m))}} \eqsp.
		\end{align}
		In particular, choosing
		\begin{align}\label{eq:cor_randomwalk_2}
			\eta &= \frac{\varepsilon}{d} \eqsp, \quad  c = \frac{\varepsilon^2}{d \log (m) \log (d \log(m)/\varepsilon^2)} \eqsp, \quad  \Tf =  \frac{25m^2}{16\pi^2} \log \frac{d\log(m)}{\varepsilon^2}
			\eqsp,
		\end{align}
		implies the error $\tilde O(m\varepsilon)$
		and the number of iterations $\tilde O(dm^2/\varepsilon^2) $.
	\end{theorem}
	
	\begin{proof}[Proof of~\Cref{theo:early_stopping_rw}]
        It is provided in Appendix \ref{sec:proof_theo:early_stopping_rw}.
	\end{proof}

	\subsection{Biased random walk on \texorpdfstring{$\N^d$}{Nd}}
	
	In the same spirit as the previous settings, to control the resulting approximation error, we introduce the following assumptions concerning the accuracy of the approximation scheme and the regularity of the data distribution:
	\begin{assumptionBRW}  \label{ass:approx_score_brw}
		There exists $\varepsilon \geq 0$ such that
		\begin{align}\label{eq:ass_kl_loss_brw}
			\sum_{k=0}^{K-1}\int_{t_k}^{t_{k+1}} \sum_{\sigma \in \mcs}  \E  \l( \uBRW_{t} \foqbrw_t \log  \frac{\uBRW_{t_k}}{\uBRWstar_{t_k}}  - (\uBRW_{t_k} - \uBRWstar_{t_k}) \foqbrw_{t_k} \r)(\sigma) \rmd t  \leq \varepsilon^2 \Tf
		\end{align}and
		\begin{align}\label{eq:ass_quad_loss_brw}
			\sum_{k=0}^{K-1}  h_{k+1}  \E \Bigg[ \sum_{\sigma \in \mcs} \l\| (\uBRW_{t_k} - \uBRWstar_{t_k}) \foqbrw_{t_k}(\sigma)  \r\|^2 \Bigg] \leq \varepsilon^2\Tf 
			\eqsp,
		\end{align}
		where $\uBRW_t(\sigma) \eqdef \uBRW_t(\foXbrw_{\Tf-t}, \sigma(\foXbrw_{\Tf-t}))$ and other terms are defined similarly. 
	\end{assumptionBRW}
	
	\begin{remark}
		Minimizing $\loss_{\mathrm{e}}(\theta)$ given in \eqref{eq:loss_entropy_function} is equivalent to minimizing
		\begin{align}
			\tilde{\loss}_{\mathrm{e}} (\theta) \eqdef \sum_{k=0}^{K-1} \int_{t_k}^{t_{k+1}} \sum_{\sigma \in \mcs}  \E  \l(\uBRW_{t}\foqbrw_t\log\frac{\uBRW_{t_k}}{\uBRWtheta_{t_k}} - (\uBRW_{t_k} - \uBRWtheta_{t_k}) \foqbrw_{t_k}  \r)(\sigma) \rmd t \eqsp,
		\end{align}
		since they differ only by an additive constant independent of $\theta$. In fact, a good approximation $\uBRWstar_{t_k}$ should be sufficiently close to $\uBRW_{t_k}$, which makes $\tilde{\loss}_{\mathrm{e}}(\theta^\star)$ nearly zero.
	\end{remark}
	
	\begin{assumptionBRW}\label{ass:finite_fisher_brw}
		The data distribution has finite second order moment and finite discrete Fisher information, \ie, $\m^\star_2:=\m_2(\mustar) =\sum_{x\in \N^d} \|x\|^2 \mustar( x) < \infty$ and
		\begin{align}\label{eq:finite_fisher_brw}
			\Ic_{\brw}(\mustar) \eqdef \E \l[ \sum_{\sigma \in \mcs} \mathbf{h}\l(\frac{\tmustar(\sigma(\foXbrw_0))}{\tmustar(\foXbrw_0)} \r) \foqbrw  (\foXbrw_0, \sigma(\foXbrw_0))  \r]< \infty \eqsp,
		\end{align}
		where $\tmustar \eqdef \mustar/\brw$ represents the relative marginal density.
	\end{assumptionBRW}
	Note that the requirement $\m_2^\star <\infty$ coincides with the minimal condition required on the data distribution in the continuous setting (see \cite[Corollary 1]{conforti2025kl}). 
	
	
	
	
	\begin{theorem}\label{theo:scale_brw}
		Let $c \in (0,1/2]$ and $\Tf \geq 1+2c$. Suppose~\Cref{ass:approx_score_brw} and~\Cref{ass:finite_fisher_brw} hold and assume that  $L =  d^{-1} \Ic_{\brw} (\mustar) \geq 2$.
		Choose
		$h_{k+1} = c \min \l\{ \max \l\{ \Tf-t_k, 1/L \r\}, 1 \r\}$ for $k < K-1$ and $h_K = \Tf-t_{K-1}$, then it holds
		\begin{align}
			&\KL(\mustar|\mathrm{Law}(\baXbrwstar_{\Tf} ) ) \notag \\
			&\quad \lesssim \underbrace{\rme^{-\Tf} \KL(\mustar|\brw)}_{\text{initialization error}} + \underbrace{\varepsilon^2\Tf }_{\text{arrox. error}} +\underbrace{ (\rme^c-1) [d\Tf + (d+\mstar_2)\log (L)] }_{\text{discretization error}} \eqsp. \label{eq:bound_convergence_brw}
		\end{align}
		Furthermore, \eqref{eq:bound_convergence_brw} still holds if we replace $\KL(\mustar|\brw)$ by $d+\mstar_2$.
		In particular, setting the time horizon $\Tf$ and the maximum step-size $c$ as
		\begin{align}\label{eq:cor_brw_scale}
			\Tf = \log \frac{d+\mstar_2}{\varepsilon^2} \eqsp,  \quad  c =  \log \l( \frac{\varepsilon^2}{ d\log((d+\mstar_2)/\varepsilon^2) + (d+\mstar_2)\log(L)} +1\r) \eqsp,
		\end{align}
		implies the number of iterations
		\begin{align}
			K \lesssim \frac{\log((d+\mstar_2)/\varepsilon^2) + \log (L)}{\log(1+ \varepsilon^2/[d\log((d+\mstar_2)/\varepsilon^2) + (d+\mstar_2)\log(L) ] )}  \eqsp,
		\end{align}
		and makes the approximation error $\tilde O(\varepsilon^2)$.
	\end{theorem}
	\begin{proof}[Proof of \Cref{theo:scale_brw}]
		It is deferred to~\Cref{sec:rw_brw_proof_conv}.
	\end{proof}
	\begin{remark}
		The bound \eqref{eq:bound_convergence_brw} remains valid for any time discretization $\{t_k\}_{k=0}^K$ satisfying $t_0=0$ and $t_K = \Tf$. However, the term $(\rme^c-1)[d+\m_2(\mu^\star)]\log(L)$ is then replaced by $c\mathcal{I}_{\brw}(\mustar)$, where $c=\max_k h_k$.
	\end{remark}
	
	We now remove the finite Fisher information assumption by employing an early stopping strategy as follows.
	\begin{theorem}\label{theo:early_stopping_brw}
		Assume~\Cref{ass:approx_score_brw} holds and suppose that the data distribution has finite second order moment, \ie, $\mstar_2 <\infty$.  For $\eta \in (0,\Tf)$, let $c \in (0, 1/2]$, $\Tf -\eta \geq 1+2c$ and $L_\eta = d^{-1} \Ic_{\brw}(\mubrw_\eta) \geq 2$. Choose
		$h_{k+1} = c \min \l\{ \max \l\{ \Tf-\eta-t_k, 1/L_\eta \r\}, 1 \r\}$ for $k < K-1$ and $h_K = \Tf -t_{K-1}$, then it holds
		\begin{align}
			&\KL(\mubrw_\eta|\mathrm{Law}(\baXbrwstar_{\Tf-\eta} ) ) \\ &\quad \lesssim \rme^{-\Tf} (d+\mstar_2) + \varepsilon^2\Tf + (\rme^c-1)\l[ d\Tf + (d+\mstar_2)\log (\eta^{-1}(1+\mstar_2 d^{-1})) \r]   \eqsp.
		\end{align}
		Consequently, the total variation distance between $\mustar$ and the generated distribution admits the following upper bound:
		\begin{align}
			&\quad\tvnorm{\mustar - \mathrm{Law}(\baXbrwstar_{\Tf-\eta})}  \lesssim \underbrace{\eta (d+\mstar_1)}_{\text{early stopping error}} \notag \\
			& \qquad +\sqrt{
				\rme^{-\Tf} (d + \mstar_2) + \varepsilon^2\Tf + (\rme^c-1)\l[   d\Tf + (d+\mstar_2)\log (\eta^{-1}(1+\mstar_2 d^{-1})) \r]} \label{eq:conv_early_stopping_brw}
		\end{align}
		where $\mstar_1 \eqdef \m_1(\mustar) = \sum_{x\in \N^d} \sum_{i=1}^d x^i \mustar(x)$ denotes the first-order moment of $\mustar$.
		Moreover, choosing \begin{align}\label{eq:cor_early_stopping_brw_1}
			\eta = \frac{\varepsilon}{d+\mstar_1} \quad \text{and} \quad
			\Tf = \log \frac{d+\mstar_2}{\varepsilon^2} \eqsp,
		\end{align}
		and setting the maximum step-size 
		\begin{align}\label{eq:cor_early_stopping_brw_2}
			c =  \log \l( \frac{\varepsilon^2}{ d\Tf + (d+\mstar_2)\log (\eta^{-1}(1+\mstar_2 d^{-1}))} +1\r)
		\end{align}
		imply the error $\tilde O(\vareps)$
		and the number of iterations $\tilde O((d+\m^\star_2)/\varepsilon^2)$.
	\end{theorem}
	\begin{proof}[Proof of \Cref{theo:early_stopping_brw}]
        It is given in Appendix \ref{sec:proof_theo:early_stopping_brw}.
	\end{proof}
	


	\section{Related works}
	\label{app:related_works}
	This section provides details of recent researches on DDMs.
	
	\paragraph*{Analysis of Discrete diffusion models}
	There have been plenty of studies of diffusion models tailored for discrete data.
	\cite{hoogeboom2021argmax} proposed Argmax Flows and Multinomial Diffusion for categorical data. Argmax Flows linked discrete data to continuous models via an argmax with a probabilistic inverse, while Multinomial Diffusion added categorical noise and trained a model to reverse it.
	\cite{dieleman2022continuous} and \cite{chen2022analog} embedded discrete data in Euclidean space, while \cite{richemond2022categorical} used the simplex, all leveraging continuous forward diffusion models.
	Later, \cite{campbell2022continuous} introduced a continuous-time framework for discrete denoising diffusion using CTMCs and \cite{gat2024discrete} further added a correction step to bring the sample distribution closer to the desired one, but their general approach relies on ELBO-based marginal learning and costly correction steps, making it less efficient in high dimension.
	Another recent approach to handle discrete data is generative modeling with arbitrary Markov processes using generator matching, introduced by \cite{holderrieth2024generator}. This method is flexible and can be applied to various state spaces, particularly in discrete settings.
	Recently, in \cite{anonymous2025countsdiff}, the authors introduced CountsDiff, a diffusion model tailored to
	$\N^d$-supported data. The forward noising dynamics are given by a pure-death chain with transition rates identical to the backward rates discussed in \Cref{sec:brw}. Empirically, CountsDiff matches the performance of state-of-the-art discrete diffusion methods, emphasizing the need for dedicated approaches to $\N^d$-valued data.
	Nonetheless, despite these promising results, all of the aforementioned models came with limited theoretical justification.

	\paragraph*{Masked diffusion models}
	One important step toward more advanced models is the ``masked'' diffusion process, a discrete diffusion approach first introduced by \cite{austin2021structured}. Later, \cite{shi2024simplified} looked into this model further, simplifying its training objective by expressing it as a signal-to-noise ratio, which helps highlight some useful features. However, despite these improvements, the model still lacks theoretical guarantees. \cite{sahoo2024simpleeffectivemaskeddiffusion} improved upon this direction by leveraging the structure of the absorbing kernel and refining the bridge-based reverse process, leading to more efficient optimization. The model’s reliance on absorbing-state approximations and heuristic training objectives limits its theoretical grounding.
	


	
	\paragraph*{Convergence results of discrete diffusion models}
	\cite{ren2024discrete} analyzed discrete diffusion models using L\'evy-type stochastic integrals and Poisson random measures, deriving integral expressions for the noising and denoising processes of categorical data. They also introduced a unified error analysis and established the first KL-divergence bound for the $\tau$-leaping algorithm, though under strong score assumptions (continuity and boundedness). \cite{chen2024convergence} provided explicit error bounds for hypercube sampling with a uniform rate matrix, achieving near-linear iteration complexity (in expectation) under bounded score and score-entropy assumptions. Next, \cite{pham2025discrete} achieved to establish (worst case) computational complexity that scales linearly (up to logarithmic factors) with the dimension, under only statistical assumption on the score approximation. One key improvement in \cite{pham2025discrete} comes from avoiding the Kolmogorov equation at each iteration and instead using flip rates together with a Poisson clock.
	Complementing CTMC-based approaches, \cite{bach2025samplingbinarydatadenoising} introduced a fully discrete denoising model using Bernoulli corruption as a Gaussian analogue, yielding a Langevin-style sampler on the hypercube with rigorous guarantees. However, their works remain restricted to the hypercube.
	For categorical data, \cite{zhang2024convergence}, \cite{liang2025discrete} and \cite{liang2025absorb} established convergence guarantees under uniform and absorbing rate matrices, respectively, paralleling our results in \Cref{theo:early_stopping_rw} and \Cref{theo:exp_stepsize_masked}. However, in all these works, the authors suppose that the score or its approximation is bounded uniformly in time, conditions that do not hold in most cases.  In contrast, our bounds require no assumptions on the score and hold under a simple statistical condition, which is also employed in the aforementioned literature.
	Furthermore, under stronger conditions than ours,  \cite{liang2025discrete} only obtain a computational complexity that grows quadratically with the dimension, our results scale linearly (up to logarithmic factors).
	Nonetheless, none of the above works establish theoretical foundations for DDMs on counted data, leaving this aspect unexplored.

	To conclude, this paper bridges existing gaps by explicitly formulating the forward Markov process. By deriving the conditional expectation expression of the score, we can reduce the computationally expensive signal-to-noise ratio training as used in \cite{shi2024simplified}. As a result, we obtain a simpler and more efficient training framework based on the $\mrl^2$-loss and entropy-based loss, supported by rigorous {non-asymptotic} convergence guarantees.
	The core of our results lies in the monotonicity of the score along the backward dynamics. Importantly, our analysis applies to both absorbing and uniform rate matrices and extends beyond finite to countably infinite state spaces, thereby highlighting the robustness and generality of our approach.
	
	
	

	
	\section{Proofs}\label{sec:main_proofs}
	
	To derive our results, our  approach for all considered models can be summarized as follows: 
	\begin{itemize}
		\item Characterizing the discrete score via a Hamilton–Jacobi–Bellman (HJB) equation derived from the forward Kolmogorov equation. Notably, an alternative perspective to obtain the HJB formulation comes from stochastic optimal control.
		\item Establishing the monotonicity of the score along the backward dynamics (\Cref{prop:3_rw}, \Cref{prop:monotonicity_masked_d}, \Cref{prop:3_brw}) by leveraging It\^o’s formula together with the HJB equation.
		\item Applying Girsanov’s theorem for jump processes to compute the KL divergence between the generated and true data distributions, then simplifying the resulting expression using the score monotonicity. Notably, the integrability condition in Girsanov's theorem holds naturally in the models we study, as it can be directly controlled under our assumptions. Similarly to continuous diffusion models, the overall error can be decomposed into three components: (1) initialization error, (2) approximation error, and (3) discretization error. Importantly, leveraging score monotonicity allows us to derive simple and rigorous error bounds without imposing the restrictive assumptions on the data distribution or estimated score required by previous studies.
	\end{itemize}
	
	
	\subsection{Masked diffusion on \texorpdfstring{$\Z^d_m$}{Zdm}}\label{sec:main_proofs_masked}
	
	\subsubsection{Evolution of the score along the backward dynamic}
	The evolution of the score plays a crucial role in the convergence guarantees, as it enables control of the discretization error. To analyze how this error propagates, we rely on It\^o’s formula together with the characterization provided in the following lemma. 
	\begin{lemma}\label{lem:formula_u_masked_d}
		Under~\Cref{ass:condition_on_beta} and~\Cref{ass:finite_fisher_masked_d}, the time reversal process $(\baXm_t)_{t\in [0,\Tf]}$ has a backward generator $(\baqm_t)_{t\in [0,\Tf)}$ of the form \eqref{eq:backward_generator_ctmc}, with $\tqm_t(x,y) \eqdef \foqm_{\Tf-t}(y,x)$ for any $t \in [0,\Tf]$ and 
		for $(t,x,y) \in [0,\Tf) \times \tZ^d_m \times \tZ^d_m$,
		\begin{equation}\label{eq:discrete score_masked_d}
			\uM_t(x,y) = \begin{cases}
				\quad{\mum_{\Tf-t}(y)}/{\mum_{\Tf-t}(x)}  \quad &\text{if } \exists i \in  \msm^{\complement}_y: \eqsp x=\mrm^{(i)}(y) \eqsp,\\
				-{\sum_{y \neq x} \uM_t(x,y) \tqm_t (x,y)}/{\tqm_t(x,x)} \quad &\text{if } x=y \eqsp,\\
				\qquad \quad 1  \quad &\text{otherwise}\eqsp.
			\end{cases}
		\end{equation}
		In addition, $\uM_t$ given above can be expressed as
		\begin{equation}
			\uM_t (x,\um^{(i)}_j(x))= \rme^{\Vm_t(x) - \Vm_t(\um^{(i)}_j(x))} \quad \text{for }  i \in \msm_x \eqsp, \eqsp j \in \Z_m \eqsp,
		\end{equation}
		with $\Vm_t(x) \eqdef -\log \mum_{\Tf-t}(x)$. Furthermore setting
		\begin{equation}
			f^{(i),j}(t,x):= \uM_t(x, \um^{(i)}_j(x))\1_{i \in \msm_{x}} \eqsp,
		\end{equation}
		we obtain the following equation: for $(t,x) \in [0,\Tf) \times \tZ^d_m \setminus \Z^d_m$ and $j \in \Z_m$, $i \in \msm_x$,
		\begin{align}\label{eq:hjb_masked_d}
			\partial_t f^{(i),j}(t,x) + (\baqm f^{(i),j})(t,x)  &= \beta(\Tf-t) f^{(i),j}(t,x) \eqsp.
		\end{align}
	\end{lemma}
	\begin{proof}[Proof of~\Cref{lem:formula_u_masked_d}]
		Under~\Cref{ass:condition_on_beta} and~\Cref{ass:finite_fisher_masked_d}, for any $(t,x) \in (0,\Tf] \times \tZ^d_m$, we have $\mum_t(x) >0$.
		Indeed, if $x \in \Z^d_m$, $\mum_t(x) \geq  \mustar(x)\alpha_t^d >0$, otherwise, if $x \in \tZ^d_m \setminus \Z^d_m$, we have
		\begin{align}
			\mum_t(x) = \sum_{y\in \Z^d_m} \mustar(y) \fopm_{0,t}(y,x) \geq \mustar(z_x) \alpha_t^{d-|\msm_{x}|}(1-\alpha_t)^{|\msm_x|} >0  \eqsp,
		\end{align}
		where $z_x^i= x^i$ for any $i \in \msm_x^{\complement}$ and $z_x^j = 0$ for any $j \in \msm_x$.
		Thus~\eqref{eq:backward_generator_ctmc} implies the following for any $t\in [0,\Tf)$ and $x\neq y \in \tZ^d_m$:
		\begin{equation}\label{eq:mask_baq}
			\baqm_t (x,y) = \frac{\mum_{\Tf-t}(y)}{\mum_{\Tf-t}(x)} \foqm_{\Tf-t} (y,x)
			= \frac{\mum_{\Tf-t}(y)}{\mum_{\Tf-t}(x)} \tqm_{t}(x,y) 
			\eqsp,
		\end{equation}
		where $\tqm_t(x,y) \eqdef \foqm_{\Tf-t}(y,x)$.
		We define the discrete score for $t \in [0,\Tf)$ and $y \neq x \in \tZ^d_m$ as follows:
		\begin{equation}
			\uM_t(x,y) =
			\frac{\mum_{\Tf-t} (y)}{ \mum_{\Tf-t}(x)} \eqsp, \quad \text{if } \exists i \in  \msm^{\complement}_y: \eqsp x=\mrm^{(i)}(y) \eqsp,
		\end{equation}
		otherwise, for $y \neq x$, we impose $\uM_t(x,y)=1$. Then~\eqref{eq:mask_baq} follows
		\begin{equation}
			\baqm_t (x,y) = \uM_t(x,y) \tqm_t(x,y) \eqsp.
		\end{equation}
		Furthermore, the convention
		\begin{equation}
			\uM_t(x,x) = -\frac{\sum_{y \neq x} \uM_t(x,y) \tqm_t (x,y)}{\tqm_t(x,x)}  \eqsp,
		\end{equation}
		implies that $\uM \tqm$ in fact forms a generator and is associated with $(\baXm_t)_{t\in [0,\Tf]}$
		Moreover, the discrete score $\uM_t$ can be characterized by using the function $\Vm_t(x) := -\log \mum_{\Tf-t}(x)$ for $(t,x) \in [0,\Tf) \times \tZ^d_m$ as
		\begin{equation}
			\uM_t (x,\um^{(i)}_j(x))= \rme^{\Vm_t(x) - \Vm_t(\um^{(i)}_j(x))} \quad \text{for }  i \in \msm_x \text{ and } j \in \Z_m \eqsp.
		\end{equation}
		Consequently,
		\begin{align}
			\partial_t \uM_t(x,\um^{(i)}_j(x))
			&= \uM_t (x,\um^{(i)}_j(x)) \l[ \partial_t \Vm_t(x) - \partial_t \Vm_t(\um^{(i)}_j(x)) \r]\\
			&= \uM_t (x,\um^{(i)}_j(x)) \l[ \frac{\partial_t \mum_{\Tf-t}(x)}{ \mum_{\Tf-t}(x)} - \frac{\partial_t  \mum_{\Tf-t}(\um^{(i)}_j(x))}{ \mum_{\Tf-t}(\um^{(i)}_j(x))} \r] \eqsp.
		\end{align}
		By the forward Kolmogorov equation \eqref{eq:foward_kolm}, we get
		\begin{align}
			&\quad \partial_t \uM_t(x,\um^{(i)}_j(x))\\
			&= \uM_t(x,\um^{(i)}_j(x)) \l[ \frac{\sum_{z \in \tZ^d_m} \mum_{\Tf-t}(z) \foqm_{\Tf-t}(z,x)}{ \mum_{\Tf-t}(x)} - \frac{\sum_{z \in \tZ^d_m} \mum_{\Tf-t}(z) \foqm_{\Tf-t}(z,\um^{(i)}_j(x))}{ \mum_{\Tf-t}(\um^{(i)}_j(x))} \r] \\
			&= \uM_t(x,\um^{(i)}_j(x)) \Bigg[ \foqm_{\Tf-t} (x,x) + \frac{\beta(\Tf-t)\sum_{k \in \msm_x} \sum_{n\in \Z_m} \mum_{\Tf-t}(\um^{(k)}_n(x)) }{ \mum_{\Tf-t}(x)} \\
			& - \foqm_{\Tf-t}(\um^{(i)}_j(x),\um^{(i)}_j(x)) - \frac{\beta(\Tf-t)\sum_{k \in \msm_x \setminus \{i\}}\sum_{n\in \Z_m} \mum_{\Tf-t}(\um^{(k)}_n(\um^{(i)}_j(x))) }{ \mum_{\Tf-t}(\um^{(i)}_j(x))} \Bigg]\\
			&=  \uM_t(x,\um^{(i)}_j(x)) \Bigg[ -|\msm_x^{\complement}| \beta(\Tf-t) + \frac{\beta(\Tf-t)\sum_{k \in \msm_x} \sum_{n\in \Z_m} \mum_{\Tf-t}(\um^{(k)}_n(x)) }{ \mum_{\Tf-t}(x)} \\
			& +| \msm^{\complement}_{\um^{(i)}_j(x)} | \beta(\Tf-t) - \frac{\beta(\Tf-t)\sum_{k \in \msm_x \setminus \{i\}}\sum_{n\in \Z_m} \mum_{\Tf-t}(\um^{(k)}_n(\um^{(i)}_j(x))) }{ \mum_{\Tf-t}(\um^{(i)}_j(x))} \Bigg] \eqsp.
		\end{align}
		Now using definition of $\uM_t$, $\baqm_t$ and the relation $| \msm^{\complement}_{\um^{(i)}_j(x)}| = |\msm_x^{\complement} |+1$, we obtain
		\begin{align}
			&\partial_t \uM_t(x, \um^{(i)}_j(x))\\
			&=\beta(\Tf-t) \uM_t(x, \um^{(i)}_j(x))\\
			&\quad\Bigg[ 1 +  \sum_{k \in \msm_x} \sum_{n\in \Z_m} \l(\uM_t(x, \um_n^{(k)}(x))
			-  \uM_t(\um^{(i)}_j(x), \um^{(k)}_n(\um^{(i)}_j(x)))\1_{k \neq i} \r)\Bigg] \\
			&= \beta(\Tf-t) \uM_t(x, \um^{(i)}_j(x))+ \sum_{k=1}^d\sum_{n\in \Z_m} \baqm_t(x, \um^{(k)}_n(x)) \\
			&\quad\Bigg[ \uM_t(x, \um^{(i)}_j(x)) \1_{k \in \msm_x} - \uM_t(\um^{(k)}_n(x), \um^{(i)}_j(\um^{(k)}_n(x)))\1_{i \in \msm_{\um^{(k)}_n(x)}} \Bigg] \eqsp,
		\end{align}
		where we used $\um^{(k)}_n(x), \um^{(i)}_j(\um^{(k)}_n(x)) = \um^{(i)}_j(x), \um^{(k)}_n(\um^{(k)}_n(x))$ and $\msm_{\um^{(k)}_n(x)} = \msm_x \setminus \{k\}$ in the last line. This concludes the proof by definition of $
		f^{(i),j}$ for $i,j\in \msm_x \times \Z_m$:
		\begin{equation}
			f^{(i),j}(t,x):= \uM_t(x, \um^{(i)}_j(x))\1_{i \in \msm_{x}} \eqsp.
		\end{equation}
	\end{proof}
	
	
	\begin{lemma}\label{prop:monotonicity_masked_d}
		Under \Cref{ass:condition_on_beta}, \Cref{ass:finite_fisher_masked_d}, for any $\eta \in (0,\Tf)$, $i\in [d]$ and $j\in \Z_m$, the process $\l(\uM_t(\baXm_t, \um^{(i)}_j(\baXm_t))\1_{i \in \msm_{\baXm_t}}\r)_{t\in [0,\Tf-\eta]}$ is a submartingale. In particular, the following holds for any $0 \leq r \leq t \leq \Tf-\eta$:
		\begin{equation}
			\E \l[ f^{(i),j}(t,\baXm_t)
			\middle| \Fc_r \r] = \rme^{\int_r^t \beta(\Tf-s) \rmd s}  f^{(i),j}(r,\baXm_r)\eqsp,
		\end{equation}
		where $(\Fc_{\tilde{t}})_{\tilde{t} \in [0,\Tf-\eta]}$
		is the filtration generated by $(\baXm_{\tilde{t}})_{\tilde{t} \in [0,\Tf-\eta]}$.
	\end{lemma}
	
	\begin{proof}[Proof of~\Cref{prop:monotonicity_masked_d}]
		Fix $\eta \in (0,\Tf)$, $i \in [d]$ and $j \in \Z_m$, applying It\^o's formula on $ f^{(i),j}(t,\baXm_{t})$
		for $t \in [0,\Tf-\eta]$ and noting that $\baXm_t=\baXm_{t-}$ for Lebesgue almost every $t \in [0,\Tf-\eta]$, we obtain that
		\begin{align}
			&f^{(i),j}(t,\baXm_{t} )-f^{(i),j}(0,\baXm_{0} )\\
			& \qquad = \Mm_{(i),j}(t)+\int_0^t\l[ \partial_s f^{(i),j}(s,\baXm_{s} )
			+(\baqm f^{(i),j}) (s, \baXm_{s}) \r]\rmd s \eqsp,
		\end{align}
		with
		\begin{align}
			\Mm_{(i),j}(t) = \int_{[0,t] \times \tZ^d_m} \l[ f^{(i),j}(s,x) - f^{(i),j}(s, \baXm_{s-}) \r] \tilde{N}_{\baXm}^{\baqm} (\rmd x \rmd s)
		\end{align}
		is a true martingale (see \Cref{lem:martingale_mask}), where $\tilde{N}_{\baXm}^{\baqm}$ denotes the compensated measure of the random point measure $N_{\baXm}^{\baqm}$ corresponding to the CTMC associated with $(\baqm_t)_{t\in[0,\Tf)}$;  see Appendix \ref{sec:stochastic_calculus} for completeness.
		Consequently, the process
		\begin{align}
			&f^{(i),j}(t,\baXm_{t} )-f^{(i),j}(0,\baXm_{0} ) -\int_0^t\l[ \partial_s f^{(i),j}(s,\baXm_{s} )
			+(\baqm f^{(i),j}) (s, \baXm_{s}) \r]\rmd s \eqsp,
		\end{align}
		is a martingale.
		Plugging the equation \eqref{eq:hjb_masked_d} into the integrand and using the martingale property, we get
		\begin{align}
			\E \l[f^{(i),j}(t,\baXm_t)
			\middle| \Fc_r \r] - \E \l[f^{(i),j}(s,\baXm_s)
			\middle| \Fc_r \r] =  \int_s^t \beta(T-r) \E \l[ f_r^{(i),j}(\baXm_r) \middle| \Fc_r \r] \rmd r\eqsp,
		\end{align}
		for any $0 \leq r \leq s \leq t < \Tf$, where $(\Fc_t)_{t\in [0,\Tf)}$ denotes the filtration generated by $(\baXm_t)_{t\in [0,\Tf)}$. For fixed $i  \in [d]$, $j \in \Z_m$ and $0 \leq r \leq t \leq \Tf-\eta$, we denote
		\[y_r^{(i),j}(t):=\E \l[f^{(i),j}(t,\baXm_t)
		\middle| \Fc_r \r] \eqsp,\]
		then the previous equation yields the following ordinary differential equation (ODE):
		\begin{align}
			\frac{\rmd }{\rmd t} y_r^{(i),j}(t) = \beta(\Tf-t)y_r^{(i),j}(t) \eqsp, \quad \text{for } 0 \leq r \leq t \leq \Tf-\eta \eqsp,
		\end{align}
		which in particular implies
		\[y_r^{(i),j}(t) = \rme^{\int_r^t \beta(\Tf-s) \rmd s} y_r^{(i),j}(r)  \eqsp,\] and the proof concludes.
	\end{proof}
	
	\begin{lemma}\label{prop:monotone_h_masked_d}
		Assume \Cref{ass:condition_on_beta} and \Cref{ass:finite_fisher_masked_d} holds. Recall that $\mathbf{h}(a)= a \log a- a+1$ for $a>0$ with the convention $\mathbf{h}(0)=0$. Fix $\eta \in (0,\Tf)$, $i \in [d]$, $j \in \Z_m$, then for $0 \leq r \leq t \leq \Tf-\eta$, the following holds
		\begin{align}
			\mathbf{h}(f^{(i),j}(r,\baXm_r))  \leq \E \l[ \mathbf{h}(f^{(i),j}(t,\baXm_t)) \middle| \Fc_r \r] + \rme^{\int_r^t \beta(\Tf-s) \rmd s} -1 \eqsp.
		\end{align}
	\end{lemma}
	\begin{proof}[Proof of~\Cref{prop:monotone_h_masked_d}]
		Fix $\eta \in (0,\Tf)$, $i \in [d]$ and $j \in \Z_m$, applying It\^o's formula on
		\begin{equation}
			g^{(i),j}(t,\baXm_{t}):= \mathbf{h}\l(f^{(i),j}(t,\baXm_t) \r)
		\end{equation}
		for $t \in [0,\Tf-\eta]$ and noting that $\baXm_t=\baXm_{t-}$ for Lebesgue almost every $t \in [0,\Tf-\eta]$, we obtain that
		\begin{align}
			&g^{(i),j}(t,\baXm_{t} )-g^{(i),j}(0,\baXm_{0} )\\
			&\qquad = \tMm_{(i),j}(t) + \int_0^t\l[ \partial_s g^{(i),j}(s,\baXm_{s} )
			+(\baqm g^{(i),j}) (s, \baXm_{s}) \r]\rmd s \eqsp,
		\end{align}
		with
		\begin{align}
			\tMm_{(i),j}(t) = \int_{[0,t] \times \tZ^d_m} \l[ g^{(i),j}(s,x) - g^{(i),j}(s, \baXm_{s-}) \r] \tilde{N}_{\baXm}^{\baqm} (\rmd x \rmd s)
		\end{align}
		is a local martingale. Recall that under \Cref{ass:condition_on_beta} and \Cref{ass:finite_fisher_masked_d}, we have $\mum_t(x)>0$ for any $x \in \tZ^d_m$ and $t\in [0,\Tf-\eta]$, therefore $\uM_t(x,\um^{(i)}_j(x)) >0$. Arguing similarly as before while noting that $\mathbf{h}$ is a continuous function, we obtain that $\tMm_{(i),j}(t)$ is integrable and thus a martingale.
		As a result, the process
		\begin{align}\label{eq:monotone_h_masked_d}
			&g^{(i),j}(t,\baXm_{t} )-g^{(i),j}(0,\baXm_{0} ) - \int_0^t\l[ \partial_s g^{(i),j}(s,\baXm_{s} )
			+(\baqm g^{(i),j}) (s, \baXm_{s}) \r]\rmd s 
		\end{align}
		is a genuine martingale.
		Denote the integrand
		\begin{align}
			c_s^{(i),j}(x):= \partial_s g^{(i),j}(s,x )
			+(\baqm g^{(i),j}) (s, x), \quad \text{for } s\in [0,\Tf-\eta] \eqsp.
		\end{align}
		By definition of $g^{(i),j}, \baqm$ and the equation shown in~\Cref{lem:formula_u_masked_d}, we get that
		\begin{align}
			c_s^{(i),j}(x) &= \log  f^{(i),j}\partial_s f^{(i),j} (s,x) + \sum_{y \neq x} \baqm_s(x,y)\\
			&\qquad\l[(f^{(i),j}\log f^{(i),j}-f^{(i),j})(s,y) - (f^{(i),j}\log f^{(i),j}-f^{(i),j})(s,x) \r] \\
			&=\beta(\Tf-s) f^{(i),j} \log f^{(i),j} (s,x) + \sum_{y \neq x} \baqm_s(x,y)\\
			&\qquad\l[f^{(i),j}(s,y)\log \frac{f^{(i),j}(s,y)}{f^{(i),j}(s,x)} - f^{(i),j}(s,y) +f^{(i),j}(s,x) \r] \eqsp.
		\end{align}
		By using the inequality $\log x \geq 1-1/x$, we arrive at
		\begin{align}
			c_s^{(i),j}(x) &\geq \beta(\Tf-s) f^{(i),j} \log f^{(i),j} (s,x) \geq \beta(\Tf-s)\l[  g^{(i),j}(s,x) -1 \r] \eqsp.
		\end{align}
		Plugging it into~\eqref{eq:monotone_h_masked_d} and taking the expectation given $\Fc_r$, we obtain for $0 \leq r \leq s \leq t \leq \Tf-\eta$,
		\begin{align}
			&\quad\E \l[g^{(i),j}(t,\baXm_t)
			\middle| \Fc_r \r] - \E \l[g^{(i),j}(s,\baXm_s)
			\middle| \Fc_r \r]\\
			& \qquad \geq  \int_s^t \beta(T-r) \l\{ \E \l[ g^{(i),j}(r, \baXm_r) \middle| \Fc_r \r] -1\r\} \rmd r \eqsp.
		\end{align}
		Denoting $\gamma_r^{(i),j}(t):=\E \l[g^{(i),j}(t,\baXm_t) \middle| \Fc_r \r]$ , then the previous equation implies
		\begin{align}
			\frac{\rmd }{\rmd t}  \gamma_r^{(i),j}(t) \geq \beta(\Tf-t) \l[\gamma_r^{(i),j}(t)-1 \r] \eqsp, \quad \text{for } 0 \leq r \leq t \leq \Tf-\eta \eqsp.
		\end{align}
		This together with the Gronwall's inequality yield
		\begin{align}
			\gamma_r^{(i),j}(t) -1 \geq \l[\gamma_r^{(i),j}(r) -1\r] \rme^{\int_r^t \beta(\Tf-s) \rmd s} = \gamma_r^{(i),j}(r) \rme^{\int_r^t \beta(\Tf-s) \rmd s} -   \rme^{\int_r^t \beta(\Tf-s) \rmd s} \eqsp.
		\end{align}
		Recall that $\beta(t) \geq 0$ for all $t \in [0,\Tf]$ and $\gamma_r^{(i),j}(r) \geq 0$ as the function $\mathbf{h}(a)=a\log a-a+1$ is nonnegative. Therefore,
		\begin{align}
			\gamma_r^{(i),j}(t) \geq  \gamma_r^{(i),j}(r) +1 - \rme^{\int_r^t \beta(\Tf-s) \rmd s} \eqsp.
		\end{align}
		As a result, for $0 \leq r \leq t\leq \Tf-\eta$, the following holds
		\begin{align}
			\gamma_r^{(i),j}(r) \leq \gamma_r^{(i),j}(t) + \rme^{\int_r^t \beta(\Tf-s) \rmd s} -1 \eqsp,
		\end{align}
		and the proof concludes.
	\end{proof}
	
	A useful result that can be leveraged to reduce the complexity of DDMs is the following upper bound on the discrete Fisher information of the marginal density.
	\begin{lemma}\label{lem:bound_fisher_masked}
		Assume \Cref{ass:condition_on_beta} and \Cref{ass:finite_fisher_masked_d} hold. For $t \in [0,\Tf)$, let us denote \[ \Ic(\mum_{\Tf-t}) \eqdef \E \l[ \sum_{i=1}^d\sum_{j \in \Z_m}\mathbf{h}\l( f^{(i),j}(t,\baXm_t)\r) \r] \eqsp,\]
		then we have
		\begin{align}
			\Ic(\mum_{\Tf-t}) \lesssim d \l(\frac{\alpha_{\Tf-t}}{1-\alpha_{\Tf-t}}+m \r) \eqsp.
		\end{align}
		In particular, for a constant generator $\beta(t)=1$, this estimate reduces to
		\begin{align}
			\Ic(\mum_{\Tf-t}) \lesssim d \l(\frac{1}{\Tf-t}+m \r) \eqsp.
		\end{align}
	\end{lemma}
	
	\begin{proof}[Proof of \Cref{lem:bound_fisher_masked}]
		It is given in Appendix \ref{proof_lem:bound_fisher_masked}.
	\end{proof}
	
	\subsubsection{Convergence proofs}\label{sec:masked_proof_conv}
	\begin{proof}[Proof of~\Cref{theo:5_masked_d}]
		
		We show first the bound for the “distance” between the backward path measure $\baPm$ of $(\baXm_t)_{t\in [0,\Tf-\eta]}$ and $\baPmstar$ of the simulated backward process $(\baXmstar_t)_{t\in[0,\Tf-\eta]}$. Consider the path measure $\baPmstar \in \mathrm{MP}(\baqmstar)$ as the reference measure in Girsanov's theorem, we have
		\begin{align}
			\KL(\baPm| \baPmstar) 
			&=\KL(\mum_{\Tf}|\mathrm{Uniform}(\Z^d_m)\fopm_{0,\Tf})\\
			&\qquad+\E \Bigg[ \int_{[0,\Tf-\eta]}\sum_{x\in \Z^d_m}  \mathbf{h}\l( \frac{\baqm_t}{\baqmstar_t}\r) \baqmstar_t (\baXm_{t},x) \1_{\baXm_{t} \neq x} \rmd t \Bigg] \eqsp.
		\end{align}
		With a partition $0=t_0<...<t_K=\Tf-\eta$ for $K \geq 1$ of $[0,\Tf-\eta]$ associated with the sequence of step-size $h_{k+1} = t_{k+1} - t_k$, the previous expression rewrites as
		\begin{align}
			\KL(\baPm| \baPmstar)  &=\KL(\mum_{\Tf}|\mathrm{Uniform}(\Z^d_m)\fopm_{0,\Tf})\\
			&\quad+\sum_{k=0}^{K-1} \E \Bigg[ \int_{[t_k,t_{k+1})}\sum_{x\in \Z^d_m}  \baqmstar_t\mathbf{h}\l(  \frac{\baqm_{t}}{\baqmstar_t} \r)  (\baXm_{t},x) \1_{\baXm_{t} \neq x}  \rmd t \Bigg] \eqsp.
		\end{align}
		Replacing the formula of $(\baqm_t)$ from~\Cref{lem:formula_u_masked_d} and $(\baqmstar_t)_{t\in[0,\Tf)}$ from~\eqref{eq:generator_generative_ctmc},~\eqref{eq:control_generative_masked_d} and using that $\tqm_t(x,y) = \foqm_{\Tf-t}(y,x)$ where $(\foqm_t)_{t\in [0,\Tf]}$ given in~\eqref{eq:forward_generator_masked_d}, we obtain
		\begin{align}
			\KL(\baPm| \baPmstar)  &=\KL(\mum_{\Tf}|\mathrm{Uniform}(\Z^d_m)\fopm_{0,\Tf})\\
			&+\sum_{k=0}^{K-1} \E \Bigg[ \int_{[t_k,t_{k+1})}\sum_{j \in \Z_m}\sum_{i \in \msm_{\baXm_{t}}} (\uMstar_{t_k})^{(i),j} \mathbf{h}\l(  \frac{(\uM_{t})^{(i),j}}{(\uMstar_{t_k})^{(i),j}} \r) \beta(\Tf-t)    \rmd t \Bigg] \eqsp,
		\end{align}
		where $(\uM_t)^{(i),j} \eqdef \uM_t(\baXm_{t},\um^{(i)}_j(\baXm_t))$ and $(\uMstar_{t_k})^{(i),j} \eqdef \uMstar_{t_k}(\baXm_{t_k},\um^{(i)}_j(\baXm_{t_k}))$. Noting that $\beta(\Tf-t) \in [0,1]$, we deduce that
		\begin{align}
			\KL(\baPm| \baPmstar)  &\leq \underbrace{\KL(\mum_{\Tf}|\mathrm{Uniform}(\Z^d_m)\fopm_{0,\Tf})}_{F_1}\\
			&+\underbrace{\sum_{k=0}^{K-1} \E \Bigg[ \int_{[t_k,t_{k+1})}\sum_{j \in \Z_m}\sum_{i \in \msm_{\baXm_{t_k}}} (\uMstar_{t_k})^{(i),j} \mathbf{h}\l(  \frac{(\uM_{t_k})^{(i),j}}{(\uMstar_{t_k})^{(i),j}} \r)   \rmd t \Bigg]}_{F_2}\\
			& \underbrace{\substack{\displaystyle +\sum_{k=0}^{K-1} \E \Bigg[ \int_{[t_k,t_{k+1})}\sum_{x\in \Z^d_m} \sum_{i=1}^d \Bigg\{ (\uMstar_{t_k})^{(i),j} \mathbf{h}\l(  \frac{(\uM_{t})^{(i),j}}{(\uMstar_{t_k})^{(i),j}} \r)  \1_{i \in \msm_{\baXm_{t}}}\\
					\displaystyle \qquad-(\uMstar_{t_k})^{(i),j} \mathbf{h}\l(  \frac{(\uM_{t_k})^{(i),j}}{(\uMstar_{t_k})^{(i),j}} \r)  \1_{i \in \msm_{\baXm_{t_k}}} \Bigg\} \rmd t \Bigg]}}_{F_3} \eqsp.
		\end{align}
		The term $F_2$ is easily bounded  by~\Cref{ass:approx_score_masked_d}:
		\begin{align}
			F_2 &= \sum_{k=0}^{K-1} h_{k+1} \E \Bigg[ \sum_{j \in \Z_m}\sum_{i \in \msm_{\baXm_{t_k}}} (\uMstar_{t_k})^{(i),j} \mathbf{h}\l(  \frac{(\uM_{t_k})^{(i),j}}{(\uMstar_{t_k})^{(i),j}} \r)  \Bigg] \leq \varepsilon^2\Tf \eqsp.\label{eq:e2_masked}
		\end{align}
		We bound next the term $F_1$. To this purpose, we first compute explicitly the starting measure of our generative process: for $x \in \tZ^d_m$,
		\begin{align}
			\l(\mathrm{Uniform}(\Z^d_m)\fopm_{0,\Tf} \r)(x) &= \sum_{y \in \Z^d_m} \mathrm{Uniform}(y) \fopm_{0,\Tf}(y,x)
			= \sum_{y \in \Z^d_m} \frac{1}{m^d} \prod_{i=1}^d \fopmm_{0,\Tf}(y^i,x^i) \\
			&= \sum_{y \in \Z^d_m} \frac{1}{m^d} \prod_{i \in \msm_x} (1-\alpha_{\Tf}) \prod_{j \in \msm_x^{\complement}} \alpha_{\Tf} \1_{y^j = x^j} \\
			&= \frac{1}{m^d} (1-\alpha_{\Tf})^{|\msm_x|}  \alpha_{\Tf}^{d-|\msm_x|} m^{|\msm_x|}\\
			& = (1-\alpha_{\Tf})^{|\msm_x|}  \l(\frac{\alpha_{\Tf}}{m} \r)^{d-|\msm_x|} \eqsp.
		\end{align}
		It yields the expression of $F_1$ as follows
		\begin{align}
			F_1 &= \KL(\mum_{\Tf}|\mathrm{Uniform}(
			\Z^d_m)\fopm_{0,\Tf}) = \sum_{x \in \tZ^d_m} \mum_{\Tf}(x) \log \frac{\mum_{\Tf}(x)}{(\mathrm{Uniform}(
				\Z^d_m)\fopm_{0,\Tf})(x)} \\
			&= \sum_{x\in \tZ^d_m} \mum_{\Tf} (x)\log\mum_{\Tf}(x) - \sum_{x \in \tZ^d_m} \mum_{\Tf}(x) \log \l(\mathrm{Uniform}(
			\Z^d_m)\fopm_{0,\Tf}\r)(x) \\
			&\leq -\sum_{x\in \tZ^d_m} \mum_{\Tf}(x) \log \l[ (1-\alpha_{\Tf})^{|\msm_x|}  \l(\frac{\alpha_{\Tf}}{m} \r)^{d-|\msm_x|} \r] 
		\end{align}
		since $\mum_{\Tf}(x) \leq 1$. Hence
		\begin{align}
			F_1 \leq -\E \l[|\msm_{\foXm_{\Tf}}|\r]\log (1-\alpha_{\Tf}) -  \l(d-\E \l[|\msm_{\foXm_{\Tf}}| \r]\r) \log (\alpha_{\Tf}/m) \eqsp.
		\end{align}
		Furthermore, we know that
		\begin{align}
			\E \l[|\msm_{\foXm_{\Tf}}|\r] &= \sum_{x \in \tZ^d_m} \sum_{i=1}^d \1_{x^i=m} \mum_{\Tf}(x) = \sum_{i=1}^d \P ((\foXm_t)^i =m) \notag \\
			&= \sum_{i=1}^d  (1-\alpha_{\Tf}) \P((\foXm_0)^i \neq m)= d(1-\alpha_{\Tf}) \eqsp, \label{eq:expectation_msm_masked}
		\end{align}
		since $\foXm_0 \sim \mustar$ with $\mustar$ supported on $\Z^d_m$. Plugging this into $F_1$ gives
		\begin{align}
			F_1 &\quad \leq -d(1-\alpha_{\Tf})\log(1-\alpha_{\Tf}) -d\alpha_{\Tf} \log (\alpha_{\Tf}/m) \notag \\
			& \overset{\log a \leq a-1}{\leq} (1-\alpha_{\Tf}) \l({1}/({1-\alpha_{\Tf}})-1 \r) +d\alpha_{\Tf} \log (m/\alpha_{\Tf}) \notag \\
			&\quad \leq d\alpha_{\Tf} (1+\log(m/\alpha_{\Tf})) \eqsp. \label{eq:e1_masked}
		\end{align}
		It remains to control the term $F_3$.
		Let us denote $(\tuM_t)^{(i),j} \eqdef (\uM_t)^{(i),j} \1_{i \in \msm_{\baXm_t}}$ and $(\tuMstar_t)^{(i),j} \eqdef (\uMstar_t)^{(i),j} \1_{i \in \msm_{\baXm_t}}$ throughout the rest of this proof. We rewrite
		\begin{align}
			&F_3 = \sum_{k=0}^{K-1} \E \Bigg[ \int_{[t_k,t_{k+1})}\sum_{j\in \Z_m} \sum_{i=1}^d \Bigg\{  (\tuM_t)^{(i),j} \log \frac{ (\tuM_t)^{(i),j}}{(\uMstar_{t_k})^{(i),j}}  -  (\tuM_t)^{(i),j}\\
			&\quad + (\uMstar_{t_k})^{(i),j}\1_{i \in \msm_{\baXm_t}} - \l[ (\tuM_{t_k})^{(i),j} \log \frac{(\tuM_{t_k})^{(i),j}}{(\uMstar_{t_k})^{(i),j}} - (\tuM_{t_k})^{(i),j} + (\tuMstar_{t_k})^{(i),j}
			\r]\Bigg\} \rmd t \Bigg] \eqsp.
		\end{align}
		We arrange $F_3$ as follows
		\begin{align}
			F_3 &\leq  \underbrace{\sum_{k=0}^{K-1} \E \l[\sum_{j\in \Z_m} \sum_{i=1}^d \int_{[t_k,t_{k+1})}
				\l[\mathbf{h}((\tuM_t)^{(i),j})- \mathbf{h}( (\tuM_{t_k})^{(i),j}) \r] \rmd t \r]}_{F_{3.1}}\\
			&+\underbrace{\sum_{k=0}^{K-1} \E \l[\sum_{j\in \Z_m} \sum_{i=1}^d \int_{[t_k,t_{k+1})} (\uMstar_{t_k})^{(i),j} \l[\1_{i \in \msm_{\baXm_t}} - \1_{i \in \msm_{\baXm_{t_k}}} \r] \rmd t \r]}_{F_{3.2}}\\
			&+\underbrace{\sum_{k=0}^{K-1} \E \l[\sum_{j\in \Z_m} \sum_{i=1}^d \int_{[t_k,t_{k+1})} \log {(\uMstar_{t_k})^{(i),j}}\l[(\tuM_{t_k})^{(i),j}-(\tuM_t)^{(i),j}
				\r] \rmd t \r]}_{F_{3.2}} \eqsp.
		\end{align}
		The first quantity $F_{3.1}$ can be controlled by the tower property and~\Cref{prop:monotone_h_masked_d} as follows
		\begin{align}
			F_{3.1} 
			&\leq \sum_{k=0}^{K-1} \E \Bigg[\sum_{j\in \Z_m} \sum_{i=1}^d \int_{[t_k,t_{k+1})}
			\Big\{ \mathbf{h}((\tuM_{t_{k+1}})^{(i),j})- \mathbf{h}( (\tuM_{t_k})^{(i),j})  + \rme^{\int_{t_k}^{t_{k+1}} \beta(\Tf-s) \rmd s} -1  \Big\} \rmd t \Bigg] \\
			&\leq \sum_{k=0}^{K-1} \E \Bigg[\sum_{j\in \Z_m} \sum_{i=1}^d h_{k+1}
			\Big\{ \mathbf{h}((\tuM_{t_{k+1}})^{(i),j})- \mathbf{h}( (\tuM_{t_k})^{(i),j}) + \rme^{t_{k+1}-t_k} -1  \Big\} \Bigg]\\
			&\leq h\sum_{j\in \Z_m} \sum_{i=1}^d \l\{\E \l[\mathbf{h}((\tuM_{t_{k+1}})^{(i),j}) \r] - \E\l[\mathbf{h}( (\tuM_{t_k})^{(i),j})\r] \r\} + dm(\rme^h-1)\sum_{k=0}^{K-1}h_{k+1} \eqsp,
		\end{align}
		where $h = \max_k \{h_{k+1} \}$. We obtain telescoping sums on the right hand side, therefore,
		\begin{align}
			F_{3.1} &\leq h \l\{\E\l[ \sum_{j\in \Z_m} \sum_{i=1}^d \mathbf{h}((\tuM_{t_{K}})^{(i),j}) \r] - \E\l[ \sum_{j\in \Z_m} \sum_{i=1}^d \mathbf{h}((\tuM_{t_{0}})^{(i),j}) \r] \r\}\\
			&\qquad \qquad \quad+ dm(\rme^h-1)(t_K -t_0) \\
			&\leq h \E \l[ \sum_{i=1}^d \sum_{j \in \Z_m} \mathbf{h}((\tuM_{\Tf-\eta})^{(i),j}) \r] + dm(\rme^h -1)(\Tf- \eta )  \\
			&\leq h \Ic(\mum_\eta) + dm(\rme^h -1)\Tf \eqsp,
		\end{align}
		where $\Ic(\mum_\eta) = \E \l[ \sum_{i=1}^d \sum_{j \in \Z_m} \mathbf{h}((\tuM_{\Tf-\eta})^{(i),j})\r]$ is the discrete Fisher information of $\mum_\eta$.
		Leveraging the upper bound of Fisher information showed in \Cref{lem:bound_fisher_masked}, we get
		\begin{align}
			F_{3.1} \lesssim h \l(\frac{d\alpha_\eta}{1-\alpha_\eta} +dm \r)+ dm(\rme^h -1)\Tf \eqsp.\label{eq:e3_1_masked}
		\end{align}
		Next, from~\eqref{eq:forward_transition_matrix_mask_d}, we have
		\[
		\P \l((\foXm_t)^i=m | (\foXm_s)^i=m \r) = 1  \quad \text{ for any $0 \leq s \leq t \leq \Tf$} \eqsp,
		\]
		which yields $\msm_{\foXm_s} \subset \msm_{\foXm_t}$ a.s., \ie, $\msm_{\baXm_t} \subset \msm_{\baXm_s}$ a.s. for $0 \leq s \leq t \leq \Tf$. Therefore,
		the quantity $F_{3.2}$ can be cancelled out.
		The last term $F_{3.3}$ can be controlled by tower property and~\Cref{prop:monotonicity_masked_d} as follows
		\begin{align}
			F_{3.3} &= \sum_{k=0}^{K-1} \E \l[\sum_{\substack{j\in \Z_m \\
					i \in [d]}} \int_{[t_k,t_{k+1})} \log {(\uMstar_{t_k})^{(i),j}}\l[(\tuM_{t_k})^{(i),j}-\E\l[(\tuM_t)^{(i),j} |\Fc_{t_k}\r]
			\r]  \rmd t \r] \\
			&= \sum_{k=0}^{K-1} \E \l[\sum_{j\in \Z_m }\sum_{i=1}^d \int_{[t_k,t_{k+1})} \log {(\uMstar_{t_k})^{(i),j}}(\tuM_{t_k})^{(i),j}\l[1-\rme^{\int_{t_k}^t \beta(\Tf-s)\rmd s}
			\r]  \rmd t \r]\\
			&\leq \underbrace{(\rme^h-1) \sum_{k=0}^{K-1} h_{k+1} \E \l[\sum_{j\in \Z_m }\sum_{i=1}^d \l| \log \frac{(\uM_{t_k})^{(i),j}} {(\uMstar_{t_k})^{(i),j}}(\tuM_{t_k})^{(i),j} \r|
				\r]}_{F_{3.3a}} \\
			& + \underbrace{\sum_{k=0}^{K-1}  \E \l[ \sum_{j \in \Z_m} \sum_{i=1}^d (\tuM_{t_k})^{(i),j}\log (\tuM_{t_k})^{(i),j} \r] \int_{[t_k,t_{k+1})} \l[1-\rme^{\int_{t_k}^t \beta(\Tf-s) \rmd s}
				\r]  \rmd t}_{F_{3.3b}} \eqsp,
		\end{align}
		Concerning $F_{3.3a}$, first apply triangle inequality, second use \Cref{ass:approx_score_masked_d} to bound it:
		\begin{align}
			F_{3.3a} 
			&\leq (\rme^h-1) \sum_{k=0}^{K-1} h_{k+1} 
			\sum_{j \in \Z_m} \sum_{i=1}^d \E \Bigg[   (\tuMstar_{t_k})^{(i),j} \mathbf{h} \l( \frac{(\uM_{t_k})^{(i),j}} {(\uMstar_{t_k})^{(i),j}} \r) \\
			&\qquad \qquad \quad+ \underbrace{\l| (\tuM_{t_k})^{(i),j} - (\tuMstar_{t_k})^{(i),j} \r|}_{\leq \| (\tuM_{t_k})^{(i),j} - (\tuMstar_{t_k})^{(i),j} \|^2+1}  \Bigg] \\
			&\leq (\rme^h-1) \varepsilon^2\Tf + md(\rme^h-1) \Tf \eqsp.
		\end{align}
		The term $F_{3.3b}$ is handled as follows
		\begin{align}
			F_{3.3b} & = - \sum_{k=0}^{K-1}  \E \l[ \sum_{j \in \Z_m} \sum_{i=1}^d (\tuM_{t_k})^{(i),j}\log (\tuM_{t_k})^{(i),j} \r] \int_{[t_k,t_{k+1})} \l[\rme^{\int_{t_k}^t \beta(\Tf-s) \rmd s} -1
			\r]  \rmd t \\
			& \overset{\log x \geq 1-1/x}{\leq}  \sum_{k=0}^{K-1}  \sum_{j \in \Z_m} \sum_{i=1}^d\E \l[  1-(\tuM_{t_k})^{(i),j} \r] \int_{[t_k,t_{k+1})} \l[\rme^{\int_{t_k}^t \beta(\Tf-s) \rmd s} -1
			\r]  \rmd t \\
			&\qquad\leq md (\rme^h-1) \sum_{k=0}^{K-1} h_{k+1} = md (\rme^h-1) \Tf \eqsp.
		\end{align}
		This together with the bound on $F_{3.3a}$ imply
		\begin{align}
			F_{3.3} \leq (\rme^h-1)( \varepsilon^2+ dm) \Tf  \eqsp. \label{eq:e3_4_masked}
		\end{align}
		Combining~\eqref{eq:e3_1_masked} and~\eqref{eq:e3_4_masked} yields
		\begin{align}
			F_3 \lesssim hd \l(\frac{\alpha_\eta}{1-\alpha_\eta} +m \r) + (\rme^h-1) dm\Tf + \rme^h \varepsilon^2\Tf \eqsp. \label{eq:e3_masked}
		\end{align}
		Substituting all~\eqref{eq:e2_masked},~\eqref{eq:e1_masked} and~\eqref{eq:e3_masked} into the bound of $\KL(\baPm|\baPmstar)$, we arrive at
		\begin{align}
			\KL(\baPm|\baPmstar) &\lesssim d\alpha_{\Tf} (1+\log(m/\alpha_{\Tf})) + \rme^h \varepsilon^2\Tf \\
			&\quad + hd \l(\frac{\alpha_\eta}{1-\alpha_\eta} +m\r) + (\rme^h-1) dm\Tf \eqsp.
		\end{align}
		To this end, notice that $\mum_\eta=\mathrm{Law}(\baXm_{\Tf-\eta})$, therefore
		\begin{align}
			\KL(\mum_\eta|\mathrm{Law}(\baXmstar_{\Tf-\eta})) &=\KL(\mathrm{Law}(\baXm_{\Tf-\eta})|\mathrm{Law}(\baXmstar_{\Tf-\eta}) ) \\
			&\leq \KL(\mathrm{Law}((\baXm_t)_{t\in [0,\Tf-\eta]})|\mathrm{Law}((\baXmstar_t)_{t\in [0,\Tf-\eta]} ))\\
			&= \KL( \baPm|\baPmstar) \eqsp,
		\end{align}
		where the inequality is known as {Data processing} inequality for relative entropy \cite[Lemma 1.6]{nutz2021introduction}. As a consequence,
		\begin{align}
			\KL(\mum_\eta|\mathrm{Law}(\baXmstar_{\Tf-\eta}))
			&\lesssim  d\alpha_{\Tf} (1+\log(m/\alpha_{\Tf})) + \rme^h \varepsilon^2\Tf  \\
			&\quad+ hd \l(\frac{\alpha_\eta}{1-\alpha_\eta} +m\r) + (\rme^h-1) dm\Tf \eqsp,
		\end{align}
		which concludes the proof of~\Cref{theo:5_masked_d}.
	\end{proof}
	
	\begin{proof}[Proof of \Cref{theo:exp_stepsize_masked}]
		For the constant generator $\beta(t) =1 $, $ t \in [0,\Tf]$, we obtain a precise upper bound of the Fisher information. Therefore, we can leverage the choice of step-sizes to reduce the complexity of the model as follows. We proceed analogously as \Cref{theo:main_masked_d} and note that only the term $F_{3.1}$ in \eqref{eq:e3_1_masked} is handled differently as:
		\begin{align}
			F_{3.1} &\leq \sum_{k=0}^{K-1}h_{k+1} \E \Bigg[\sum_{j\in \Z_m} \sum_{i=1}^d
			\Big\{ \mathbf{h}((\tuM_{t_{k+1}})^{(i),j})- \mathbf{h}( (\tuM_{t_k})^{(i),j}) + \rme^{t_{k+1}-t_k} -1  \Big\} \Bigg] \\
			& \leq \sum_{k=0}^{K-1}h_{k+1} \l( \Ic(\mum_{\Tf-t_{k+1}}) - \Ic(\mum_{\Tf-t_k}) \r) + md(\rme^h-1) \sum_{k=0}^{K-1} h_{k+1} \\
			&\leq  \underbrace{\sum_{k=0}^{K-1}h_{k+1} \l( \Ic(\mum_{\Tf-t_{k+1}}) - \Ic(\mum_{\Tf-t_k}) \r)}_{F_{3.1a}} + md(\rme^h-1) \Tf \eqsp,
		\end{align}
		where $\Ic(\mum_{\Tf-t}) = \E \l[\sum_{j \in \Z_m} \sum_{i=1}^d \mathbf{h}((\tuM_t)^{(i),j})\r]$ for $t \in [0,\Tf)$ and $ \max_k h_{k}=h$.
		Let us choose the following step-sizes:
		\begin{equation}\label{eq:step-size_masked}
			h_{k+1} = \begin{cases}
				\Tf -\eta - t_{K-1} \quad & k = K-1 \eqsp, \\
				\quad c/L \quad & k_0 + k_1+1 \leq k \leq k_0 + k_1 + k_2 - 1 \eqsp, \\
				c( \Tf -\eta- t_{k})\quad & k_0+1 \leq k \leq k_0 +k_1  \eqsp, \\
				\quad c  \quad & 0 \leq k \leq k_0  \eqsp,
			\end{cases}
		\end{equation}
		with $L_\eta = \Ic(\mum_{\eta})/d$, then $\max_k h_k = h =c$. We set the number of iterations $K = k_0 +k_1 + k_2 +1 $, with
		\begin{align}
			k_0 &= \max \l\{ k \geq 0 \, : \, \Tf-\eta - t_k \geq 1 \r\} \quad  k_1 = \max \l\{ k \geq 0 \,: \, \Tf-\eta - t_{k_0+k} \geq 1/L_\eta \r\} \notag \\
			k_2 &= \max \l\{ k \geq 0 \,: \, \Tf -\eta- t_{k_0 + k_1 +k} \geq 0 \r\} \eqsp. \label{eq:def_k0k1k2_masked}
		\end{align}
		It is shown in \cite{conforti2025kl} that
		\begin{equation}\label{eq:bound_k0k1k2_masked}
			\begin{split}
				&\quad k_0 = \lfloor c^{-1} (\Tf-\eta - 1) \rfloor, \quad k_1 = \lfloor \log(L_\eta^{-1}/ (\Tf-\eta - t_{k_0})) / \log (1-c) \rfloor \lesssim \log(L_\eta) /c \eqsp, \\
				& K  - k_0 - k_1 = k_2 + 1 \lesssim 1/c \eqsp, \eqsp h_{k+1} = c(1-c)^{k-k_0} (\Tf-\eta-t_{k_0}) \text{ for } k_0+1 \leq k \leq k_0+ k_1 \eqsp.
			\end{split}
		\end{equation}
		Using \eqref{eq:step-size_masked} and the monotonicity of $\Ic(\mum_{\Tf-t})$ established in \Cref{prop:monotone_h_masked_d}, we can bound $E_{3.1a}$ as follows
		\begin{align}
			F_{3.1a} &\leq \sum_{k=0}^{K-1} h_{k+1} \l( \Ic(\mum_{\Tf-t_{k+1}}) - \Ic(\mum_{\Tf-t_k}) \r) \\
			& \leq h_K \Ic(\mum_{\Tf-t_k}) + \sum_{k=1}^{K-1} \Ic(\mum_{\Tf-t_k}) (h_k - h_{k+1}) \\
			& =  \sum_{k=1}^{k_0+1} \Ic(\mum_{\Tf-t_k}) (h_k - h_{k+1}) + \sum_{k = k_0 +2}^{k_0+k_1+1} \Ic(\mum_{\Tf-t_k}) (h_k - h_{k+1})  + h_K \Ic(\mum_{\Tf-t_K}) \\
			& \qquad+\underbrace{\sum_{k = k_0 +k_1+2}^{k_0 + k_1 + k_2 -1}\Ic(\mum_{\Tf-t_k}) (h_k - h_{k+1})}_{=0} + \Ic(\mum_{\Tf-t_{k_0+k_1+k_2}})(h_{K-1} - h_{K})  \\
			& \lesssim  \underbrace{\Ic(\mum_{\Tf-t_{k_0+1}}) [c - c (\Tf-\eta - t_{k_0+1})]}_{(1)} + \underbrace{c \sum_{k = k_0 +2}^{k_0 + k_1} \Ic(\mum_{\Tf-t_k}) h_k}_{(2)} + md (\rme^{c}-1)\Tf \\
			& \qquad  + \underbrace{c \Ic(\mum_{\Tf-t_{k_0+k_1+1}}) (\Tf-\eta - t_{k_0+k_1 } - 1/L)}_{(3)} + \underbrace{\Ic(\mum_{\Tf-t_K}) h_{K-1}}_{(4)}  \eqsp.
		\end{align}
		We now bound $(1) - (2) - (3) - (4)$ by leveraging \Cref{lem:bound_fisher_masked}. We start with
		\begin{align}
			&(1) : \Ic(\mum_{\Tf-t_{k_0+1}}) [c - c (\Tf-\eta - t_{k_0+1})]\\
			&\leq c \Ic(\mum_{\Tf-t_{k_0+1}}) \overset{\Cref{lem:bound_fisher_masked}}{\lesssim} {cd} \l( \frac{1}{\Tf - t_{k_0+1}} +m \r)\\
			&= {cd} \l( \frac{1}{\Tf - t_{k_0}-h_{k_0+1}} +m \r)\\
			&\overset{\eqref{eq:def_k0k1k2_masked}}{\leq} cd\l(\frac{1}{1-c}+m \r) \overset{c \leq 1/2}{\lesssim} {cd}(m+1) \eqsp.
		\end{align}
		Next, we bound the second term
		\begin{align}
			&(2) : c \sum_{k = k_0 +2}^{k_0 + k_1 } \Ic(\mum_{\Tf-t_k}) h_k\\ &\overset{\Cref{lem:bound_fisher_masked}}{\lesssim} {cd} \sum_{k = k_0 +2}^{k_0 + k_1 } { h_k} \l( \frac{1}{\Tf- t_k}+m \r) \\
			&\quad\overset{\eqref{eq:step-size_masked}}{\leq} { c^2 d } \sum_{k = k_0 +2}^{k_0 + k_1 }  \frac{h_k}{h_{k+1}} + c^2dm \sum_{k = k_0 +2}^{k_0 + k_1 } (\Tf-\eta-t_k)  \\
			& \quad \overset{\eqref{eq:bound_k0k1k2_masked}}{=} {c^2 d} \sum_{k = k_0 +2}^{k_0 + k_1 }  \dfrac{c(1-c)^{k-k_0-1}(\Tf-\eta - t_{k_0})}{c(1-c)^{k-k_0}(\Tf-\eta - t_{k_0})} + c^2dm (k_1-1) \\
			&\quad \overset{\eqref{eq:bound_k0k1k2_masked}} {\lesssim} {c^2d} \sum_{k = k_0 +2}^{k_0 + k_1 } \dfrac{1}{1-c} + c^2dmk_1 \\
			&\quad\overset{c \leq 1/2}{\lesssim} c^2 dk_1 +c^2dmk_1 \overset{\eqref{eq:bound_k0k1k2_masked}}{\lesssim} {c d}(1+m) \log(L_\eta) \eqsp.
		\end{align}
		The third term $(3)$ can be bounded as follows
		\begin{align}
			(3) : \quad  & c \Ic(\mum_{\Tf-t_{k_0+k_1+1}}) (\Tf-\eta - t_{k_0+k_1 } - 1/L_\eta) \\
			&\overset{\Cref{lem:bound_fisher_masked}}{\lesssim} cd(\Tf-\eta - t_{k_0+k_1 } )\l(\frac{1}{\Tf-t_{k_0+k_1+1}}+m \r)  \\
			& \leq cd(\Tf-\eta - t_{k_0+k_1} )\l(\frac{1}{\Tf-\eta-t_{k_0+k_1}-h_{k_0+k_1+1}}+m \r) \\
			&\overset{ \eqref{eq:step-size_masked}}{\leq}
			cd(\Tf-\eta - t_{k_0+k_1 } )\l(\frac{1}{(1-c)(\Tf-\eta-t_{k_0+k_1})}+m \r)\\
			&= \frac{cd}{1-c}+cdm \underbrace{(\Tf-\eta - t_{k_0+k_1})}_{\leq 1} \\
			&\overset{c \leq 1/2}{ \lesssim } cd(1+m) \eqsp.
		\end{align}
		Finally, for the last term, we have by definition of $L_\eta = \Ic(\mum_{\eta})/d$,
		\begin{equation}
			(4) : \Ic(\mum_{\Tf-t_K}) h_{K-1} = \Ic(\mum_{\eta}) c/L_\eta = cd L_\eta/L_\eta = cd \eqsp.
		\end{equation}
		Plugging all the bounds of $(1) - (2) - (3) - (4)$ into $F_{3.1}$ gives
		\begin{equation}
			F_{3.1} \lesssim {cd} [m+1+ (m+1)\log(L_\eta)] +  md (\rme^{c}-1)\Tf \eqsp.
		\end{equation}
		Noting that $c \leq \rme^c-1$ and $m \geq 1$, $L_\eta \geq 2$ yield
		\begin{align}
			F_{3.1} \lesssim (\rme^{c}-1)[{dm} \log(L_\eta) + md  \Tf] \eqsp.
		\end{align}
		Therefore, the ultimate sampling error admits the following expression
		\begin{align}
			\KL(\mum_\eta|\mathrm{Law}(\baXmstar_{\Tf-\eta}))
			&\lesssim  d \rme^{-\Tf} (\Tf+\log (m) ) + \varepsilon^2\Tf  \\
			&\quad+ (\rme^c-1) [dm\log(L_\eta) + d m\Tf ]\eqsp,
		\end{align}
		where we replace $\alpha_t = \rme^{-t}$ for any $t \in [0,\Tf]$.
		Now note that
		\begin{align}
			d \rme^{-\Tf} (\Tf+\log (m) ) \lesssim d \rme^{-\Tf}(\rme^{\Tf/2}+\log(m)) \overset{m \geq 2}{\lesssim} d\log(m)\rme^{-\Tf/2} \eqsp,
		\end{align}
		and by \Cref{lem:bound_fisher_masked}, $L_\eta \leq m+1/\eta$.
		Plugging it into the overall error yields
		\begin{align}
			\KL(\mum_\eta|\mathrm{Law}(\baXmstar_{\Tf-\eta}))
			&\lesssim  d\log(m)\rme^{-\Tf/2} + \varepsilon^2\Tf  \\
			&\quad (\rme^c-1)dm [\log(m+\eta^{-1})+  \Tf ]\eqsp.
		\end{align}
		Combining this with
		\[
		\tvnorm{\mum_\eta - \mustar} \leq 1-\alpha_\eta^d = 1-\rme^{-d\eta} \leq d\eta \eqsp,
		\]
		then invoking triangle and Pinsker's inequalities gives
		\begin{align}
			&\quad\tvnorm{\mustar - \mathrm{Law}(\baXmstar_{\Tf-\eta})} \\
			&\lesssim d\eta + \sqrt{ d\log(m)\rme^{-\Tf/2} + \varepsilon^2\Tf  +(\rme^c-1)dm [\log(m+\eta^{-1}) + \Tf ]} \eqsp.
		\end{align}
		Moreover, choosing $\Tf, \eta, c$ as in \eqref{eq:cor_T_eta_masked_d}, \eqref{eq:cor_c_masked_d} immediately yields
		\begin{align}
			\tvnorm{\mustar- \mathrm{Law}(\baXmstar_{ \Tf-\eta} ) } \lesssim  \varepsilon+ \sqrt{\varepsilon^2\Tf} \lesssim \varepsilon+ \vareps\sqrt{\log (d\log(m)/\varepsilon^2)} = \tilde{O}(\vareps) \eqsp,
		\end{align}
		and the number of iterations is given by
		\begin{align}
			K = k_0 + k_1 +K-k_0-k_1 &\overset{\eqref{eq:bound_k0k1k2_masked}}{\lesssim} \frac{\Tf -\eta - 1}{c} + \frac{\log(L_\eta)}{c} + \frac{1}{c} = \frac{\Tf + \log(m+\eta^{-1})}{c}\eqsp,
		\end{align}
		where $\Tf$, $\eta$, and $c$ are specified in \eqref{eq:cor_T_eta_masked_d}, \eqref{eq:cor_c_masked_d}. Note that $1/\log(1+h))$ has the complexity $ \tilde O(1/h)$ for $h \approx 0$, therefore $K = \tilde O(dm/\varepsilon^2)$ and we complete the proof of \Cref{theo:exp_stepsize_masked}.
	\end{proof}
	
	\subsection{Random walk on \texorpdfstring{$\Z^d_m$}{Zdm} and Biased random walk on \texorpdfstring{$\N^d$}{Nd}}\label{sec:main_proofs_rw_brw}
	
	The convergence proof relies on the monotonicity of the score, which is obtained through its characterization, along with several supporting lemmas presented below.
	
	\subsubsection{Characterization of the score}
	\begin{lemma}\label{lem:formula_u_brw}
		The time reversal process $(\baXbrw_t)_{t\in [0,\Tf]}$ has a backward generator $(\baqbrw_t)_{t\in[0,\Tf]}$ of the form \eqref{eq:backward_generator_ctmc}, with $\tqbrw_t = \foqbrw$ for any $t\in [0,\Tf]$ and for $(t,x,y) \in [0,\Tf) \times \N^d \times \N^d$,
		\begin{align}\label{eq:formula_u}
			\uBRW_t(x,y) = \begin{cases}
				\quad{\tmubrw_{\Tf-t}(\sigma(x))}/{\tmubrw_{\Tf-t}(x)} \quad &\text{if } y = \sigma(x) \text{ for } \sigma \in \mcs \eqsp,\\
				-{\sum_{\sigma \in \mcs} \uBRW_t \foqbrw(x,\sigma(x)) }/{\foqbrw(x,x)} \quad  &\text{if } y =x \eqsp,\\
				\qquad \qquad 1 \quad &\text{otherwise}\eqsp,
			\end{cases}
		\end{align}
		where $\tmubrw \eqdef \mubrw/\brw$ denotes the relative density with $\brw$ corresponds to the invariant measure of $(\foXbrw_t)_{t\in [0,\Tf]}$.
		Furthermore, we can express $\uBRW_t$ as
		\begin{equation}
			\uBRW_t (x,\sigma(x)) = \rme^{\Vbrw_t(x) - \Vbrw_t(\sigma(x))} \eqsp,
		\end{equation}
		with $\Vbrw_t(x) = -\log \tmubrw_{\Tf-t}(x)$ satisfying the following HJB equation
		\begin{align}\label{hjb_brw}
			\begin{cases}
				\partial_t \Vbrw_t(x)- \sum_{\sigma \in \mcs} \foqbrw (x,\sigma(x)) [\uBRW_t(x, \sigma(x))-1] = 0 \eqsp,\\
				\Vbrw_{\Tf}(x)= -\log \tmustar (x)\eqsp.
			\end{cases}
		\end{align}
	\end{lemma}
	\begin{remark}
		The preceding result remains valid for the process $(\foXrw_t)_{t\in [0,\Tf]}$, whose invariant measure coincides with the uniform distribution on $\Z^d_m$. As a consequence, the expression for $\uRW$ admits the simplified form
		\begin{align}
			\uRW_t(x,\sigma(x)) =  {\murw_{\Tf-t}(\sigma(x))}/{\murw_{\Tf-t}(x)} \eqsp, \quad \text{for all $\sigma \in \mcs$ and $(t,x) \in [0,\Tf] \times \Z^d_m$.}
		\end{align}
	\end{remark}
	\begin{proof}[Proof of~\Cref{lem:formula_u_brw}]
		Recall that the backward generator $(\baqbrw_t)_{t\in [0,\Tf]}$ satisfies~\eqref{eq:backward_generator_ctmc}, \ie, for $t \in [0,\Tf]$ and $x \neq y \in \N^d$, the following holds
		\begin{equation}\label{eq:balance_brw}
			\mubrw_{\Tf-t}(x)\baqbrw_t(x,y) = \mubrw_{\Tf-t}(y)\foqbrw_{\Tf-t}(y,x) = \mubrw_{\Tf-t}(y)\foqbrw(y,x) \eqsp,
		\end{equation}
		since $\foqbrw$ is time-independent (see~\eqref{eq:forward_generator_brw}).
		In addition, the forward generator $(\foqbrw_t)_{t\in[0,\Tf]}$ satisfies~\Cref{ass:ctcm},~\Cref{ass:non_explo} and irreducible, which implies $\mubrw_t(x)>0$ for any $t\in (0,\Tf]$ and $x \in \N^d$.
		Therefore,~\eqref{eq:balance_brw} implies: for any $t \in [0,\Tf)$ and $x \neq y \in \N^d$,
		\begin{equation}\label{eq:baq_brw}
			\baqbrw_t (x,y) = \frac{\mubrw_{\Tf-t}(y)}{\mubrw_{\Tf-t}(x)} \foqbrw(y,x) \eqsp.
		\end{equation}
		In order to interchange $x$ and $y$ on the right-hand side, observe that $\foqbrw$ satisfies the following balance equation: for all $x \neq y \in \N^d$,
		\begin{equation}\label{eq:balance_invariant_brw}
			\brw (x) \foqbrw(x,y) = \brw (y)\foqbrw (y,x) \eqsp,
		\end{equation}
		where $\brw = \mathrm{Poisson}(1)^{\otimes d}$ denotes the invariant measure of $(\foXbrw_t)_{t\in [0,\Tf]}$. Therefore
		\begin{align}
			\foqbrw (y,x) = \frac{\brw (y)}{\brw (x)}\foqbrw(x,y) \eqsp,
		\end{align}
		and replacing it into \eqref{eq:baq_brw} yields
		\begin{equation}
			\baqbrw_t(x,y) = \foqbrw (x,y) \frac{\brw (x) \mubrw_{\Tf-t}(y)}{\brw (y) \mubrw_{\Tf-t}(x)} = \foqbrw(x,y) \frac{\tmubrw_{\Tf-t} (y)}{\tmubrw_{\Tf-t}(x)} \eqsp,
		\end{equation}
		where $\tmubrw_t (x) \eqdef {\mubrw_t(x)}/{\brw(x)}$ for $(t,x) \in [0,\Tf) \times \N^d$ denotes the relative density of the forward dynamic.
		Note that we only need to consider transitions where $\foqbrw(x,y)\neq 0$, that is, when $y \in \{ \sigma^\ell_+(x), \sigma^\ell_- (x)\}$ for some $\ell  \in [d]$. Hence, it suffices to compute $\uBRW_t$ in this setting. Define
		\begin{equation}
			\uBRW_t(x,\sigma(x)) \eqdef
			\frac{\tmubrw_{\Tf-t} (\sigma(x))}{ \tmubrw_{\Tf-t}(x)} = \rme^{\Vbrw_t(x)-\Vbrw_t(\sigma(x))}  \quad \text{for } \sigma \in \mcs \eqsp,
		\end{equation}
		where $\Vbrw_t(x) \eqdef -\log \tmubrw_{\Tf-t}(x)$ for any $t\in [0,\Tf)$ and $x \in \N^d$. Otherwise, for $y \notin \{\sigma(x),x \,: \, \sigma \in \mcs \}$, we impose $\uBRW_t(x,y)=1$ for any $t \in [0,\Tf)$. Finally, the convention
		\begin{equation}
			\uBRW_t(x,x) \eqdef -\frac{\sum_{y \neq x} \uBRW_t(x,y) \foqbrw (x,y)}{\foqbrw(x,x)} = -\frac{\sum_{\sigma \in \mcs} \uBRW_t \foqbrw (x,\sigma(x)) }{\foqbrw(x,x)}
		\end{equation}
		ensures that $(\uBRW_t \foqbrw)_{t\in [0,\Tf)}$ in fact forms a generator $(\baqbrw_t)_{t\in [0,\Tf)}$.\\
		In addition, $\Vbrw_t$ satisfies the terminal condition
		\begin{align}
			\Vbrw_{\Tf} (x)=-\log \tmubrw_0(x) = -\log \tmustar (x) \eqsp,
		\end{align}
		and the following equation: for $(t,x) \in [0,\Tf) \times \N^d$,
		\begin{align}
			\partial_t \Vbrw_t(x) &= \frac{\partial_t  \mubrw_{\Tf-t}(x)}{ \mubrw_{\Tf-t}(x)} \overset{\eqref{eq:foward_kolm}}{=} \frac{\sum_{y\in \N^d} \mubrw_{\Tf-t}(y) \foqbrw(y,x)}{\mubrw_{\Tf-t}(x)} \\
			&\overset{\eqref{eq:balance_invariant_brw}}{=} \frac{\sum_{y\in \N^d} \mubrw_{\Tf-t}(y) \foqbrw(x,y) \brw (x)/\brw (y)}{\mubrw_{\Tf-t}(x)} \\
			&= \sum_{y \in \N^d} \frac{\tmubrw_{\Tf-t} (y)}{ \tmubrw_{\Tf-t}(x)} \foqbrw (x,y) \\
			&= \sum_{\sigma \in \mcs} \foqbrw (x, \sigma(x)) [\uBRW_t(x, \sigma(x)) -1 ] \eqsp,
		\end{align}
		which concludes the proof of \Cref{lem:formula_u_brw}.
	\end{proof}
	
	The explicit formula of the score derived in \Cref{lem:formula_u_brw} allows us to obtain the following properties, which will be used later in the theoretical results:
	\begin{lemma}\label{lem:2_brw}
		For any $t \in [0,\Tf)$ and $\ell \in [d]$, the following holds
		\begin{align}
			\E \l[ \uBRW_t \foqbrw (\baXbrw_t, \sigma^\ell_-(\baXbrw_t)) \r] = \E \l[\foqbrw (\baXbrw_t, \sigma^\ell_+(\baXbrw_t)) \r] \eqsp,
		\end{align}
		and reversely.
	\end{lemma}
	
	\begin{proof}[Proof of~\Cref{lem:2_brw}]
		It is given in Appendix \ref{proof_lem:2_brw}.
	\end{proof}

	\begin{lemma}\label{lem:1_brw}
		For any $t \in [0,\Tf)$, the following holds for any fixed $\ell  \in [d]$:
		\begin{align}
			&\E \l[ \uBRW_t \log \uBRW_t \foqbrw (\baXbrw_t, \sigma^\ell_+(\baXbrw_t)) \r] \\
			&\quad = -  \E \l[  \log \uBRW_t \foqbrw (\baXbrw_t, \sigma^\ell_-(\baXbrw_t)) \r] \eqsp, 
		\end{align}
		and reversely.
	\end{lemma}
	
	\begin{proof}[Proof of~\Cref{lem:1_brw}]
	It is given in Appendix \ref{proof_lem:1_brw}.
	\end{proof}

	\subsubsection{Evolution of the score along the backward dynamic}
	\begin{lemma}\label{prop:3_rw}
		Under \Cref{ass:finite_fisher_randomwalk}, the process $ (\uRW_t(\baXrw_t,\sigma(\baXrw_t)))_{t\in\ccint{0,\Tf}}$ is a martingale for any fixed $ \sigma \in \mcs$.
	\end{lemma}
	\begin{proof}[Proof of~\Cref{prop:3_rw}]
		
		Fix $t \in [0,\Tf]$, $ \sigma \in \mcs$, applying It\^o's formula on $f^{\sigma}(t,\baXrw_t )$ with 
		\[
		f^{\sigma}(t,x) \eqdef \uRW_t(x, \sigma(x))
		\]
		and noting that $\baXrw_t=\baXrw_{t-}$ for Lebesgue almost every $t \in [0,\Tf]$, we obtain that
		\begin{align}
			&f^{\sigma}(t,\baXrw_t )-f^{\sigma}(0,\baXrw_0 ) = \Mrw_\sigma(t) +\int_0^t\l[ \partial_s f^{\sigma}(s,\baXrw_s )+(\baqrw f^\sigma)(s,\baXrw_{s} ) \r]\rmd s \eqsp,
		\end{align}
		with
		\begin{align}
			\Mrw_\sigma(t) = \int_{[0,t] \times \Z^d_m} \l[ f^{\sigma}(s,x) - f^\sigma(s, \baXrw_{s-}) \r] \tilde{N}_{\canoX}^{\baqrw} (\rmd x \rmd s)
		\end{align}
		is a true martingale  (see \Cref{lem:martingale_rw} for detailed proof). 
		We now simplify the integrand by using \Cref{lem:formula_u_brw} to derive the partial differential equation (PDE) satisfied by $f^\sigma$ first: for $s \in [0,\Tf]$, 
		\begin{align}
			\partial_s f^{\sigma}(s,x) 
			&= f^{\sigma}(s,x)[\partial_s \Vrw_s(x) - \partial_s \Vrw_s(\sigma(x))] \\
			&= \sum_{\sigma' \in \mcs} f^{\sigma}(s,x)\foqrw(x, \sigma'(x))[f^{\sigma'}(s,x)-f^{\sigma'}(s,\sigma(x))] \\
			&= \sum_{\sigma' \in \mcs} \baqrw_s (x, \sigma'(x)) [f^\sigma(s,x)  - f^{\sigma}(s,\sigma'(x))] = -(\baqrw f^\sigma)(s,x) \eqsp,
		\end{align}
		since $f^{\sigma}(s,x)f^{\sigma'}(s,\sigma(x))  = f^{\sigma'}(s,x)f^{\sigma}(s, \sigma'(x))$ as $\sigma' \circ \sigma = \sigma \circ \sigma'$ for any $\sigma, \sigma' \in \mcs$. 
		Therefore $f^\sigma(t, \baXrw_t)$ is a martingale and we conclude the proof.
		
	\end{proof}

	\begin{lemma}\label{prop:3_brw}
		Assume $\m^\star_2<\infty$ and for fixed $\sigma \in \mcs$, $0\leq r \leq t <\Tf$, denote
		\[y_r^\sigma(t):= \E \l[ \foqbrw \uBRW_t (\baXbrw_t, \sigma (\baXbrw_t)) -1 | \Fc_r\r] \eqsp,\] where $(\Fc_t)_{t\in [0,\Tf]}$ denotes the filtration generated by $(\baXbrw_t)_{t\in [0,\Tf]}$. Then it holds for any $0 \leq r \leq t < \Tf$ and $\ell \in [d]$:
		\begin{equation}
			y_r^{\sigma^\ell_+} (t) = y_r^{\sigma^\ell_+} (r) \rme^{t-r}   \quad \text{and} \quad y_r^{\sigma^\ell_-} (t)  =  y_r^{\sigma^\ell_-} (r) \rme^{-(t-r)} \eqsp.
		\end{equation}
	\end{lemma}
	
	\begin{proof}[Proof of~\Cref{prop:3_brw}]
		For $t \in [0,\Tf)$, $\sigma \in \mcs$, applying It\^o's formula on $\mathbf{f}^{{\sigma}}(t,\baXbrw_{t} )$ with
		\begin{align}
			\mathbf{f}^{{\sigma}}(t,x )
			:&= \foqbrw \uBRW_t(x, \sigma(x)) -1 = \baqbrw_t (x, \sigma(x)) \eqsp,
		\end{align}
		and noting that $\baXbrw_t=\baXbrw_{t-}$ for Lebesgue almost every $t \in (0,\Tf]$, we get
		\begin{align}\label{eq:martingale_brw}
			& \f^{{\sigma}}(t,\baXbrw_{t} )-\f^{{\sigma}}(0,\baXbrw_{0} ) = \Mbrw_\sigma(t) +  \int_0^t\Bigg[ \partial_s \f^{\sigma} + (\baqbrw \f^\sigma) \Bigg]( s, \baXbrw_{s}) \rmd s \eqsp,
		\end{align}
		with
		\begin{align}
			\Mbrw_\sigma(t) = \int_{[0,t] \times \N^d} \l[ \f^\sigma(s,x) - \f^\sigma(s, \baXbrw_{s-}) \r] \tilde{N}_{\baXbrw}^{\baqbrw} (\rmd x \rmd s)
		\end{align}
		is a martingale; see \Cref{lem:true_martingale_brw} for completeness.
		Let us simplify the integrand by using \Cref{lem:formula_u_brw}:
		\begin{align}
			\bbf^\sigma_s(x) &= \partial_s \f^{\sigma}(s,x) + (\baqbrw \f^\sigma)(s,x) \\
			&= \baqbrw_s(x, \sigma(x)) [\partial_s \Vbrw_s(x) - \partial_s \Vbrw_s(\sigma(x))] + (\baqbrw \f^\sigma )(s, x) \\
			&= \baqbrw_s(x, \sigma(x)) \sum_{\sigma' \in \mcs} [ \foqbrw(\sigma(x), \sigma'(\sigma(x))) - \foqbrw(x, \sigma'(x))]\\
			&\quad + \sum_{\sigma' \in \mcs} \baqbrw_s(x, \sigma(x)) [\baqbrw_s(x, \sigma'(x)) - \baqbrw_s(\sigma(x), \sigma'(\sigma(x)))] \\
			& \quad+ \sum_{\sigma' \in \mcs} \baqbrw_s(x, \sigma'(x)) [\baqbrw_s(\sigma'(x), \sigma(\sigma'(x))) - \baqbrw_s(x, \sigma(x))] \\
			&= \baqbrw_s(x, \sigma(x)) \sum_{\sigma' \in \mcs} [ \foqbrw(\sigma(x), \sigma'(\sigma(x))) - \foqbrw(x, \sigma'(x))]\\
			&\quad + \sum_{\sigma' \in \mcs} [ \foqbrw \uBRW_s(x, \sigma'(x)) \foqbrw \uBRW_s (\sigma'(x), \sigma(\sigma'(x))) \\
			&\qquad- \foqbrw \uBRW_s(x, \sigma(x)) \foqbrw \uBRW_s (\sigma(x), \sigma'(\sigma(x)))] \eqsp.
		\end{align}
		Specifically, for fixed $\ell \in [d]$, we can compute $\bbf^{\sigma^\ell_+}_s(x)$ and $\bbf^{\sigma^\ell_-}_s(x)$ as follows
		\begin{align}
			\bbf^{\sigma^\ell_+}_s(x) &= \baqbrw_s (x, \sigma^\ell_+)  + x^\ell - (x^\ell+1) = \f^{\sigma^\ell_+}(s,x) \eqsp,
		\end{align}
		and
		\begin{equation}
			\bbf^{\sigma^\ell_-}_s(x) = -\baqbrw_s(x, \sigma^\ell_-)  + x^\ell+1 - x^\ell = - \f^{\sigma^\ell_-}(s,x) \eqsp.
		\end{equation}
		Substituting the above into~\eqref{eq:martingale_brw} and taking the expectation yields
		\begin{equation}
			\E\l[ \f^{\sigma^\ell_+}(t,\baXbrw_{t} ) \middle| \Fc_r \r]- \E \l[\f^{\sigma^\ell_+}(0,\baXbrw_{0} ) \middle| \Fc_r \r] =  \int_0^t \E \l[ \f^{{\sigma^\ell_+}}(s,\baXbrw_{s} )  ) \middle| \Fc_r\r] \rmd s \eqsp,
		\end{equation}
		and
		\begin{equation}
			\E\l[ \f^{\sigma^\ell_-}(t,\baXbrw_{t} )  \middle| \Fc_r \r]- \E \l[ \f^{{\sigma^\ell_-}}(0,\baXbrw_{0} )   \middle| \Fc_r \r] = - \int_0^t \E \l[ \f^{{\sigma^\ell_-}}(s,\baXbrw_{s} )  \middle| \Fc_r \r] \rmd s \eqsp,
		\end{equation}
		for any $0 \leq r \leq t < \Tf$. For $r \leq t <\Tf$ and $\sigma \in \mcs$,
		define the function
		\begin{equation}
			y^{\sigma}_r (t) \eqdef  \E\l[ \f^{{\sigma}}(t,\baXbrw_{t} ) \middle| \Fc_r \r]  \eqsp.
		\end{equation}
		Then the previous computation implies
		\begin{equation}
			y_r^{\sigma^\ell_+}(t) - y_r^{\sigma^\ell_+}(0) = \int_0^t  y_r^{\sigma^\ell_+}(r)  \rmd r \quad \text{and} \quad   y_r^{\sigma^\ell_-}(t) - y_r^{\sigma^\ell_-}(0) = - \int_0^t  y_r^{\sigma^\ell_-}(r) \rmd r  \eqsp.
		\end{equation}
		Divide the both hand sides by $t$ and let $t \to 0^+$, we obtain
		\begin{equation}
			\frac{\rmd }{\rmd t} y_r^{\sigma^\ell_+}(t) = y_r^{\sigma^\ell_+}(t) \quad \text{and} \quad \frac{\rmd }{\rmd t} y_r^{\sigma^\ell_-}(t) = -y_r^{\sigma^\ell_-}(t) \eqsp.
		\end{equation}
		Solving these ODEs yields
		\[y_r^{\sigma^\ell_+} (t) =y_r^{\sigma^\ell_+}(r) \rme^{t-r} \quad \text{and} \quad y_r^{\sigma^\ell_-} (t) =y_r^{\sigma^\ell_-}(r) \rme^{-(t-r)} \eqsp, \quad  \text{for any } t\in [r, \Tf) \eqsp,\]
		which concludes the proof of~\Cref{prop:3_brw}.
	\end{proof}
	
	\begin{lemma}\label{lem:evolution_q_brw}
		Assume $\mstar_1<\infty$ and for fixed $\ell \in [d]$, $0\leq r \leq t \leq \Tf$, denote
		\[\psi^\ell_r(t) \eqdef \E \l[ \foqbrw(\baXbrw_t, \sigma^\ell_-(\baXbrw_t)) -1 \middle| \Fc_r \r] \eqsp,\]
		then it holds
		\begin{align}\label{eq:evolution_q}
			\psi^\ell_r(t) = \psi^\ell_r(s) \rme^{t-s} \quad \text{for $0 \leq r \leq s \leq t \leq \Tf$} \eqsp.
		\end{align}
	\end{lemma}
	\begin{proof}[Proof of \Cref{lem:evolution_q_brw}]
		By applying It\^o’s formula on \[\foqbrw(\baXbrw_t,\sigma^\ell_-(\baXbrw_t)) = (\baXbrw_t)^{\ell}=(\foXbrw_{\Tf-t})^\ell\] for each fixed $\ell \in [d]$ and noting that $\foXbrw_t = \foXbrw_{t-}$ for Lebesgue almost every $t\in (0,\Tf]$, we obtain
		\begin{align}
			&\quad(\foXbrw_{\Tf-t})^{\ell} -  (\foXbrw_{\Tf})^\ell\\
			&= \bfMbrw_\ell(t)  + \int_{\Tf}^{\Tf-t} \sum_{\sigma \in \mcs} \l[ (\sigma(\foXbrw_s))^\ell - (\foXbrw_s)^\ell \r] \foqbrw(\foXbrw_s, \sigma(\foXbrw_s)) \rmd s  \eqsp,
		\end{align}
		where
		\begin{align}
			\bfMbrw_\ell(t) = \int_{[\Tf-t,\Tf] \times \N^d} \l[ x^\ell - (\foXbrw_{s-})^\ell \r] \tilde{N}_{\foXbrw}^{\foqbrw} (\rmd x\rmd s)
		\end{align}
		is a local martingale.
		Since
		\[\tilde{N}^{\foqbrw}_{\foXbrw} = {N}_{\foXbrw}^{\foqbrw} - \bar \rmn_{\foXbrw}^{\foqbrw} \eqsp,\]
		where $\bar \rmn_{\foXbrw}^{\foqbrw}$ is the compensator of ${N}_{\foXbrw}^{\foqbrw}$, we obtain that
		\begin{align}
			&\quad \E[|\bfMbrw_\ell(t)|]\\
			 &\leq \E \l[ \int_{[\Tf-t,\Tf] \times \N^d} | x^\ell - (\foXbrw_{s-})^\ell  | {N}_{\foXbrw}^{\foqbrw} (\rmd s \rmd x) \r] \\
			&\quad+ \E \l[ \int_{[\Tf-t,\Tf] \times \N^d} | x^\ell - (\foXbrw_{s-})^\ell | \bar \rmn_{\foXbrw}^{\foqbrw} (\rmd s \rmd x) \r] \\
			&= 2\E \l[ \int_{[\Tf-t,\Tf]} \sum_{\sigma \in \mcs} \l| (\sigma(\foXbrw_s))^\ell- (\foXbrw_s)^\ell \r| \foqbrw(\foXbrw_{s},\sigma(\foXbrw_{s})) \rmd s \r]\\
			&= 2 \E \l[ \int_{[\Tf-t,\Tf]} \l( (\foXbrw_{s})^\ell +1 \r) \rmd s \r] \eqsp.
		\end{align}
The assumption $\mstar_1 <\infty$ implies $\m_1(\mubrw_t) <\infty$ for any $t \in [0,\Tf]$ (see \Cref{prop:finite_first_moment_brw} for completeness), which yields $\E \l[(\foXbrw_t)^\ell \r]<\infty$ for any $\ell \in [d]$. Therefore
		\begin{align}
			\E[|\bfMbrw_\ell(t)|] &\leq 2  \int_{[\Tf-t,\Tf]} \l( \E \l[(\foXbrw_{s})^\ell \r] +1 \r) \rmd s <\infty \eqsp.
		\end{align}
		As a result, $\bfMbrw_\ell(t)$ is a true martingale. It follows that
		\begin{align}
			&(\baXbrw_{t})^{\ell} -  (\baXbrw_{0})^\ell \\
			&\quad - \int_{\Tf-t_k}^{\Tf-t} \l( \foqbrw (\foXbrw_s, \sigma^\ell_+(\foXbrw_s)) - \foqbrw(\foXbrw_s, \sigma^\ell_-(\foXbrw_s)) \r) \rmd s
		\end{align}
		is a martingale. Changing the variable in the integral, the following process is a martingale:
		\begin{equation}
			(\baXbrw_{t})^{\ell} -  (\baXbrw_{0})^\ell
			- \int_{0}^{t} \sum_{\sigma \in \mcs} \l( \foqbrw (\baXbrw_s, \sigma^\ell_-(\baXbrw_s)) - 1 \r)  \rmd s \eqsp.
		\end{equation}
		Denote $ \psi^\ell_r(t) \eqdef \E \l[ \foqbrw(\baXbrw_t, \sigma^\ell_-(\baXbrw_t)) -1 \middle| \Fc_r \r]$ for $0 \leq r \leq t \leq \Tf$, then the previous equation implies
		\begin{equation}
			\frac{\rmd }{\rmd t}\psi^\ell_r(t) = \psi^\ell_r(t) \eqsp.
		\end{equation}
		As a consequence, for any $0 \leq r \leq s \leq t \leq \Tf$ we have
		\begin{equation}
			\psi^\ell_r(t) = \psi^\ell_r(s)\rme^{t-s} \eqsp,
		\end{equation}
		which concludes the proof.
	\end{proof}
	
	\begin{lemma}\label{prop:4_brw}
		$\Ic_{\brw}(\mubrw_{\Tf-t}) = \E \l[\sum_{\sigma \in \mcs}\mathbf{h} (\uBRW_t)\foqbrw (\baXbrw_t, \sigma(\baXbrw_t) ) \r]$ is non-decreasing on $[0,\Tf]$, where $\mathbf{h}(a) = a \log a -a +1$.
	\end{lemma}
	
	\begin{proof}[Proof of~\Cref{prop:4_brw}]
		It is provided in Appendix \ref{sec:proof_prop:4_brw}.
	\end{proof}
	
	\begin{remark}
		The preceding monotonicity also applies for 
		\[\Ic (\murw_{\Tf-t}) = \E \l[ \sum_{\sigma \in \mcs} \mathbf{h}(\uRW_t (\foXrw_{\Tf-t},\sigma (\foXrw_{\Tf-t}))) \r] \eqsp.
		\]
	\end{remark}
	
	The monotonicity established in \Cref{prop:4_brw} leads to the following bound on the discrete Fisher information, which will be used to reduce the complexity of DDMs later.
	
	\begin{lemma}\label{prop:bound_fisher_rw}
		For any $t \in [0,\Tf)$,
		\begin{align}
			\Ic(\murw_{\Tf-t})
			&\lesssim (\Tf-t)^{-1}d\log (m) \eqsp. \label{eq:bound_fisher_rw}
		\end{align}
	\end{lemma}
	
	\begin{proof}[Proof of \Cref{prop:bound_fisher_rw}]
	It is given in Appendix \ref{proof_prop:bound_fisher_rw}.
	\end{proof}
	\begin{lemma}\label{prop:bound_fisher_brw}
		Assume $\m^\star_2<\infty$ then the following holds for any $t \in [0,\Tf)$:
		\begin{align}
			\Ic_{\brw}(\mubrw_{\Tf-t})
			&\leq (\Tf-t)^{-1} (d+\m_2^\star ) \eqsp. \label{eq:bound_fisher_brw}
		\end{align}
	\end{lemma}
	
	\begin{proof}[Proof of \Cref{prop:bound_fisher_brw}]
		It is given in Appendix \ref{proof_prop:bound_fisher_brw}.
	\end{proof}

	\subsubsection{Convergence proofs}\label{sec:rw_brw_proof_conv}
	
	\begin{proof}[Proof of~\Cref{theo:5_randomwalk}]
		We show first the bound for the “distance” between the backward path measure $\baPrw$ of $(\baXrw_t)_{t\in [0,\Tf]}$ and $\baPrwstar$ of the simulated backward process $(\baXrwstar_t)_{t\in[0,\Tf]}$. Consider $\baPrwstar \in \mathrm{MP}(\baqrwstar)$ as the reference measure in Girsanov's theorem, we have
		\begin{align}
			&\KL(\baPrw| \baPrwstar) \\
			 &\quad =\KL(\murw_{\Tf}|\rw) 
			 +\E \l[ \int_{[0,\Tf]}\sum_{x\in \Z^d_m} \mathbf{h} \l( \frac{\baqrw_t}{\baqrwstar_t} \r) \baqrwstar_t (\baXrw_{t},x) \1_{\baXrw_{t} \neq x} \rmd t \r] \eqsp.
		\end{align}
		With a partition $0=t_0<...<t_K=\Tf$ for $K \geq 1$ of $[0,\Tf]$ associated with the sequence of step-size $h_{k+1} = t_{k+1} - t_k$, the previous expression rewrites as
		\begin{align}
			& \KL(\baPrw| \baPrwstar) \\
			 & =\KL(\murw_{\Tf}|\rw)
			 +\sum_{k=0}^{K-1} \E \Bigg[ \int_{[t_k,t_{k+1})}\sum_{x\in \Z^d_m}  \mathbf{h}\l( \frac{\baqrw_t}{\baqrwstar_t} \r) \baqrwstar_t (\baXrw_{t},x) \1_{\baXrw_{t} \neq x} \rmd t \Bigg] \eqsp.
		\end{align}
		Substituting the expressions of $(\baqrw_t)_{t\in[0,\Tf)}$ from~\Cref{lem:formula_u_brw} and $(\baqrwstar_t)_{t\in[0,\Tf)}$ from \eqref{eq:generator_generative_ctmc},~\eqref{eq:control_generative_rw} into the preceding equation, where $\foqrw$ given in~\eqref{eq:forward_generator_rw}, we obtain that
		\begin{align}
			&\KL(\baPrw|  \baPrwstar)&  \notag \\
			&\quad =\KL(\murw_{\Tf}|\rw) + \frac{1}{2}\sum_{k=0}^{K-1}\E  \l[ \int_{[t_k,t_{k+1})}\sum_{\sigma \in \mcs} \mathbf{h}\l(
			\dfrac{\uRW_t(\sigma)}{\uRWstar_{t_k}(\sigma)}\r)\uRWstar_{t_k}(\sigma) \rmd t\r] \eqsp, \label{eq:3_rw}
		\end{align}
		where we write $\uRW_t(\baXrw_{t},\sigma(\baXrw_t))$ as $\uRW_t(\sigma)$,  $\uRWstar_{t_k}(\baXrw_{t_k},\sigma(\baXrw_{t_k}))$ as $\uRWstar_{t_k}(\sigma)$ for short. We estimate~\eqref{eq:3_rw} by decomposing it into three following addends
		\begin{align}
			&\KL(\baPrw|  \baPrwstar) \\
			&\quad \lesssim \underbrace{\KL(\murw_{\Tf}|\rw)}_{E_1}+ \underbrace{\sum_{k=0}^{K-1}\E  \l[ \int_{[t_k,t_{k+1})}\sum_{\sigma \in \mcs} (
				\mathbf{h}({\uRW_t}(\sigma))-\mathbf{h}(\uRW_{t_k}(\sigma) )) \rmd t \r]}_{E_2} \notag \\
			&\quad  + \underbrace{\sum_{k=0}^{K-1}\E  \l[ \int_{[t_k,t_{k+1})}\sum_{\sigma \in \mcs} \l(\mathbf{h}(\uRW_{t_k})-\mathbf{h}(\uRWstar_{t_k})+ (\uRWstar_{t_k} - \uRW_t) \log{\uRWstar_{t_k}} \r)(\sigma) \rmd t\r]}_{E_3} \eqsp,\label{eq:e1_e2_e3_rw}
		\end{align}
		where $\mathbf{h}(a) = a\log a-a+1$ for $a>0$.
		We bound $E_1-E_2-E_3$ one-by-one as follows. For $E_2$, using the monotonicity obtained in~\Cref{prop:4_brw}, we get
		\begin{align}
			E_2 
			&\leq \sum_{k=0}^{K-1}\E  \l[ \int_{[t_k,t_{k+1})}\sum_{\sigma \in \mcs} (
			\mathbf{h}(\uRW_{t_{k+1}}(\sigma))-\mathbf{h}(\uRW_{t_k}(\sigma) )) \rmd t \r] \\
			&\leq  h \sum_{k=0}^{K-1}  \l( \E  \l[ \sum_{\sigma \in \mcs}
			\mathbf{h}(\uRW_{t_{k+1}}(\sigma)) \r] -\E  \l[ \sum_{\sigma \in \mcs} \mathbf{h}(\uRW_{t_k}(\sigma) )  \r] \r) \eqsp,
		\end{align}
		where $h \eqdef \max_k \{t_{k+1}-t_k\}$. We obtain a telescoping sum on the right-hand side, which allows us to simplify it as
		\begin{align}
			E_2 &\leq h  \l( \E  \l[ \sum_{\sigma \in \mcs}
			\mathbf{h}(\uRW_{t_{K}}(\sigma)) \r] -\E  \l[ \sum_{\sigma \in \mcs} \mathbf{h}(\uRW_{t_0}(\sigma) )  \r] \r)\\
			&\leq h \E  \l[ \sum_{\sigma \in \mcs}
			\mathbf{h}(\uRW_{\Tf}(\sigma)) \r] = h \Ic(\mustar) \eqsp,
		\end{align}
		where we used the fact that $\mathbf{h}$ is a nonnegative function to remove the second term. Here $\Ic(\mustar)$ is a discrete Fisher information defined in~\eqref{eq:finite_fisher_randomwalk}, which is finite by~\Cref{ass:finite_fisher_randomwalk}.
		
		We evaluate next $E_3$ by the tower property again and the martingality of $(\uRW_t(\sigma))_{t\in [0,\Tf]}$ established in~\Cref{prop:3_rw}:
		\begin{align}
			E_3 &= \sum_{k=0}^{K-1}\E  \Bigg[ \int_{[t_k,t_{k+1})}\sum_{\sigma \in \mcs} \Bigg(\mathbf{h}(\uRW_{t_k}(\sigma))-\mathbf{h}(\uRWstar_{t_k}(\sigma)))\\
			&\qquad+ (\uRWstar_{t_k}(\sigma) - \E[\uRW_t(\sigma))|\Fc_{t_k}] \log{\uRWstar_{t_k}(\sigma)} \Bigg) \rmd t\Bigg] \\
			&= \sum_{k=0}^{K-1}\E  \Bigg[ \int_{[t_k,t_{k+1})}\sum_{\sigma \in \mcs} \Bigg(\mathbf{h}(\uRW_{t_k}(\sigma))-\mathbf{h}(\uRWstar_{t_k}(\sigma)))\\
			&\qquad \qquad + (\uRWstar_{t_k}(\sigma) - \uRW_{t_k}(\sigma)) \log{\uRWstar_{t_k}(\sigma)} \Bigg) \rmd t\Bigg] \\
			&= \sum_{k=0}^{K-1}h_{k+1} \E  \l[ \sum_{\sigma \in \mcs}  \l(\uRW_{t_k}(\sigma)\log \dfrac{\uRW_{t_k}(\sigma)}{\uRWstar_{t_k}(\sigma)}+\uRWstar_{t_k}(\sigma)-\uRW_{t_k}(\sigma)\r)  \r] \\
			&\leq \varepsilon^2\Tf \eqsp,
		\end{align}
		where the last inequality comes from~\Cref{ass:approx_score_randomwalk}. Subsequently, $E_1$ can be controlled by~\eqref{eq:entropy_decay_rw}:
		\begin{align}
			E_1 \leq \rme^{-\frac{16\pi^2}{25m^2}\Tf}\KL(\mustar| \rw) \eqsp.
		\end{align}
		Note that $\KL(\mustar|\rw)<\infty $ thanks to \Cref{ass:finite_fisher_randomwalk}.
		Combining all the upper bounds on $E_1,E_2,E_3$ implies
		\begin{align}
			\KL(\baPrw| \baPrwstar) \lesssim \rme^{-\frac{16\pi^2}{25m^2}\Tf}\KL(\mustar|\rw)+ h \Ic(\mustar)+  \varepsilon^2\Tf \eqsp,
		\end{align}
		Finally, notice that $\mustar=\mathrm{Law}(\baXrw_{\Tf})$, therefore
		\begin{align}
			\KL(\mustar|\mathrm{Law}(\baXrwstar_{\Tf}) )&=\KL(\mathrm{Law}(\baXrw_{\Tf})|\mathrm{Law}(\baXrwstar_{\Tf}) ) \\
			&\leq \KL(\mathrm{Law}((\baXrw_t)_{t\in[0,\Tf]})|\mathrm{Law}((\baXrwstar_t)_{t\in [0,\Tf]}) )\\
			&= \KL(\baPrw| \baPrwstar) \eqsp,
		\end{align}
		where the inequality is known as {Data processing} inequality for relative entropy \cite[Lemma 1.6]{nutz2021introduction}. We then conclude that
		\begin{align}
			\KL(\mustar|\mathrm{Law}(\baXrwstar_{\Tf} ) ) \lesssim \rme^{-\frac{16\pi^2}{25m^2}\Tf}\KL(\mustar|\rw)+ h \Ic(\mustar)+  \varepsilon^2\Tf \eqsp.
		\end{align}
		Moreover, since $\KL(\mustar|\rw) \leq d\log (m)$ (see Appendix \ref{proof_prop:bound_fisher_rw} for completeness), the previous estimate still holds if we replace $\KL(\mustar|\rw)$ by $d \log(m)$ and we complete the proof of \Cref{theo:5_randomwalk}.
	\end{proof}

	\begin{proof}[Proof of \Cref{theo:scale_rw}]
		
			We proceed analogously as \Cref{theo:5_randomwalk}, the only difference is the way we handle the term $E_2$ in \eqref{eq:e1_e2_e3_rw}. Recall that
			\begin{align}
				E_{2} &\leq \sum_{k=0}^{K-1}h_{k+1} \l( \E  \l[ \sum_{\sigma \in \mcs}
				\mathbf{h}(\uRW_{t_{k+1}}(\sigma)) \r] -\E  \l[ \sum_{\sigma \in \mcs} \mathbf{h}(\uRW_{t_k}(\sigma) )  \r] \r) \\
				& \leq \sum_{k=0}^{K-1}h_{k+1} \l( \Ic(\murw_{\Tf-t_{k+1}}) - \Ic(\murw_{\Tf-t_k}) \r)
			\end{align}
			where $\Ic(\murw_{\Tf-t}) = \E  \l[ \sum_{\sigma \in \mcs}
			\mathbf{h}(\uRW_{t}(\sigma)) \r]$ for $t \in [0,\Tf]$.
			Following precisely the proof of \Cref{theo:main_masked_d} and using the upper bound of $\Ic(\murw_{\Tf-t})$ in \Cref{prop:bound_fisher_rw}, choosing the sequence of step-size as
			$h_{k+1} = c \min \l\{ \max \l\{ \Tf-t_k, 1/L\r\}, 1 \r\}$ for $k < K-1$ and $h_K = \Tf-t_{K-1}$, with $L = d^{-1}\Ic (\mustar) \geq 2$ yields
			\begin{align}
				E_{2} \lesssim {cd \log(m)} \log (L) \eqsp.
			\end{align}
			Therefore, the ultimate sampling error admits the following expression
			\begin{align}
				\KL(\mustar|\mathrm{Law}(\baXrwstar_{\Tf}))
				&\lesssim  \rme^{-\frac{16\pi^2}{25m^2}\Tf}d\log(m)+  \varepsilon^2\Tf + {cd\log(m)}\log (L) \eqsp.
			\end{align}
			Finally, choosing $\Tf$ and $c$ as in \eqref{eq:cor_randomwalk_1} immediately implies
			\begin{align}
				\KL(\mustar|\mathrm{Law}(\baXrwstar_{\Tf}))
				&\lesssim \varepsilon^2+ m^2\varepsilon^2\log (d\log(m)/\varepsilon^2) \eqsp,
			\end{align}
			with the number of iterations is given by
			\begin{align}
				K = k_0 + k_1 +K-k_0-k_1 
				&\lesssim \frac{d \log (m)\log (L)[m^2\log(d\log(m)/\varepsilon^2)+\log (L)]}{\varepsilon^2} \eqsp,
			\end{align}
			and we complete the proof of \Cref{theo:scale_rw}.
		\end{proof}


		\begin{proof}[Proof of \Cref{theo:scale_brw}]
			Similarly to the proof of \Cref{theo:5_randomwalk}, we have
			\begin{align}\label{eq:3_brw}
				&\KL(\baPbrw|  \baPbrwstar) \notag \\
				 &\quad =\KL(\mubrw_{\Tf}|\brw) +  \sum_{k=0}^{K-1}\E  \l[ \int_{[t_k,t_{k+1})}\sum_{\sigma \in \mcs} 
				\mathbf{h} \l( \frac{\uBRW_t}{\uBRWstar_{t_k}}  \r) \uBRWstar_{t_k} \foqbrw_t(\sigma) \rmd t\r] \eqsp, 
			\end{align}
			where we write $\uBRW_t(\sigma) \eqdef \uBRW_t(\baXbrw_{t},\sigma(\baXbrw_t))$ for short and similarly for other terms. We decompose the previous expression into following three terms:
			\begin{align}
				&\quad\KL(\baPbrw|  \baPbrwstar)  = \underbrace{\KL(\mubrw_{\Tf}|\brw)}_{G_1} \notag \\
				&+ \underbrace{ \sum_{k=0}^{K-1}\E  \l[ \int_{t_k}^{t_{k+1}}\sum_{\sigma \in \mcs} 
					\l( \uBRW_t \foqbrw_t \log \frac{\uBRW_{t_k}}{\uBRWstar_{t_k}}  -\uBRW_{t_k}\foqbrw_{t_k} + \uBRWstar_{t_k} \foqbrw_{t_k}\r) (\sigma) \rmd t\r] }_{G_2} \notag \\
				&+ \underbrace{\sum_{k=0}^{K-1} \E \l[\int_{t_k}^{t_{k+1}} \sum_{\sigma \in \mcs} \l\{ \uBRW_t \foqbrw_t \log \frac{\uBRW_t}{\uBRW_{t_k}} + (\uBRWstar_{t_k}-1)(\foqbrw_t -\foqbrw_{t_k}) \r\}(\sigma) \rmd t\r] 
				}_{G_3} \label{eq:kl_path_brw} \eqsp,
			\end{align}
			where we used \Cref{lem:2_brw} to simplify the last addend.
			We control $G_1$ by \eqref{eq:entropy_decay_brw} and estimate $G_2$ by \Cref{ass:approx_score_brw}:
			\begin{align}
				G_1 \leq \rme^{-\Tf}\KL(\mustar|\brw)  \quad \text{and} \quad G_2 \leq \varepsilon^2\Tf \eqsp.
			\end{align}
				Note that $\KL(\mustar|\rw) < \infty$ since the invariant measure $\brw$ satisfies the convex Sobolev inequality with constant $1$ (see \cite[Theorem 1.1]{Wu2000ANM}, \cite[Section 3]{conforti2022probabilistic}), namely,
			\begin{align}
				\KL(\mustar|\brw) \leq \Ic_{\brw}(\mustar) \overset{\eqref{eq:finite_fisher_brw}}{<} \infty \eqsp.
			\end{align}
			It remains to upper bound $G_3$ to finish the proof. Note that $G_3$ can be simplified and decomposed into three following quantities:
			\begin{align}
				G_3 &= \underbrace{ \sum_{k=0}^{K-1}\E \l[ \sum_{\sigma \in \mcs}\int_{[t_k,t_{k+1})}\l(  \foqbrw_t \uBRW_t \log \uBRW_t(\sigma)
					-
					\foqbrw_{t_k}\uBRW_{t_k}\log \uBRW_{t_k}(\sigma) \r) \rmd t
					\r]}_{G_{3.1}} \\
				&+  \underbrace{ \sum_{k=0}^{K-1}\E \l[ \sum_{\sigma \in \mcs} \int_{[t_k,t_{k+1})}
					\log \uBRW_{t_k} \l( \uBRW_{t_k} \foqbrw_{t_k}  - \uBRW_t \foqbrw_t \r)(\sigma)
					\rmd t\r]}_{G_{3.2}} \\
				&+ \underbrace{\sum_{k=0}^{K-1}\E \l[ \sum_{\sigma \in \mcs} \int_{[t_k,t_{k+1})}
					(\uBRWstar_{t_k}-1) \l( \foqbrw_{t} -  \foqbrw_{t_k} \r) (\sigma)
					\rmd t\r]}_{G_{3.3}} \eqsp.
			\end{align}
			We evaluate the term $G_{3.2}$ by the tower property and the evolution obtained in~\Cref{prop:3_brw}.
			\begin{align}
				G_{3.2} &= \sum_{k=0}^{K-1} \sum_{\ell =1}^d \E \l[ \int_{[t_k,t_{k+1})}
				\log \uBRW_{t_k} \l( \uBRW_{t_k} \foqbrw_{t_k}  - \E \l[\uBRW_t \foqbrw_t\middle| \Fc_{t_k}\r] \r)(\sigma^\ell_+)  
				\rmd t \r]  \\
				& + \sum_{k=0}^{K-1} \sum_{\ell =1}^d \E \l[ \int_{[t_k,t_{k+1})}
				\log \uBRW_{t_k} \l( \uBRW_{t_k} \foqbrw_{t_k} - \E\l[ \uBRW_t \foqbrw_t \middle| \Fc_{t_k} \r]  \r)(\sigma^\ell_-)
				\rmd t \r]\\
				&= \sum_{k=0}^{K-1} \sum_{\ell =1}^d \E \l[ \int_{[t_k,t_{k+1})}
				\log \uBRW_{t_k}(\sigma^\ell_+) ( \uBRW_{t_k}\foqbrw_{t_k}(\sigma^\ell_+)  -1) (1  - \rme^{t-t_k})
				\rmd t \r]  \\
				&+ \sum_{k=0}^{K-1} \sum_{\ell =1}^d \E \l[ \int_{[t_k,t_{k+1})}
				\log \uBRW_{t_k}(\sigma^\ell_-) ( \uBRW_{t_k} \foqbrw_{t_k}(\sigma^\ell_-) -1) (1  - \rme^{-(t-t_k)})
				\rmd t \r] \eqsp.
			\end{align}
			By \Cref{lem:1_brw}, we have $\E [-\log \uBRW_{t_k}\foqbrw_{t_k}(\sigma^\ell_+)] = \E [\uBRW_{t_k}\log \uBRW_{t_k} \foqbrw_{t_k}(\sigma^\ell_-)]$ and reversely. Therefore
			\begin{align}
				G_{3.2} &= \sum_{k=0}^{K-1}  \E \l[ \int_{[t_k,t_{k+1})}
				\sum_{\sigma \in \mcs} \log\uBRW_{t_k}\foqbrw_{t_k}(\sigma)  ( \rme^{t-t_k} -1)
				\rmd t \r]  \\
				&+ \sum_{k=0}^{K-1} \E \l[ \int_{[t_k,t_{k+1})} \sum_{\ell =1}^d 
				\l\{\log \uBRW_{t_k} \foqbrw_{t_k}(\sigma^\ell_+) +\log \uBRW_{t_k}(\sigma^\ell_-) \r\} (  \rme^{-(t-t_k)} -1 )
				\rmd t \r]
			\end{align}
			We add and subtract $\log \uBRW_{t_k}\foqbrw_{t_k}(\sigma^\ell_-)$ in the second summand to obtain
			\begin{align}
				&\quad G_{3.2} = \sum_{k=0}^{K-1}  \E \l[ 
				\sum_{\sigma \in \mcs} \log\uBRW_{t_k}\foqbrw_{t_k}(\sigma) \r] \int_{[t_k,t_{k+1})}\underbrace{( \rme^{t-t_k}+ \rme^{-(t-t_k)} -2)}_{\geq 0}
				\rmd t  \\
				&+ \sum_{k=0}^{K-1} \sum_{\ell =1}^d
				\E \l[  \l(\foqbrw_{t_k}(\sigma^\ell_-) - 1\r)\log \uBRW_{t_k}(\sigma^\ell_-) \1_{\foqbrw_{t_k}(\sigma^\ell_-) \geq 1} \r] \int_{[t_k,t_{k+1})} (1- \rme^{-(t-t_k)})
				\rmd t \\
				& \overset{\log x \leq x-1}{\leq} \sum_{k=0}^{K-1} \underbrace{\sum_{\sigma \in \mcs}
					\E \l[ (\uBRW_{t_k}-1) \foqbrw_{t_k}(\sigma) \r]}_{=0 \text{ by \Cref{lem:2_brw}}} \int_{[t_k,t_{k+1})} ( \rme^{-(t-t_k)}+\rme^{t-t_k}-2)
				\rmd t\\
				& + \sum_{k=0}^{K-1} \sum_{\ell =1}^d
				\E \l[ ( \uBRW_{t_k}(\sigma^\ell_-)-1)( \foqbrw_{t_k}(\sigma^\ell_-)-1)\1_{\foqbrw_{t_k}(\sigma^\ell_-) \geq 1} \r] \int_{[t_k,t_{k+1})} (1- \rme^{-(t-t_k)})
				\rmd t\\
				&\overset{\Cref{lem:2_brw}}{\leq} \sum_{k=0}^{K-1} \sum_{\ell =1}^d
				\E \l[  \underbrace{\foqbrw_{t_k}(\sigma^\ell_+)}_{=1} \r] \int_{[t_k,t_{k+1})} (1- \rme^{-(t-t_k)})
				\rmd t \eqsp.
			\end{align}
			To this end, note that $1-\rme^{-(t-t_k)} \leq \rme^{t-t_k}-1 \leq \rme^{t_{k+1}-t_k} -1 \leq \rme^c-1$ for any $t\in [t_k,t_{k+1})$. Thereby we achieve the following bound on $G_{3.2}$:
			\begin{align}
				G_{3.2} &\leq (\rme^c-1)d \sum_{k=0}^{K-1} h_{k+1} = (\rme^c-1) d\Tf \eqsp. \label{eq:e32_brw}
			\end{align}
			The next term $G_{3.3}$ can be reduced to the sum over all $\sigma^\ell_-$ since $\foqbrw_t(\sigma^\ell_+)=1$ for any $t \in [0,\Tf]$ and $\ell \in [d]$. Therefore we derive
			\begin{align}
				G_{3.3} &= \sum_{k=0}^{K-1}\E \l[ \sum_{\ell=1}^d \int_{[t_k,t_{k+1})}
				(\uBRWstar_{t_k}(\sigma^\ell_-) -1 ) \l(  \foqbrw_{t}(\sigma^\ell_-)   -  \foqbrw_{t_k}(\sigma^\ell_-) \r)
				\rmd t\r] \eqsp.
			\end{align}
			By the tower property and~\Cref{lem:evolution_q_brw}, we obtain that
			\begin{align}
				G_{3.3} &= \sum_{k=0}^{K-1}\E \l[ \sum_{\ell=1}^d \int_{[t_k,t_{k+1})}
				(\uBRWstar_{t_k} (\sigma^\ell_-) -1 )\l( \E \l[ \foqbrw_{t}(\sigma^\ell_-) | \Fc_{t_k} \r]  -  \foqbrw_{t_k}(\sigma^\ell_-) \r) \rmd t \r] \notag \\
				&= \sum_{k=0}^{K-1}\E \l[ \sum_{\ell=1}^d \int_{[t_k,t_{k+1})}
				(\uBRWstar_{t_k} (\sigma^\ell_-) -1)  \l(  \foqbrw_{t_k}(\sigma^\ell_-) -1\r)(\rme^{t-t_k}-1) \rmd t \r] \notag \\
				&\leq (\rme^c-1)\sum_{k=0}^{K-1} h_{k+1} \E \l[ \sum_{\ell=1}^d
				\uBRWstar_{t_k}  \foqbrw_{t_k}(\sigma^\ell_-) \r] + (\rme^c-1) d\Tf \eqsp. \notag 
			\end{align}
			We bound the first addend by elementary inequality $ab \leq a^2+b^2$, \Cref{ass:approx_score_brw} and \Cref{lem:2_brw}:
			\begin{align}
				G_{3.3}
				& = (\rme^c-1)\sum_{k=0}^{K-1} h_{k+1} \underbrace{\E \Bigg[ \sum_{\ell=1}^d
					\l(\uBRWstar_{t_k} - \uBRW_{t_k} \r)  \foqbrw_{t_k}(\sigma^\ell_-) \Bigg]}_{\leq \sum_{\ell=1}^d \E\l[ \|(\uBRWstar_{t_k} - \uBRW_{t_k})\foqbrw_{t_k}(\sigma^\ell_-) \|^2 +1\r]} \notag \\
				&\qquad  \quad+ (\rme^c-1)\sum_{k=0}^{K-1} h_{k+1} \E \l[ \sum_{\ell=1}^d
				\uBRW_{t_k} \foqbrw_{t_k}(\sigma^\ell_-) \r] +(\rme^c-1) d\Tf \notag \\
				& \overset{\Cref{lem:2_brw}}{\underset{\Cref{ass:approx_score_brw}}{\leq}}
				(\rme^c-1) \varepsilon^2\Tf +
				(\rme^c-1) d \Tf
				+ (\rme^c-1)\sum_{k=0}^{K-1} h_{k+1} \E \l[ \sum_{\ell=1}^d
				\underbrace{\foqbrw_{t_k}(\sigma^\ell_+)}_{=1} \r] \notag \\
				&\quad \overset{c \leq 1/2}{\lesssim} \varepsilon^2\Tf + (\rme^c-1) d\Tf \eqsp.\label{eq:e33_brw}
			\end{align}
			The last term $G_{3.1}$ can be controlled using the tower property and the monotonicity showed in~\Cref{prop:4_brw}:
			\begin{align}
				G_{3.1}
				& \leq  \sum_{k=0} ^{K-1} h_{k+1} \E \l[ \sum_{\sigma \in \mcs} \l( \foqbrw_{t_{k+1}} \uBRW_{t_{k+1}}\log \uBRW_{t_{k+1}}(\sigma)  - \sum_{\sigma \in \mcs} \foqbrw_{t_k} \uBRW_{t_k}\log \uBRW_{t_k}(\sigma)   \r) \r]\\
				&{=} \sum_{k=0}^{K-1} h_{k+1} \l(\Ic_{\brw}(\mubrw_{\Tf-t_{k+1}}) - \Ic_{\brw}(\mubrw_{\Tf-t_k}) \r) \eqsp.
			\end{align}
			Following precisely the proof of \Cref{theo:main_masked_d} and using the upper bound of $\Ic_{\brw}(\mubrw_{\Tf-t})$ showed in \Cref{prop:bound_fisher_brw}, choosing
			$h_{k+1} = c \min \l\{ \max \l\{ \Tf-t_k, 1/L \r\}, 1 \r\}$ for $k < K-1$ and $h_K = \Tf-t_{K-1}$, with $L = d^{-1}\Ic (\mustar) \geq 2$ yields
			\begin{align}
				G_{3.1} \lesssim {c} [d+\m^\star_2] \log(L) \eqsp.\label{eq:e31_brw}
			\end{align}
			From~\eqref{eq:e32_brw},~\eqref{eq:e33_brw} and~\eqref{eq:e31_brw} and note that $c \leq \rme^c-1$, we obtain the universal bound on $G_3$:
			\begin{align}
				G_3 \lesssim  (\rme^c-1) [d\Tf + (d+\mstar_2)\log (L) ] + \varepsilon^2\Tf \eqsp. \label{eq:e3_brw}
			\end{align}
			Substituting all the upper bounds of $G_1-G_2-G_3$ in~\eqref{eq:e32_brw}, ~\eqref{eq:e33_brw} and \eqref{eq:e31_brw} into~\eqref{eq:kl_path_brw} yields
			\begin{align}
				&\KL(\baPbrw|  \baPbrwstar)\\
				 &\quad \lesssim \rme^{-\Tf} \KL(\mustar|\brw) + \varepsilon^2\Tf
				 + (\rme^c-1) [d\Tf + (d+\mstar_2)\log (L) ]  \eqsp.
			\end{align}
			To this end, notice that $\mustar=\mathrm{Law}(\baXbrw_{\Tf})$, therefore
			\begin{align}
				\KL(\mustar|\mathrm{Law}(\baXbrwstar_{\Tf} ) ) &=\KL(\mathrm{Law}(\baXbrw_{\Tf})|\mathrm{Law}(\baXbrwstar_{\Tf}) ) \\
				&\leq \KL(\mathrm{Law}((\baXbrw_t)_{t\in [0,\Tf]})|\mathrm{Law}((\baXbrwstar_t)_{t\in [0,\Tf]}) )\\
				&= \KL(\baPbrw|\baPbrwstar) \eqsp,
			\end{align}
			where the inequality is known as {Data processing} inequality for relative entropy \cite[Lemma 1.6]{nutz2021introduction}. We then conclude that
			\begin{align}
				&\KL(\mustar|\mathrm{Law}(\baXbrwstar_{\Tf} ) ) \\
				&\quad \lesssim \rme^{-\Tf} \KL(\mustar|\brw) + \varepsilon^2\Tf + (\rme^c-1) [d\Tf + (d+\mstar_2)\log (L) ]  \eqsp.
			\end{align}
			Moreover, we know that $\KL(\mustar|\brw) \leq d+\mstar_2$ (see Appendix \ref{proof_prop:bound_fisher_brw} for completeness), hence
			\begin{align}
				&\KL(\mustar|\mathrm{Law}(\baXbrwstar_{\Tf} ) ) \\
				&\quad \lesssim \rme^{-\Tf} (d+\mstar_2) + \varepsilon^2\Tf + (\rme^c-1) [d\Tf + (d+\mstar_2)\log (L) ]   \eqsp.
			\end{align}
			In particular, choosing $\Tf$ and $c$ as in \eqref{eq:cor_brw_scale} immediately implies the error $\tilde O(\varepsilon^2)$
			with the number of iterations given by
			\begin{align}
				&\quad K = k_0 + k_1 +K-k_0-k_1 {\lesssim} \dfrac{\Tf  - 1}{c} + \dfrac{\log(L)}{c} + \dfrac{1}{c} = \dfrac{\Tf  + \log(L)}{c} \\
				&\lesssim \frac{\log((d+\mstar_2)/\varepsilon^2) + \log (L)}{\log(1+ \varepsilon^2/[d\log((d+\mstar_2)/\varepsilon^2) + (d+\mstar_2)\log(L) ] )}  \eqsp,
			\end{align}
			and the proof of \Cref{theo:scale_brw} is finished.
		\end{proof}

\appendix
\section*{APPENDIX}
  \section{DDMs algorithms}
  \label{sec:ddms-algorithms}
This section provides the pseudo-code for sampling the generative models in the considered cases.
\begin{algorithm}[H]
    \caption{DDM with Random walk on $\Z^d_m$}\label{alg:rw}
    \textbf{Input:} a time horizon $\Tf \gg 1$, a partition $0 = t_0 < t_1 < \cdots < t_K = \Tf$, a trained backward generator $(\baqrwstar_t)_{t\in [0,\Tf)}$
\begin{algorithmic}
\State $\baXrwstar_0 \sim \mathrm{Uniform}(\Z^d_m)$
\For{$k=0$ to $K-1$}
    \State $t \gets t_k$
    \While{$t \in [t_k,t_{k+1})$}
        \State $\lambda_t \gets \sum_{\sigma \in \mcs} \baqrwstar_t(\baXrwstar_t, \sigma(\baXrwstar_t))$
        \State Draw $\tau \sim \mathrm{Exp}(\lambda_t)$
        \If{$t+\tau > t_{k+1}$}
            \State Set $\baXrwstar_{\tilde{t}} \gets \baXrwstar_t$ for all $\tilde{t} \in [t, t_{k+1}]$
            \State $t \gets t_{k+1}$
        \Else
            \State  $\baXrwstar_{\tilde{t}} \gets \baXrwstar_{t}$ for all $\tilde{t} \in [t,t+\tau)$
            \State Draw $\sigma^\star \sim \mathrm{Cate}(\{\baqrwstar_t(\baXrwstar_t, \sigma(\baXrwstar_t))/\lambda_t\}_{\sigma \in \mcs})$
            \State $\baXrwstar_{t+\tau} \gets \sigma^\star(\baXrwstar_{t})$ 
            \State $t \gets t+\tau$
        \EndIf
    \EndWhile
\EndFor
\end{algorithmic}
\textbf{Output:} $\baXrwstar_{\Tf}$
\end{algorithm}

\begin{algorithm}[H]
    \caption{DDMs with Masked diffusion on $\Z^d_m$}\label{alg:masked_d}
    \textbf{Input:} a time horizon $\Tf-\eta \gg 1$, a partition $0 = t_0 < t_1 < \cdots < t_K = \Tf-\eta$, a trained backward generator $(\baqmstar_t)_{t\in [0,\Tf)}$
\begin{algorithmic}
\State $\baXmstar_0 \gets \text{MASK}$
\For{$k=0$ to $K-1$}
    \State $t \gets t_k$
    \While{$t \in [t_k, t_{k+1})$}
        \If{$\msm_{\baXmstar_{t}} \neq \emptyset$}
            \State $\lambda_t \gets \sum_{i \in \msm_{\baXmstar_t}} \sum_{j \in \Z_m} \baqmstar_t(\baXmstar_t, \um^{(i)}_j (\baXmstar_t))$
            \State Draw $t \sim \mathrm{Exp}(\lambda_t)$
            \If{$t+\tau > t_{k+1}$}
                \State Set $\baXmstar_{\tilde{t}} \gets \baXmstar_t$ for all $\tilde{t} \in [t, t_{k+1}]$
                \State $t \gets t_{k+1}$
            \Else
                \State  $\baXmstar_{\tilde{t}} \gets \baXmstar_{t}$ for all $\tilde{t} \in [t,t+\tau)$
                \State Draw $(i^\star,j^\star) \sim \mathrm{Cate}(\{{\baqmstar_{t}(\baXmstar_{t},\um^{(i)}_j(\baXmstar_{t}))}/\lambda_t \}_{(i,j) \in \msm_{\baXmstar_{t}} \times \Z_m})$
                \State $\baXmstar_{t+\tau} \gets \um^{(i^\star)}_{j^\star} (\baXmstar_{t})$ 
                \State $t \gets t+\tau$
            \EndIf
        \EndIf
    \EndWhile
\EndFor
\end{algorithmic}
\textbf{Output:} $\baXmstar_{\Tf-\eta}$
\end{algorithm}

\begin{algorithm}[H]
    \caption{DDMs with Biased random walk on $\N^d$}\label{alg:brw}
    \textbf{Input:} a time horizon $\Tf \gg 1$, a partition $0 = t_0 < t_1 < \cdots < t_K = \Tf$, a trained backward generator $(\baqbrwstar_t)_{t\in [0,\Tf)}$
\begin{algorithmic}
\State $\baXbrwstar_0 \sim \mathrm{Poisson}(1)^{\otimes d}$
\For{$k=0$ to $K-1$}
    \State $t \gets t_k$
    \While{$t \in [t_k,t_{k+1})$}
        \State $\lambda_t \gets \sum_{\sigma \in \mcs} \baqbrwstar_t(\baXbrwstar_t, \sigma(\baXbrwstar_t))$
        \State Draw $\tau \sim \mathrm{Exp}(\lambda_t)$
        \If{$t+\tau > t_{k+1}$}
            \State Set $\baXbrwstar_{\tilde{t}} \gets \baXbrwstar_t$ for all $\tilde{t} \in [t, t_{k+1}]$
            \State $t \gets t_{k+1}$
        \Else
            \State  $\baXbrwstar_{\tilde{t}} \gets \baXbrwstar_{t}$ for all $\tilde{t} \in [t,t+\tau)$
            \State Draw $\sigma^\star \sim \mathrm{Cate}(\{\baqbrwstar_t(\baXbrwstar_t, \sigma(\baXbrwstar_t))/\lambda_t\}_{\sigma \in \mcs})$
            \State $\baXbrwstar_{t+\tau} \gets \sigma^\star(\baXbrwstar_{t})$ 
            \State $t \gets t+\tau$
        \EndIf
    \EndWhile
\EndFor
\end{algorithmic}
\textbf{Output:} $\baXbrwstar_{\Tf}$
\end{algorithm}

    \section{Time-reversal of CTMC}
    \label{app:application_reversal}
    We provide in this section a sketch of existence proof of the time-reversal of CTMCs for completeness.
    \begin{proposition}\label{prop:time_reversal_ctmc}
      Let $(q_t)_{t\in [0,\Tf]}$ be an irreducible and non-explosive generator matrix. The time-reversal $(\baX_t)_{t\in [0,\Tf]}$ of a CTMC $(X_t)_{t\in [0,\Tf]}$ associated with $(q_t)_{t\in [0,\Tf]}$ is defined as $\baX_t \eqdef X_{\Tf-t}$. Then  $(\overleftarrow{X}_t)_{t\in\ccint{0,\Tf}}$ is also an {inhomogeneous} CTMC, associated with a family of generator matrices $(\overleftarrow{q}_t)_{t \in\ccint{0,\Tf}}$ satisfying the time-reversal formula: for any $0 \leq t \leq \Tf$ and  $x \neq y \in \msx$,
\begin{equation}
    \muX_{\Tf-t}(x)\overleftarrow{q}_t(x,y)=\muX_{\Tf-t}(y)q_{\Tf -t}(y,x) \eqsp,
\end{equation}
where $\muX_t(x) = \P(X_t = x)$.
\end{proposition}
\begin{proof}[Proof of~\Cref{prop:time_reversal_ctmc}]
    We first show the Markov property of the time-reversal process: for any $0\leq t_0<t_1\cdots <t_n \leq \Tf$ and $x_0,x_1,\dots,x_n \in \msx$, almost surely it holds
    \begin{align}
        &\quad\P(\baX_{t_n} = x_n | \baX_{t_0} = x_0, \dots, \baX_{t_{n-1}}=x_{n-1})\\
        &=  \P(X_{\Tf-t_n} = x_n | X_{\Tf-t_0}=x_0, \dots, X_{\Tf - t_{n-1}}=x_{n-1})\\
        &= \frac{\P(X_{\Tf-t_n} = x_n, X_{\Tf-t_0}=x_0, \dots, X_{\Tf - t_{n-1}}=x_{n-1})}{\P( X_{\Tf-t_0}=x_0, \dots, X_{\Tf - t_{n-1}}=x_{n-1})} \\
        &= \frac{\P(X_{\Tf-t_n}=x_n)\prod_{i=0}^{n-1}\P(X_{\Tf-t_i}=x_i|X_{\Tf-t_{i+1}}=x_{i+1})}{\P(X_{\Tf-t_{n-1}}=x_{n-1})\prod_{i=0}^{n-2}\P(X_{\Tf-t_i}=x_i|X_{\Tf-t_{i+1}}=x_{i+1})} \eqsp,
    \end{align}
    where we used the Markov property of $(X_t)_{t\in [0,\Tf]}$ to factorize as in the last equality. Simplifying the expression above yields
    \begin{align}
        &\quad\P(\baX_{t_n} = x_n | \baX_{t_0} = x_0, \dots, \baX_{t_{n-1}}=x_{n-1})\\
        &= \frac{\P(X_{\Tf-t_n}=x_n)\P(X_{\Tf-t_{n-1}}=x_{n-1}|X_{\Tf-t_{n}}=x_{n})}{\P(X_{\Tf-t_{n-1}}=x_{n-1})}\\
        &= \frac{\P(X_{\Tf-t_{n-1}}=x_{n-1},X_{\Tf-t_{n}}=x_{n})}{\P(X_{\Tf-t_{n-1}}=x_{n-1})} \\
        &= \P(X_{\Tf-t_{n}}=x_{n}|X_{\Tf-t_{n-1}}=x_{n-1}) \\
        &= \P(\baX_{t_n} = x_n | \baX_{t_{n-1}}=x_{n-1}) \eqsp,
    \end{align}
    meaning that the time-reversal $(\baX_t)_{t\in [0,\Tf]}$ satisfies the Markov property, \ie, it is indeed a CTMC.
    Moreover, the backward generator can be deduced by noting that: for $t \in [0,\Tf]$, $x \neq y \in \msx$ and $h>0$,
    \begin{align}
        \P(\baX_{t+h}=y,\baX_{t}=x) = \P(X_{\Tf-(t+h)}=y , X_{\Tf -t} =x) \eqsp.
    \end{align}
    By the Bayes' formula, we then have
    \begin{align}
        {\muX_{\Tf-t}(x)}\P(\baX_{t+h}=y|\baX_{t}=x) = \P(X_{\Tf -t} =x|X_{\Tf-(t+h)}=y) {\muX_{\Tf-(t+h)}(y)} \eqsp.
    \end{align}
    This together with the Kolmogorov equation imply, in particular, the following relation
    \begin{align}
       {\muX_{\Tf-t}(x)}[ h\baq_t(x,y)+o(h)] = {\muX_{\Tf-(t+h)}(y)} [hq_{\Tf-t}(y,x)  + o(h)] \quad \text{for $h \to 0^+$} \eqsp,
    \end{align}
    where $o$ is the standard little-o Landau notation. Dividing both hand sides by $h$ and letting $h \to 0^+$ give desired equation, which concludes our proof.
\end{proof}

\section{Further derivation of considered models}

\subsection{Random walk on $\Z^d_m$}\label{sec:further_rw}

One useful property of the discrete score is that it can be presented as the conditional expectation, which enables efficient training in practice via an $\mrl^2$-loss.

\begin{proposition}\label{prop:condition_expectation_rw}
      For any $\sigma \in \mcs$ and $(t,x) \in [0,\Tf) \times \Z^d_m$, it holds:
      \[
      \uRW_t(x, \sigma(x)) = \E \l[ \frac{\foprw_{0,\Tf-t}(\foXrw_0,\sigma(x))}{\foprw_{0,\Tf-t}(\foXrw_0,x)} \middle| \foXrw_{\Tf-t} = x  \r] \eqsp.
      \]
\end{proposition}
\begin{proof}[Proof of \Cref{prop:condition_expectation_rw}]
  For any $x \in \Z^d_m$, $t \in [0,\Tf)$ and $\sigma \in \mcs$, we have
\begin{align}
    \uRW_t(x, \sigma(x)) &= \frac{\murw_{\Tf-t}(\sigma(x))}{\murw_{\Tf-t}(x)} = \sum_{x_0 \in \Z^d_m} \frac{\foprw_{0,\Tf-t}(x_0,\sigma(x))}{\murw_{\Tf-t}(x)}\murw_{0}(x_0)\\
    &=\sum_{x_0 \in \Z^d_m} \frac{\foprw_{0,\Tf-t}(x_0,\sigma(x))}{\foprw_{0,\Tf-t}(x_0,x)}\P(\foXrw_0 = x_0 | \foXrw_{\Tf-t} = x) \\
    &= \E \l[ \frac{\foprw_{0,\Tf-t}(\foXrw_0,\sigma(x))}{\foprw_{0,\Tf-t}(\foXrw_0,x)} \middle| \foXrw_{\Tf-t} = x  \r] \eqsp.
\end{align}
\end{proof}

\begin{lemma}\label{lem:martingale_rw}
   Assume \textbf{RW2} holds. For any fixed $\sigma \in \mcs$, the following process is a true martingale on $[0,\Tf]$:
    \[
   \Mrw_\sigma(t) = \int_{[0,t] \times \Z^d_m} \l[ \uRW_s(x,\sigma(x)) - \uRW_s( \baXrw_{s-}, \sigma(\baXrw_{s-})) \r] \tilde{N}_{\canoX}^{\baqrw} (\rmd x \rmd s) \eqsp.
   \]
\end{lemma}
\begin{proof}[Proof of \Cref{lem:martingale_rw}]
  Since the compensated measure $\tilde{N}_{\canoX}^{\uRW \qrw}$ is a martingale, the stochastic integral $\Mrw_\sigma(t)$ is a local martingale. Hence it suffices to show the integrability of $\Mrw_\sigma(t)$ to conclude the proof. Note that
  \[\tilde{N}^{\baqrw}_{\baXrw} = {N}^{\baqrw}_{\baXrw} - \bar \rmn^{\baqrw}_{\baXrw} \eqsp,\]
where $\bar \rmn^{\baqrw}_{\baXrw}$ is the compensator of ${N}^{\baqrw}_{\baXrw}$, therefore for fixed $\sigma \in \mcs$ and for $t \in [0,\Tf]$,
\begin{align}
    \E[|\Mrw_\sigma(t)|] &\leq \E \l[ \int_{[0,t] \times \Z^d_m} | \uRW_s(x,\sigma(x)) - \uRW_s( \baXrw_{s-}, \sigma(\baXrw_{s-})) | {N}^{\baqrw}_{\baXrw} (\rmd s \rmd x) \r]\\
    & + \E \l[ \int_{[0,t] \times \Z^d_m} | \uRW_s(x,\sigma(x)) - \uRW_s( \baXrw_{s-}, \sigma(\baXrw_{s-})) | \bar \rmn^{\baqrw}_{\baXrw} (\rmd s \rmd x) \r] \\
    &= \E \Big[ \int_{[0,t] } \sum_{\sigma' \in \mcs} | \uRW_s(\sigma'(\baXrw_{s-}),\sigma(\sigma'(\baXrw_{s-}))) - \uRW_s( \baXrw_{s-}, \sigma(\baXrw_{s-})) |\\
    &\qquad \qquad \uRW_s  (\baXrw_{s-},\sigma'(\baXrw_{s-})) \rmd s \Big] \\
    &\leq 2\sup_{(s,x) \in [0,t] \times \Z^d_m} \uRW_s(x,\sigma(x)) \E \l[ \int_{[0,t]} \sum_{\sigma' \in \mcs} \uRW_s (\baXrw_{s-}, \sigma'(\baXrw_{s-})) \rmd s \r] \eqsp.
\end{align}
Since the state space $\Z^d_m$ is finite, it suffices to show $\sup_{t \in [0,\Tf]} u_t(x,\sigma(x)) <\infty$ for each $x \in \Z^d_m$ to obtain the integrability of $\Mrw_\sigma(t)$. Recall the formula of the discrete score for fixed $x \in \Z^d_m$:
\[ \uRW_t(x,\sigma(x)) = \frac{\murw_{\Tf-t}(\sigma(x))}{\murw_{\Tf-t}(x)} \eqsp, \quad \text{with } \murw_t(x) = \sum_{z\in \Z^d_m}\mustar(z) \foprw_{0,t}(z,x) \eqsp,
\]
where $\foprw_{0,t}$ satisfies the forward Kolmogorov equation,  implying $\foprw_{0,t}=\rme^{t\foqrw}$ thus it is continuous in $t$. As a result,
$t \mapsto \murw_t(x)$ is continuous on $[0,\Tf]$. This combined with the fact $\murw_t(x)>0$ for any $t\in [0,\Tf]$ (by \textbf{RW2} and irreducibiliy of $(\foXrw_t)_{t\in [0,\Tf]}$) implies $t \mapsto \uRW_t(x,\sigma(x))$ is continuous on $[0,\Tf]$, thus $\sup_{t\in [0,\Tf]} u_t(x,\sigma(x)) <\infty$. Therefore $\Mrw_\sigma(t)$ is integrable and indeed a true martingale, which completes the proof of \Cref{lem:martingale_rw}.
\end{proof}

\subsection{Masked diffusion on \texorpdfstring{$\Z^d_m$}{Zdm}}\label{sec:further_mask}
Similarly as the random walk case, the discrete score can be viewed as a conditional expectation in the following proposition. 

\begin{proposition}\label{prop:condition_expectation_masked_d}
    Under \textbf{M1} and \textbf{M2}, for any $x \in \tZ^d_m$, $t \in [0,\Tf)$, $j \in \Z_m$ and $i \in \msm_x$, we have
\begin{align}
    \uM_t(x, \um^{(i)}_j(x))
    &= \E \l[ \frac{\fopm_{0,\Tf-t}(\foXm_0,\um^{(i)}_j(x))}{\fopm_{0,\Tf-t}(\foXm_0,x)} \middle| \foXm_{\Tf-t} = x  \r] \eqsp.
\end{align}
\end{proposition}
\begin{proof}[Proof of \Cref{prop:condition_expectation_masked_d}]
Using the formula of the discrete score showed in Lemma 5.1.1 of the main paper, for any $x \in \tZ^d_m$, $t \in [0,\Tf)$, $j \in \Z_m$ and $i \in \msm_x$, we have
  \begin{align}
    \uM_t(x, \um^{(i)}_j(x)) &= \frac{\mum_{\Tf-t}(\um^{(i)}_j(x))}{\mum_{\Tf-t}(x)} = \sum_{x_0 \in \Z^d_m} \frac{\fopm_{0,\Tf-t}(x_0,\um^{(i)}_j(x))}{\mum_{\Tf-t}(x)}\mum_{0}(x_0)\\
    &=\sum_{x_0 \in \Z^d_m} \frac{\fopm_{0,\Tf-t}(x_0,\um^{(i)}_j(x))}{\fopm_{0,\Tf-t}(x_0,x)}\P(\foXm_0 = x_0 | \foXm_{\Tf-t} = x) \\
    &= \E \l[ \frac{\fopm_{0,\Tf-t}(\foXm_0,\um^{(i)}_j(x))}{\fopm_{0,\Tf-t}(\foXm_0,x)} \middle| \foXm_{\Tf-t} = x  \r] \eqsp,
\end{align}
and we complete the proof.
\end{proof}

\begin{lemma}\label{lem:martingale_mask}
   Under \textbf{M1} and \textbf{M2}, for fixed $\eta \in (0,\Tf)$, $i \in [d]$ and $j \in \Z_m$,
   \begin{align}
    \Mm_{(i),j}(t) = \int_{[0,t] \times \tZ^d_m} \l[ f^{(i),j}(s,x) - f^{(i),j}(s, \baXm_{s-}) \r] \tilde{N}_{\baXm}^{\uM \tqm} (\rmd x \rmd s)
\end{align}
is a true martingale on $[0,\Tf-\eta]$, where
\[
       f^{(i),j}(t,\baXm_{t}):= \uM_t(\baXm_t, \um^{(i)}_j(\baXm_t))\1_{i \in \msm_{\baXm_t}}= \rme^{\Vm_t(\baXm_t)-\Vm_t(\um^{(i)}_j(\baXm_t))}\1_{i \in \msm_{\baXm_t}} \eqsp.
\]
\end{lemma}
\begin{proof}[Proof of \Cref{lem:martingale_mask}]
  We begin by noting that $\Mm_{(i),j}(t)$ is a local martingale since the compensated measure $\tilde{N}_{\baXm}^{\uM \tqm}$ is a martingale.
  Hence it suffices to show the integrability of $\Mm_{(i),j}(t)$ to conclude the proof. Note that
  \[\tilde{N}^{\uM \foqm}_{\baXm} = {N}^{\uM \foqm}_{\baXm} - \bar \rmn^{\uM \foqm}_{\baXm} \eqsp,\]
where $\bar \rmn^{\uM \foqm}_{\baXm}$ is the compensator of ${N}^{\uM \foqm}_{\baXm}$, therefore for $t \in [0,\Tf-\eta]$,
\begin{align}
    \E[|\Mm_{(i),j}(t)|] &\leq 2 \E \Big[ \int_{[0,t] } \sum_{\ell \in \msm_{\baXm_{s-}}} \sum_{n \in \Z_m} | f^{(i),j}(s,\um^{(\ell)}_n(\baXm_{s-}))) - f^{(i),j}(s, \baXm_{s-}) |\\
    &\qquad \qquad f^{(\ell),n}(s, \baXm_{s-})\beta(\Tf-s) \rmd s \Big]\\
    &\overset{\beta(\Tf-s)\leq 1}{\leq} 4\sup_{(s,x) \in [0,t] \times \tZ^d_m} f^{(i),j}(s, x) \E \l[ \int_{[0,t]} \sum_{\ell \in \msm_{\baXm_{s-}}} \sum_{n \in \Z_m}  f^{(\ell),n}(s, \baXm_{s-}) \rmd s \r] \eqsp.
\end{align}
To acquire $\E[|\Mm_{(i),j}(t)|] <\infty$, it suffices to show $\sup_{(t,x) \in [0,\Tf-\eta] \times \tZ^d_m} f^{(i),j}(t, x)<\infty$.
For $t \in [0,\Tf-\eta]$ and $x \in \tZ^d_m$, recall that
\begin{align}
    f^{(i),j}(t, x) = \uM_t(x,\um^{(i)}_j(x))\1_{i \in \msm_x} =\1_{i \in \msm_x} \frac{\mum_{\Tf-t}(\um^{(i)}_j(x))}{\mum_{\Tf-t}(x)} \eqsp,
\end{align}
with
\begin{align}
    \mum_t(x) = \sum_{z\in \Z^d_m}\mustar(z) \fopm_{0,t}(z,x) = \mustar(x) \alpha_t^d \1_{x\in \Z^d_m} + (1-\alpha_t)^{|\msm_x|} \alpha_t^{d-|\msm_x|} \1_{x\in \tZ^d_m \setminus \Z^d_m} \eqsp.
\end{align}
Since $\alpha_t = \exp(-\int_0^t \beta(s) \rmd s)$ with $\beta(t)$ continuous and bounded, we can deduce that $t \mapsto \mum_t(x)$ is continuous on $[0,\Tf]$. This together with the fact $\mum_{\Tf-t}(x)>0$ for any $t \in [0,\Tf-\eta]$ (see detailed argument in the proof of Lemma 5.1.1 of the main paper) implies in turn that $t \mapsto f^{(i),j}(t, x)$ is continuous on $[0,\Tf-\eta]$, therefore $\sup_{t \in [0,\Tf-\eta]} f^{(i),j}(t, x) < \infty$ for any $x \in \tZ^d_m$. Taking the maximum of $x$ over the finite set $\tZ^d_m$ yields
\[
\sup_{(t,x) \in [0,\Tf-\eta] \times \tZ^d_m} f^{(i),j}(t, x)<\infty \eqsp,
\]
hence $\Mm_{(i),j}(t)$ is integrable and indeed a true martingale, which concludes the proof of \Cref{lem:martingale_mask}.
\end{proof}


\subsection{Biased random walk on \texorpdfstring{$\N^d$}{Nd}}\label{sec:further_brw}

We now present a detailed argument showing that the biased random walk on $\N^d$ is non-explosive.

\begin{lemma}\label{cond:non_explosive_zhang}
  Follow \cite[Theorem 3.3]{zhang2018nonexplosion}, a stable conservative generator $(q_t)_{t\geq 0}$ on $\msx$ is non-explosive if there exists a monotone nondecreasing sequence $(S_n) \subset \mB(\msx)$ and a $[0,\infty)$-valued
 measurable function $V$ on $\msx$ such that the following conditions hold:
\begin{enumerate}[wide,label=(\alph*)]
    \item As $n \to \infty$, $S_n \to \msx$.
    \item For each $n =0,1,\dots$, $\sup_{x\in S_n, t \geq 0} q_t(x) < \infty$.
    \item As $n \to \infty$, $\inf_{x\in \msx \setminus S_n} V(x) \to \infty$.
    \item There exists a constant $\alpha>0$ s.t. for any $x \in \msx$ and $t \geq 0$,
    \begin{align}
        \int_{\msx} V(y) q_t (x,\rmd y) \leq \alpha V(x) \eqsp.
    \end{align}
\end{enumerate}
\end{lemma}

\begin{proposition}\label{prop:non_explosive_brw}
  The forward generator $(\foqbrw_t)_{t\in [0,\Tf]}$ in eq. (17) of the main paper associated with the biased random walk on $\N^d$ is non-explosive.
\end{proposition}
\begin{proof}[Proof of~\Cref{prop:non_explosive_brw}]
    The proof relies on~\Cref{cond:non_explosive_zhang}. For $n \in \N$, we set $S_n \eqdef \{x \in \msx \, : \, \sum_{i=1}^d x^i \leq n \}$ and consider the function $V(x) = \sum_{i=1}^d x^i$ for any $x \in \msx$. Then it is clear that the condition (a) holds. Moreover, for each $n=0,1,\dots$, we have
    \begin{align}
        \sup_{x\in S_n, t \geq 0} \foqbrw_t(x) = \sup_{x\in S_n} (\sum_{i=1}^d x^i +d) \leq n+d <\infty \eqsp,
    \end{align}
    \ie, (b) holds as well. In addition, we observe that
    \begin{align}
        \inf_{x\in \msx \setminus S_n} V(x) = \inf_{x\in \msx \setminus S_n} (\sum_{i=1}^d x^i +d) > n+d \eqsp.
    \end{align}
    As $n \to \infty$, we indeed obtain $\inf_{x\in \msx \setminus S_n} V(x) \to \infty$ and the condition (c) is satisfied. Finally, for any $x \in \msx$ and $t \geq 0$, we have
    \begin{align}
        \sum_{y\in \msx} V(y) \foqbrw_t (x,y) &=
        \sum_{y \neq x} \foqbrw (x,y) [V(y)-V(x)]\\
        &= \sum_{i=1}^d \{ [V(x+e_i)-V(x)] + x^i[V(x-e_i)-V(x)] \}\\
        &= d-\sum_{i=1}^d x^i \leq d+\sum_{i=1}^d x^i = V(x) \eqsp,
    \end{align}
    which implies the condition (d) in~\Cref{cond:non_explosive_zhang}. We then conclude that $(\foqbrw_t)_{t\in [0,\Tf]}$ is indeed non-explosive by using~\Cref{cond:non_explosive_zhang}, therefore satisfies \textbf{H2} and the forward process $(\foXbrw_t)_{t\in [0,\Tf]}$ is then well-defined.
\end{proof}

Similarly as before, the discrete score in this setting can also be expressed as a conditional expectation, enabling the stable $\mrl^2$-loss in training.

\begin{proposition}\label{prop:condition_expectation_brw}
    For any $x \in \N^d_m$, $t \in [0,\Tf)$ and $\sigma \in \mcs$, we have
\begin{align}
    \uBRW_t \foqbrw(x, \sigma(x))
    &= \E \l[  \foqbrw(\sigma(x),x)\frac{\fopbrw_{0,\Tf-t}(\foXbrw_0,\sigma(x))}{\fopbrw_{0,\Tf-t}(\foXbrw_0,x)} \middle| \foXbrw_{\Tf-t} = x  \r] \eqsp.
\end{align}
\end{proposition}
\begin{proof}[Proof of \Cref{prop:condition_expectation_brw}]
Using the formula of the discrete score established in Lemma 5.2.1 of the main paper, for any $t \in [0,\Tf)$, $x\in \N^d$ and $\sigma \in \mcs$, we have
  \begin{align}
    \uBRW_t  \foqbrw(x, \sigma(x)) &=  \foqbrw (x, \sigma(x)) \frac{\tmubrw_{\Tf-t}(\sigma(x))}{\tmubrw_{\Tf-t}(x)}\\
    &= \foqbrw(x,\sigma(x))\frac{\brw(x)}{\brw(\sigma(x))} \sum_{x_0 \in \N^d}
    \frac{\fopbrw_{0,\Tf-t}(x_0,\sigma(x))}{\mubrw_{\Tf-t}(x)}\mubrw_{0}(x_0)\\
    &{=}  \foqbrw(\sigma(x),x) \sum_{x_0 \in \N^d}  \frac{\fopbrw_{0,\Tf-t}(x_0,\sigma(x))}{\fopbrw_{0,\Tf-t}(x_0,x)}\P(\foXbrw_0 = x_0 | \foXbrw_{\Tf-t} = x) \\
    &= \E \l[  \foqbrw(\sigma(x),x)\frac{\fopbrw_{0,\Tf-t}(\foXbrw_0,\sigma(x))}{\fopbrw_{0,\Tf-t}(\foXbrw_0,x)} \middle| \foXbrw_{\Tf-t} = x  \r] \eqsp,
\end{align}
and the proof concludes.
\end{proof}


We now establish a connection between the moments of the data distribution and those of the marginal densities of $(\foXbrw_t)_{t \in [0,\Tf]}$, and then leverage this relation to derive integrability properties.

\begin{proposition}\label{prop:finite_first_moment_brw}
   Assume $\mstar_1<\infty$, then
   \[
   \m_1(\mubrw_t) = \rme^{-t} \mstar_1 + d(1-\rme^{-t}) <\infty \quad \text{for any $t \geq 0$} \eqsp.
   \]
\end{proposition}
\begin{proof}[Proof of \Cref{prop:finite_first_moment_brw}]
    For $m \in \N^*$, define the stopping time
    \begin{align}
        \tau_m \eqdef \inf \l\{ s \geq 0 \, : \,  \|\foXbrw_s\|_1 \geq m \r\} \eqsp,
    \end{align}
    with the convention $\inf \emptyset = +\infty$. Note that $\E [\|\foXbrw_0\|_1]<\infty$ implies $\|\foXbrw_0\|_1<\infty$ a.s., which combines with the fact $(\foXbrw_t)_{t\geq 0}$ is a non-explosive CTMC and each jump changes $\|\foXbrw_t\|_1$ at most $1$ in turn yield
    \begin{align}
        \tau_m \uparrow \infty \text{ a.s. as $m \to \infty$} \eqsp.
    \end{align}
    Let us consider the stopped process $(\foXbrw_{t \wedge \tau_m})_{t\geq0}$. On the finite time interval $[0, t \wedge \tau_m)$, the process stays in the finite set $\l\{ x \, : \, \|x\|_1 \leq m-1 \r\}$, thus the integrability condition holds and we can apply Dynkin's formula to $\|\cdot\|_1$ for the stopped dynamic:
    \begin{align}
        \E \l[ \|\foXbrw_{t \wedge \tau_m}\|_1 \r] = \E \l[ \|\foXbrw_{0}\|_1 \r] + \E \l[ \int_0^{t \wedge \tau_m} (\foqbrw \|\cdot \|_1)(\foXbrw_s) \rmd s \r] \eqsp.
    \end{align}
    Plugging the formula of $\foqbrw$ in eq. (17) of the main paper and differentiating the integral expressions above yield
    \begin{align}
        \frac{\rmd }{\rmd t} \E \l[ \|\foXbrw_{t \wedge \tau_m}\|_1 \r] = d - \E \l[ \|\foXbrw_{t \wedge \tau_m}\|_1 \r] \eqsp.
    \end{align}
    As a result, for fixed $m \in \N^*$, the function $\Mbf_m (t) \eqdef \E \l[  \|\foXbrw_{t \wedge \tau_m}\|_1\r]$ satisfies the following linear ODE
    \begin{align}
        \Mbf'_m(t) = d - \Mbf_m(t) \quad \text{with $\Mbf_m(0)=\E \l[\|\foXbrw_0\|_1 \r]$} \eqsp.
    \end{align}
    Solving this equation gives
    \begin{align}
        \Mbf_m(t) = \rme^{-t} \E \l[\|\foXbrw_0\|_1 \r] + d(1-\rme^{-t}) \eqsp.
    \end{align}
    Letting $m \to \infty$ and noting that $\|\foXbrw_{t \wedge \tau_m}\|_1 \uparrow \|\foXbrw_t\|_1$ a.s., by monotone convergence, we arrive at
    \begin{align}
       \m_1(\mubrw_t) &= \E \l[ \|\foXbrw_t\|_1 \r] = \lim_{m\to \infty} \E \l[  \|\foXbrw_{t \wedge \tau_m}\|_1\r] \\
       &= \lim_{m\to\infty} \l( \rme^{-t} \E \l[\|\foXbrw_0\|_1 \r] + d(1-\rme^{-t}) \r)\\
        &= \rme^{-t} \mstar_1 + d(1-\rme^{-t})\eqsp.
    \end{align}
    Hence $\m_1(\mubrw_t) = \E \l[ \|\foXbrw_t\|_1 \r]$ is finite for any $t \geq 0$, given $\mstar_1$ is finite and admits an explicit expression above.
\end{proof}

\begin{proposition}\label{prop:second_moment_brw}
      Assume $\m^\star_2<\infty$, then for any $t \geq 0$,
      \[
      \m_2(\mubrw_t) = \rme^{-2t} [\m^\star_2 - 3\mstar_1 +d ] + 3\rme^{-t} [\mstar_1-d] +2d <\infty  \eqsp.
      \]
\end{proposition}
\begin{proof}[Proof of \Cref{prop:second_moment_brw}]
    We will employ the same trick as \Cref{prop:finite_first_moment_brw}. For $m \in \N^*$, define the stopping time
    \begin{align}
        \tau_m \eqdef \inf \l\{ s \geq 0 \, : \,  \|\foXbrw_s\|_2^2 \geq m \r\} \eqsp,
    \end{align}
    with the convention $\inf \emptyset = +\infty$. Note that $\E [\|\foXbrw_0\|_2^2]<\infty$ implies $\|\foXbrw_0\|_2^2<\infty$ a.s., which combines with the fact $(\foXbrw_t)_{t\geq 0}$ is a non-explosive CTMC and each jump changes $\|\foXbrw_t\|_2^2$ at most $1$ in turn yield
    \begin{align}
        \tau_m \uparrow \infty \text{ a.s. as $m \to \infty$} \eqsp.
    \end{align}
    Let us consider the stopped process $(\foXbrw_{t \wedge \tau_m})_{t\geq 0}$. On the finite time interval $[0, t \wedge \tau_m)$, the process stays in the finite set $\l\{ x \, : \, \|x\|_2^2 \leq m-1 \r\}$, thus the integrability condition holds and we can apply Dynkin's formula to $\|\cdot\|_2^2$ for the stopped dynamic:
    \begin{align}
        \E \l[ \|\foXbrw_{t \wedge \tau_m}\|_2^2 \r] = \E \l[ \|\foXbrw_{0}\|_2^2 \r] + \E \l[ \int_0^{t \wedge \tau_m} (\foqbrw \|\cdot \|_2^2)(\foXbrw_s) \rmd s \r] \eqsp.
    \end{align}
    Plugging the formula of $\foqbrw$ in eq. (17) of the main paper and differentiating the integral expressions above yield
    \begin{align}
        \frac{\rmd }{\rmd t} \E \l[ \|\foXbrw_{t \wedge \tau_m}\|_2^2 \r] = d - 2\E \l[ \|\foXbrw_{t \wedge \tau_m}\|_2^2 \r] + 3 \E \l[ \|\foXbrw_{t \wedge \tau_m}\|_1 \r]\eqsp.
    \end{align}
    Note that $\E \l[ \|\foXbrw_{t \wedge \tau_m}\|_1 \r] \leq \E \l[  \|\foXbrw_{t}\|_1 \r] = \m_1(\mubrw_t)$ for any $t \geq 0$. In addition, the finiteness of the second moment of $\mubrw_0$ implies that its first moment is also finite. Combined with \Cref{prop:finite_first_moment_brw}, this yields the finiteness of $\m_1(\mubrw_t)$ for any $t \geq 0$.
   Therefore, we can solve the ODE above to get
    \begin{align}
        \E \l[ \|\foXbrw_{t \wedge \tau_m}\|_2^2 \r] = \rme^{-2t} \E \l[\|\foXbrw_0\|_2^2 \r] + \rme^{-2t} \int_0^{t} \rme^{2s} (3 \E \l[ \|\foXbrw_{t \wedge \tau_m}\|_1 \r]+d) \rmd s \eqsp.
    \end{align}
    Letting $m \to \infty$ and noting that $\|\foXbrw_{t \wedge \tau_m}\|_2^2 \uparrow \|\foXbrw_t\|_2^2$ a.s., by monotone convergence, we achieve
    \begin{align}
       \m_2(\mubrw_t) &= \E \l[ \|\foXbrw_t\|_2^2 \r] = \lim_{m\to \infty} \E \l[  \|\foXbrw_{t \wedge \tau_m}\|_2^2\r] \\
        &= \rme^{-2t} \m^\star_2 + \rme^{-2t} \int_0^{t} \rme^{2s} (3 \m_1(\mubrw_s)+d) \rmd s <\infty \eqsp,
    \end{align}
    for any $t \geq 0$, given $\m^\star_2<\infty$.
    Furthermore, substituting the formula of $\m_1(\mubrw_s)$ in \Cref{prop:finite_first_moment_brw} leads to a closed form of $\m_2(\mubrw_t)$:
    \begin{align}
        \m_2(\mubrw_t) &=  \rme^{-2t} \m^\star_2 + \rme^{-2t} \int_0^{t} \rme^{2s} (3\rme^{-s} \mstar_1 +3d (1-\rme^{-s})+d) \rmd s\\
        &= \rme^{-2t} [\m^\star_2 - 3\mstar_1 +d ] + 3\rme^{-t} [\mstar_1-d] +2d \eqsp.
    \end{align}
    The proof is then finished.
\end{proof}

Our next step is to determine the explicit transition probabilities of $(\foXbrw_t)_{t\in [0,\Tf]}$, enabling us to verify the integrability condition stated in Lemma 5.2.5 in the main paper.
Given the componentwise factorization of the process, the task reduces to deriving the transition density in the one-dimensional case.

\begin{proposition}\label{prop:transition_formula_brw}
 The transition density of $(\foXbrw_t)_{t\in [0,\Tf]}$ admits the following formula:
     \begin{align}
      \fopbrw_{0,t} (x,y) = \prod_{i=1}^d \fopbrwb_{0,t} (x^i,y^i) \eqsp, \quad \text{for $x, y \in \N^d$ and $t \in [0,\Tf]$} \eqsp,
  \end{align}
   where
      \begin{align}\label{eq:transition_formula_brw}
          \fopbrwb_{0,t} (k,n)
      &= \rme^{\rme^{-t}-1} \sum_{j=0}^{\min(k,n)} \binom{k}{j} \rme^{-tj} \frac{(1-\rme^{-t})^{k+n-2j}}{(n-j)!} \eqsp, \quad \text{for $k,n \in \N$} \eqsp.
      \end{align}
\end{proposition}

\begin{proof}[Proof of \Cref{prop:transition_formula_brw}]
  The proof amounts to checking that \eqref{eq:transition_formula_brw} indeed solves the forward Kolmogorov equation, namely, that the following relation is satisfied:
   \begin{align}\label{eq:kolm_forward_brw}
       \fopbrwb_{0,0} = \Id \quad \text{and} \quad \partial_t \fopbrwb_{0,t}(k,n) = \sum_{i\in \N} \fopbrwb_{0,t} (k,i)\foqbrwb (i,n) \eqsp,
     \end{align}
     for any $(k,n)\in \N^2$ and $t\in [0,\Tf]$, where $\foqbrwb$ is the one-dimensional version of $\foqbrw$. We begin by noting that
     \begin{align}
          \fopbrwb_{0,0} (k,n) = 0^{k+n-2\min(k,n)} \text{ for } k,n \in \N \eqsp.
     \end{align}
     If $k=n$,  $\fopbrwb_{0,0} (k,n) = 0^0 = 1$, otherwise, it reduces to $0$. Thus the initial condition $\fopbrwb_{0,0} = \Id$ holds.
     To verify the remaining equation, for $k,n \in \N$ and $t \in [0,\Tf]$, let us consider the three following cases:

     \underline{1. $k \leq n-1$:} we have
     \begin{align}
         &\quad \partial_t \fopbrwb_{0,t}(k,n)\\
         &= \sum_{j=0}^{k} \binom{k}{j} \frac{\rme^{\rme^{-t}-1-tj}}{(n-j)!}(1-\rme^{-t})^{k+n-2j-1} \l[(-\rme^{-t}-j)(1-\rme^{-t}) + (k+n-2j) \rme^{-t}  \r] \\
         &= \sum_{j=0}^{k} \binom{k}{j} \frac{\rme^{\rme^{-t}-1-tj}}{(n-j)!}(1-\rme^{-t})^{k+n-2j-1} \l[(-\rme^{-t}-j)(1-\rme^{-t}) + (n-j) \rme^{-t}  \r]\\
         &\qquad + \sum_{j=0}^{k-1} \binom{k}{j+1} \frac{\rme^{\rme^{-t}-1-t(j+1)}}{(n-j)!}(1-\rme^{-t})^{k+n-2j-1} (j+1) \\
         &= \sum_{j=0}^{k} \binom{k}{j} \frac{\rme^{\rme^{-t}-1-tj}}{(n-j)!}(1-\rme^{-t})^{k+n-2j-1} \l[\rme^{-2t} + \rme^{-t}(n-1)-j \r] \\
         &\qquad + \sum_{j=0}^{k} \binom{k}{j} \frac{\rme^{\rme^{-t}-1-tj}}{(n-j+1)!}(1-\rme^{-t})^{k+n-2j+1} j \\
         &= \sum_{j=0}^{k} \binom{k}{j} \frac{\rme^{\rme^{-t}-1-tj}}{(n-j)!} (1-\rme^{-t})^{k+n-2j-1} \\
         &\qquad\l[\rme^{-2t} + \rme^{-t}(n-1)-j +\frac{j(1-\rme^{-t})^2}{n-j+1}\r] \eqsp.
     \end{align}
    On the other hand,
     \begin{align}
          &\quad\sum_{i\in \N} \fopbrwb_{0,t} (k,i)\foqbrwb (i,n)\\
         &= \sum_{i=n-1}^{n+1} \fopbrwb_{0,t} (k,i)\foqbrwb (i,n) \\
         &= \sum_{j=0}^{k} \binom{k}{j}\frac{\rme^{\rme^{-t}-1-tj}}{(n-j)!} (1-\rme^{-t})^{k+n-2j-1} \\
         &\qquad \l[ n-j + \frac{n+1}{n-j+1} (1-\rme^{-t})^2 - (n+1)(1-\rme^{-t}) \r] \\
         &= \sum_{j=0}^{k} \binom{k}{j}\frac{\rme^{\rme^{-t}-1-tj}}{(n-j)!} (1-\rme^{-t})^{k+n-2j-1} \\
         &\qquad \l[ n-j + (1-\rme^{-t})^2- (n+1)(1-\rme^{-t}) +\frac{j(1-\rme^{-t})^2 }{n-j+1}  \r] \\
         &= \sum_{j=0}^{k} \binom{k}{j} \frac{\rme^{\rme^{-t}-1-tj}}{(n-j)!} (1-\rme^{-t})^{k+n-2j-1} \l[\rme^{-2t} + \rme^{-t}(n-1)-j +\frac{j(1-\rme^{-t})^2}{n-j+1}\r] \eqsp,
     \end{align}
     and therefore \eqref{eq:kolm_forward_brw} holds.

      \underline{2. $k \geq n+1$:} same computation yields
      \begin{align}
          &\partial_t \fopbrwb_{0,t}(k,n) =  \binom{k}{n+1} {\rme^{\rme^{-t}-1-t(n+1)}}(1-\rme^{-t})^{k-n-1} (n+1) \\
          & +\sum_{j=0}^{n} \binom{k}{j} \frac{\rme^{\rme^{-t}-1-tj}}{(n-j)!} (1-\rme^{-t})^{k+n-2j-1} \l[\rme^{-2t} + \rme^{-t}(n-1)-j +\frac{j(1-\rme^{-t})^2}{n-j+1}\r] \eqsp,
      \end{align}
     and
     \begin{align}
         &\quad\sum_{i\in \N} \fopbrwb_{0,t} (k,i)\foqbrwb (i,n)\\
         & = \sum_{j=0}^{n-1} \binom{k}{j} \frac{\rme^{\rme^{-t}-1-tj}}{(n-1-j)!} (1-\rme^{-t})^{k+n-2j-1} \\
         &\quad+ \sum_{j=0}^{n+1} \binom{k}{j} \frac{\rme^{\rme^{-t}-1-tj}}{(n+1-j)!} (1-\rme^{-t})^{k+n-2j+1}(n+1) \\
          &\quad - \sum_{j=0}^{n} \binom{k}{j} \frac{\rme^{\rme^{-t}-1-tj}}{(n-j)!} (1-\rme^{-t})^{k+n-2j}(n+1) \\
          &= \sum_{j=0}^{n} \binom{k}{j} \frac{\rme^{\rme^{-t}-1-tj}}{(n-j)!} (1-\rme^{-t})^{k+n-2j-1} \l[n-j + \frac{n+1}{n-j+1} (1-\rme^{-t})^2 - (n+1)(1-\rme^{-t}) \r] \\
          &\qquad+ \binom{k}{n+1} \rme^{\rme^{-t}-1-(n+1)t} (1-\rme^{-t})^{k-n-1}(n+1) \eqsp.
     \end{align}
   Hence \eqref{eq:kolm_forward_brw} is satisfied.

   \underline{3. $k=n$:} computing analogously, we obtain
    \begin{align}
          &\partial_t \fopbrwb_{0,t}(k,n) = -\rme^{\rme^{-t}-1-(n+1)t}(n+1)\\
          &+\sum_{j=0}^{n-1} \binom{n}{j} \frac{\rme^{\rme^{-t}-1-tj}}{(n-j)!} (1-\rme^{-t})^{2n-2j-1} \l[\rme^{-2t} + \rme^{-t}(n-1)-j +\frac{j(1-\rme^{-t})^2}{n-j+1}\r] \eqsp.
      \end{align}
      Moreover,
      \begin{align}
           &\quad\sum_{i\in \N} \fopbrwb_{0,t} (k,i)\foqbrwb (i,n)\\
           &= \sum_{j=0}^{n-1} \binom{n}{j} \frac{\rme^{\rme^{-t}-1-tj}}{(n-j)!} (1-\rme^{-t})^{2n-2j-1} \l[n-j + \frac{n+1}{n-j+1}(1-\rme^{-t})^2 - (n+1)(1-\rme^{-t}) \r] \\
           &\qquad + {\rme^{\rme^{-t}-1-nt}}(n+1)(1-\rme^{-t}) - 1) \\
           &= \sum_{j=0}^{n-1} \binom{n}{j} \frac{\rme^{\rme^{-t}-1-tj}}{(n-j)!} (1-\rme^{-t})^{2n-2j-1} \l[\rme^{-2t} + \rme^{-t}(n-1)-j +\frac{j(1-\rme^{-t})^2}{n-j+1}\r] \\
           &\qquad -\rme^{\rme^{-t}-1-(n+1)t}(n+1) = \partial_t \fopbrwb_{0,t}(k,n) \eqsp.
      \end{align}
     We conclude that
$\fopbrwb$ is indeed a solution to \eqref{eq:kolm_forward_brw}. By \cite[Theorem 4.3]{feinberg2014solutions}, this equation admits a unique solution within the class of transition probabilities satisfying \textbf{H2}. Hence $\fopbrwb$ is a component-wise conditional density of $(\foXbrw_t)_{t\in [0,\Tf]}$ and the full transition density is then given by
 \begin{align}
      \fopbrw_{0,t} (x,y) = \prod_{i=1}^d \fopbrwb_{0,t} (x^i,y^i) \eqsp, \quad \text{for $x, y \in \N^d$ and $t \in [0,\Tf]$} \eqsp,
  \end{align}
  which is the desired conclusion.
\end{proof}

\begin{remark}
   The insight of formula \eqref{eq:kolm_forward_brw} comes from the following expression:
  \begin{align}
      (\foXbrw_t)^\ell = (\Sbrw_t)^\ell + (\Abrw_t)^\ell \eqsp, \quad  \text{for $t \in [0,\Tf]$} \eqsp,
  \end{align}
 where \( (\Sbrw_t)^\ell |(\foXbrw_0)^\ell \sim \mathrm{Binomial} \l( (\foXbrw_0)^\ell, \rme^{-t} \r) \) denotes the number of survivors from $\foXbrw_0$ initial particles since each particle survives independently with probability $\rme^{-t}$ (the death rate of $n$ particles being $n$), and \( (\Abrw_t)^\ell \sim \mathrm{Poisson}(1-\rme^{-t}) \) represents the number of arrivals still alive at time $t$, generated by the constant birth rate $1$ (the forward rate).
\end{remark}

\begin{proposition}\label{prop:bound_uq_square_brw}
      Assume $\m^\star_2<\infty$, then for any $t \in [0,\Tf)$, the following holds
      \begin{align}
          \sum_{\sigma \in \mcs} \E \l[ \l(\uBRW_t \foqbrw (\baXbrw_t, \sigma(\baXbrw_t)) \r)^2 \r] \lesssim [\mstar_1 + \m^\star_2 +d]\l[1+\frac{1}{(\Tf-t)^2} \r] \eqsp.
      \end{align}
\end{proposition}
\begin{proof}[Proof of \Cref{prop:bound_uq_square_brw}]
  Recall that $\uBRW_t \foqbrw (x,\sigma(x))$ admits a conditional expectation expression for any $(t,x) \in [0,\Tf) \times \N^d$ and $\sigma \in \mcs$; see \Cref{prop:condition_expectation_brw}. Therefore,
  \begin{align}
      &\quad \sum_{\sigma \in \mcs} \E \l[ \l(\uBRW_s \foqbrw (\baXbrw_s, \sigma(\baXbrw_s)) \r)^2 \r] \notag \\
      &=   \E \l[\sum_{\sigma \in \mcs} \E \l[  \foqbrw(\sigma(x),x)\frac{\fopbrw_{0,\Tf-s}(\foXbrw_0,\sigma(x))}{\fopbrw_{0,\Tf-s}(\foXbrw_0,x)} \middle| \foXbrw_{\Tf-s} = x  \r]^2   \r]  \notag \\
      &\leq  \E \l[ \underbrace{\sum_{\sigma \in \mcs} \l(  \foqbrw(\sigma(\foXbrw_{\Tf-s}),\foXbrw_{\Tf-s})\frac{\fopbrw_{0,\Tf-s}(\foXbrw_0,\sigma(\foXbrw_{\Tf-s}))}{\fopbrw_{0,\Tf-s}(\foXbrw_0,\foXbrw_{\Tf-s})} \r)^2}_{\dbf_{\Tf-s} (\foXbrw_0,\foXbrw_{\Tf-s})}   \r]  \eqsp,\label{eq:integrability_1}
  \end{align}
  where we used Jensen's inequality in the last inequality. Substituting the formula of transition density given in \eqref{eq:transition_formula_brw} yields
  \begin{align}
      \dbf_s (x_0,x_s) &= \sum_{\ell=1}^d
      \Bigg( \underbrace{(x_s^\ell+1) \frac{\sum_{j=0}^{\min(x_0^\ell, x_s^\ell+1)} \binom{x_0^\ell}{j} \rme^{-sj} (1-\rme^{-s})^{x_0^\ell+x_s^\ell+1-2j}/(x_s^\ell+1-j)! }{\sum_{j=0}^{\min(x_0^\ell, x_s^\ell)} \binom{x_0^\ell}{j} \rme^{-sj} (1-\rme^{-s})^{x_0^\ell+x_s^\ell-2j}/(x_s^\ell-j)! }}_{\dbf_s^1(x_0, x_s) } \\
      &\qquad +\underbrace{
      \frac{\sum_{j=0}^{\min(x_0^\ell, x_s^\ell - 1)} \binom{x_0^\ell}{j} \rme^{-sj} (1-\rme^{-s})^{x_0^\ell+x_s^\ell - 1-2j}/(x_s^\ell-1-j)! }{\sum_{j=0}^{\min(x_0^\ell, x_s^\ell)} \binom{x_0^\ell}{j} \rme^{-sj} (1-\rme^{-s})^{x_0^\ell+x_s^\ell-2j}/(x_s^\ell-j)! }
      }_{\dbf_s^2(x_0,x_s)}
      \Bigg)^2
     \eqsp.
  \end{align}
  We have
  \begin{align}
      \dbf_s^1 (x_0, x_s)  &= \1_{x_0^\ell \geq x_s^\ell+1} (x_s^\ell+1)(1-\rme^{-s}) \frac{\sum_{j=0}^{x_s^\ell+1} \binom{x_0^\ell}{j} \rme^{-sj} (1-\rme^{-s})^{x_0^\ell+x_s^\ell-2j}/(x_s^\ell+1-j)! }{\sum_{j=0}^{ x_s^\ell } \binom{x_0^\ell}{j} \rme^{-sj} (1-\rme^{-s})^{x_0^\ell+x_s^\ell-2j}/(x_s^\ell-j)! } \\
      & + \1_{x_0^\ell \leq x_s^\ell} (x_s^\ell+1)(1-\rme^{-s}) \frac{\sum_{j=0}^{x_0^\ell} \binom{x_0^\ell}{j} \rme^{-sj} (1-\rme^{-s})^{x_0^\ell+x_s^\ell-2j}/(x_s^\ell+1-j)! }{\sum_{j=0}^{ x_0^\ell } \binom{x_0^\ell}{j} \rme^{-sj} (1-\rme^{-s})^{x_0^\ell+x_s^\ell-2j}/(x_s^\ell-j)! } \\
      &\leq  \1_{x_0^\ell \geq x_s^\ell+1} (x_s^\ell+1)(1-\rme^{-s}) \frac{\sum_{j=0}^{x_s^\ell} \binom{x_0^\ell}{j} \rme^{-sj} (1-\rme^{-s})^{x_0^\ell+x_s^\ell-2j}/(x_s^\ell-j)!     }{\sum_{j=0}^{ x_s^\ell } \binom{x_0^\ell}{j} \rme^{-sj} (1-\rme^{-s})^{x_0^\ell+x_s^\ell-2j}/(x_s^\ell-j)! }\\
      &\qquad + \1_{x_0^\ell \geq x_s^\ell+1} (x_s^\ell+1) \frac{ \binom{x_0^\ell}{x_s^\ell+1} \rme^{-s(x_s^\ell+1)} (1-\rme^{-s})^{x_0^\ell-x_s^\ell-1}    }{ \binom{x_0^\ell}{x_s^\ell} \rme^{-s x_s^\ell} (1-\rme^{-s})^{x_0^\ell - x_s^\ell} }\\
      &\qquad \qquad + \1_{x_0^\ell \leq x_s^\ell} (x_s^\ell+1)(1-\rme^{-s}) \\
    &= (x_s^\ell+1) (1-\rme^{-s}) +  \1_{x_0^\ell \geq x_s^\ell+1}(x_0^\ell-x_s^\ell) \rme^{-s} (1-\rme^{-s})^{-1} \\
    &\leq  (x_s^\ell+1) (1-\rme^{-s}) +  \frac{x_0^\ell}{ \rme^s-1 } \eqsp.
  \end{align}
  We evaluate next $\dbf_s^2$ as follows
  \begin{align}
      \dbf_s^2 (x_0, x_s) &=  \1_{x_0^\ell \geq x_s^\ell } \frac{(1-\rme^{-s})^{-1}\sum_{j=0}^{ x_s^\ell -1} \binom{x_0^\ell}{j} \rme^{-sj} (1-\rme^{-s})^{x_0^\ell+x_s^\ell -2j}(x_s^\ell-j)/(x_s^\ell-j)! }{\sum_{j=0}^{ x_s^\ell} \binom{x_0^\ell}{j} \rme^{-sj} (1-\rme^{-s})^{x_0^\ell+x_s^\ell-2j}/(x_s^\ell-j)! }\\
      &\quad + \1_{x_0^\ell \leq x_s^\ell-1} \frac{(1-\rme^{-s})^{-1} \sum_{j=0}^{x_0^\ell} \binom{x_0^\ell}{j} \rme^{-sj} (1-\rme^{-s})^{x_0^\ell+x_s^\ell -2j}(x_s^\ell-j)/(x_s^\ell-j)! }{\sum_{j=0}^{x_0^\ell} \binom{x_0^\ell}{j} \rme^{-sj} (1-\rme^{-s})^{x_0^\ell+x_s^\ell-2j}/(x_s^\ell-j)! } \\
      &\leq \1_{x_0^\ell \geq x_s^\ell } \frac{x_s^\ell (1-\rme^{-s})^{-1}\sum_{j=0}^{ x_s^\ell} \binom{x_0^\ell}{j} \rme^{-sj} (1-\rme^{-s})^{x_0^\ell+x_s^\ell -2j}/(x_s^\ell-j)! }{\sum_{j=0}^{ x_s^\ell} \binom{x_0^\ell}{j} \rme^{-sj} (1-\rme^{-s})^{x_0^\ell+x_s^\ell-2j}/(x_s^\ell-j)! }\\
      &\quad + \1_{x_0^\ell \leq x_s^\ell-1} \frac{x_s^\ell (1-\rme^{-s})^{-1} \sum_{j=0}^{x_0^\ell} \binom{x_0^\ell}{j} \rme^{-sj} (1-\rme^{-s})^{x_0^\ell+x_s^\ell -2j}/(x_s^\ell-j)! }{\sum_{j=0}^{x_0^\ell} \binom{x_0^\ell}{j} \rme^{-sj} (1-\rme^{-s})^{x_0^\ell+x_s^\ell-2j}/(x_s^\ell-j)! } \\
      &=  x_s^\ell (1-\rme^{-s})^{-1} \1_{x_0^\ell \geq x_s^\ell } +  x_s^\ell (1-\rme^{-s})^{-1} \1_{x_0^\ell \leq x_s^\ell-1} = \frac{x_s^\ell}{1-\rme^{-s}} \eqsp.
  \end{align}
  Therefore
  \begin{align}
       \dbf_s^1 (x_0, x_s) + \dbf_s^2 (x_0, x_s) \leq (x_s^\ell+1) (1-\rme^{-s}) +  \frac{x_0^\ell}{ \rme^s-1 } + \frac{x_s^\ell}{1-\rme^{-s}} \eqsp.
  \end{align}
  Since $\dbf_s^1 (x_0, x_s) + \dbf_s^2 (x_0, x_s)$ is nonnegative, we then obtain
  \begin{align}
      (\dbf_s^1 (x_0, x_s) + \dbf_s^2 (x_0, x_s))^2 &\leq \l( (x_s^\ell+1) (1-\rme^{-s}) +  \frac{x_0^\ell}{ \rme^s-1 } + \frac{x_s^\ell}{1-\rme^{-s}} \r)^2 \\
      &\lesssim ((x^\ell_s)^2+1)(1-\rme^{-s})^2 + \frac{(x_0^\ell)^2}{ (\rme^s-1)^2 } + \frac{(x_s^\ell)^2}{(1-\rme^{-s})^2} \eqsp,
  \end{align}
  which follows
  \begin{align}
      \dbf_s (x_0, x_s) &\lesssim \sum_{\ell=1}^d \l[ ((x^\ell_s)^2+1)(1-\rme^{-s})^2 + \frac{(x_0^\ell)^2}{ (\rme^s-1)^2 } + \frac{(x_s^\ell)^2}{(1-\rme^{-s})^2} \r] \\
      &= (\| x_s\|_2^2+d)(1-\rme^{-s})^2 + \frac{\| x_0 \|_2^2}{(\rme^s-1)^2} + \frac{\|x_s\|_2^2}{(1-\rme^{-s})^2} \eqsp.
  \end{align}
  Plugging this estimate into \eqref{eq:integrability_1} gives
  \begin{align}
      &\quad \sum_{\sigma \in \mcs} \E \l[ \l(\uBRW_s \foqbrw (\baXbrw_s, \sigma(\baXbrw_s)) \r)^2 \r] \\
      &\lesssim [\m_2 (\mubrw_{\Tf-s})+d]  (1-\rme^{-(\Tf-s)})^2 + \m_2 (\mustar)  (\rme^{\Tf-s}-1)^{-2}\\
      &\qquad + \m_2(\mubrw_{\Tf-s})  (1-\rme^{-(\Tf-s)})^{-2} \eqsp.
  \end{align}
 Additionally, the formula of $\m_2(\mubrw_t)$ showed in \Cref{prop:second_moment_brw} implies \[\m_2(\mubrw_t) \lesssim \mstar_1+\m^\star_2+d \quad \text{for any $t \in [0,\Tf)$} \eqsp. \]
 Hence
  \begin{align}
      &\quad \sum_{\sigma \in \mcs} \E \l[ \l(\uBRW_s \foqbrw (\baXbrw_s, \sigma(\baXbrw_s)) \r)^2 \r] \\
      &\lesssim [\mstar_1+\m^\star_2+d] \Bigg[ \underbrace{(1-\rme^{-(\Tf-t)})^2}_{\leq 1}  + \underbrace{\frac{1} {(\rme^{\Tf-t}-1)^{2}}}_{\leq \frac{1}{(\Tf-t)^2}} + \frac{1}{ (1-\rme^{-(\Tf-t)})^{2}} \Bigg] \eqsp.
  \end{align}
  Relying on the fact that
  \begin{align}
      \frac{1}{ (1-\rme^{-(\Tf-t)})^{2}} &= \frac{\rme^{2(\Tf-t)}}{(\rme^{\Tf-t}-1)^2} = 1+ \frac{2}{\rme^{\Tf-t}-1}+ \frac{1}{(\rme^{\Tf-t}-1)^2} \\
      &\leq 1+ \frac{2}{\Tf-t} + \frac{1}{(\Tf-t)^2} \lesssim 1+ \frac{1}{(\Tf-t)^2} \eqsp,
  \end{align}
  we can deduce the following for $t \in [0,\Tf)$,
  \begin{align}
      \sum_{\sigma \in \mcs} \E \l[ \l(\uBRW_s \foqbrw (\baXbrw_s, \sigma(\baXbrw_s)) \r)^2 \r] \lesssim [\mstar_1+\m^\star_2+d] \l[1+ \frac{1}{(\Tf-t)^2} \r] \eqsp.
\end{align}
The proof is then complete.
\end{proof}

\begin{lemma}\label{lem:true_martingale_brw}
   Assume $\mustar$ has finite second order moment, then the following process is a martingale
   \[\Mbrw_\sigma(t) = \int_{[0,t] \times \N^d} \l[ \f^\sigma(s,x) - \f^\sigma(s, \baXbrw_{s-}) \r] \tilde{N}_{\baXbrw}^{\baqbrw} (\rmd x \rmd s) \eqsp,\]
 for $t \in [0,\Tf)$ and fixed $\sigma \in \mcs$, where
  \(
  \mathbf{f}^{{\sigma}}(t,\baXbrw_{t} )
    \eqdef \foqbrw \uBRW_t(\baXbrw_{t},{\sigma}(\baXbrw_t)) -1 \eqsp.
  \)
\end{lemma}
\begin{proof}[Proof of \Cref{lem:true_martingale_brw}]
Recall that $\tilde{N}_{\baXbrw}^{\baqbrw}$ denotes the compensated measure of the random point measure $N_{\baXbrw}^{\baqbrw}$ corresponding to the CTMC associated with $(\baqbrw_t)_{t\in [0,\Tf)}$; see Appendix \ref{sec:stochastic_calculus} for completeness.
Since
\[\tilde{N}^{\baqbrw}_{\baXbrw} = {N}_{\baXbrw}^{\baqbrw} - \bar \rmn_{\baXbrw}^{\baqbrw} \eqsp,\]
where $\bar \rmn_{\baXbrw}^{\baqbrw}$ is the compensator of ${N}_{\baXbrw}^{\baqbrw}$, we obtain that
\begin{align}
    \E[|\Mbrw_\sigma(t)|] &\leq \E \l[ \int_{[0,t] \times \N^d} | \f^\sigma(s,x) - \f^\sigma(s, \baXbrw_{s-}) | {N}_{\baXbrw}^{\baqbrw} (\rmd s \rmd x) \r] \\
    &\quad+ \E \l[ \int_{[0,t] \times \N^d} | \f^\sigma(s,x) - \f^\sigma(s, \baXbrw_{s-}) | \bar \rmn_{\baXbrw}^{\baqbrw} (\rmd s \rmd x) \r] \\
    &= 2\E \l[ \int_{[0,t] \times \N^d} | \f^\sigma(s,x) - \f^\sigma(s, \baXbrw_{s-}) | \bar \rmn_{\baXbrw}^{\baqbrw} (\rmd s \rmd x) \r] \eqsp.
\end{align}
Showing $\E[|\Mbrw_\sigma(t)|] < \infty$ is equivalent to showing that 
\begin{align}
    D^\sigma_t = \E \l[\int_{[0,t]} \sum_{\sigma' \in \mcs} \l| \mathbf{f}^{{\sigma}}(s,\sigma'(\baXbrw_{s}) ) - \mathbf{f}^{{\sigma}}(s,\baXbrw_{s} ) \r| (\mathbf{f}^{{\sigma'}}(s,\baXbrw_{s} )+1) \rmd s  \r] < \infty \eqsp.
   \end{align}
By Cauchy-Schwarz inequality, it suffices to show
\[
\int_{[0,t]} \E \l[ \sum_{\sigma' \in \mcs} \l( \mathbf{f}^{{\sigma}}(s,\sigma'(\baXbrw_{s}) ) - \mathbf{f}^{{\sigma}}(s,\baXbrw_{s} ) \r)^2 \rmd s  \r] <\infty
\]
and
\[
\int_{[0,t]} \E \l[ \sum_{\sigma' \in \mcs} (\mathbf{f}^{{\sigma'}}(s,\baXbrw_{s} )+1)^2 \rmd s  \r] <\infty
\]
to obtain the finiteness of $D^\sigma_t$.
By \Cref{prop:bound_uq_square_brw}, we have
\begin{align}
    &\quad\int_{[0,t]} \E \l[ \sum_{\sigma' \in \mcs} (\mathbf{f}^{{\sigma'}}(s,\baXbrw_{s} )+1)^2   \r] \rmd s\\
    &= \int_{[0,t]} \E \l[ \sum_{\sigma' \in \mcs}   \l(\uBRW_s \foqbrw (\baXbrw_s, \sigma'(\baXbrw_s)) \r)^2  \r] \rmd s \\
    &\lesssim [\mstar_1+\m^\star_2+d] \int_0^t \l[1+ \frac{1}{(\Tf-s)^2} \r] \rmd s \\
    &= [\mstar_1+\m^\star_2+d] \l(t +\frac{1}{\Tf-t}-\frac{1}{\Tf}\r) \eqsp,
\end{align}
which is finite for any $t \in [0,\Tf)$.
It remains to show
  \[
  \int_{[0,t]} \E \l[ \sum_{\sigma' \in \mcs} \l( \mathbf{f}^{{\sigma}}(s,\sigma'(\baXbrw_{s}) ) - \mathbf{f}^{{\sigma}}(s,\baXbrw_{s} ) \r)^2 \rmd s  \r] <\infty \eqsp.
  \]
  Notice that $\mathbf{f}^\sigma(t,x) +1 = \foqbrw \uBRW_t(x,\sigma(x))$ is nonnegative for all $(t,x) \in [0,\Tf) \times \N^d$ and $\sigma \in \mcs$, thus
  \begin{align}
      &\quad\int_{[0,t]} \E \l[ \sum_{\sigma' \in \mcs} \l( \mathbf{f}^{{\sigma}}(s,\sigma'(\baXbrw_{s}) ) - \mathbf{f}^{{\sigma}}(s,\baXbrw_{s} ) \r)^2 \r] \rmd s  \\
      &\leq \underbrace{\int_{[0,t]} \E \l[ \sum_{\sigma' \in \mcs} \l( \mathbf{f}^{{\sigma}}(s,\sigma'(\baXbrw_{s}) ) +1 \r)^2  \r]
      \rmd s }_{A_1^\sigma} +  \underbrace{\int_{[0,t]} \E \l[ \sum_{\sigma' \in \mcs} \l( \mathbf{f}^{{\sigma}}(s,\baXbrw_{s} ) +1 \r)^2   \r] \rmd s}_{A_2^\sigma} \eqsp.
  \end{align}
  The term $A_2^\sigma$ can be estimated by \Cref{prop:bound_uq_square_brw}:
  \begin{align}
      A_2^\sigma \leq 2d \int_{[0,t]} \E \l[ \sum_{\sigma \in \mcs} \l( \mathbf{f}^{{\sigma}}(s,\baXbrw_{s} ) +1 \r)^2 \r] \rmd s <\infty \eqsp.
  \end{align}
  Next, we control $A_1^\sigma$ as follows
  \begin{align}
      A_1^\sigma &\leq \int_{\Tf-t}^{\Tf} \sum_{\sigma' \in \mcs} \E   \l(  \foqbrw(\sigma(\sigma'(\foXbrw_{s})),\sigma'(\foXbrw_{s}))\frac{\fopbrw_{0,s}(\foXbrw_0,\sigma(\sigma'(\foXbrw_{s})))}{\fopbrw_{0,s}(\foXbrw_0,\sigma'(\foXbrw_{s}))} \r)^2  \rmd s \eqsp.
  \end{align}
  Computing as in \Cref{prop:bound_uq_square_brw}, we obtain for any $\ell \in [d]$,
  \begin{align}
      A_1^{\sigma^\ell_+} \lesssim \int_{\Tf-t}^{\Tf} \sum_{\sigma' \in \mcs} \E \l[ (\|\sigma'(\foXbrw_s)\|_2^2+1)(1-\rme^{-s})^2 + \|\foXbrw_0\|_2^2 (\rme^s-1)^{-2}  \r] \rmd s \eqsp,
  \end{align}
  and
  \begin{align}
      A_1^{\sigma^\ell_-} \lesssim \int_{\Tf-t}^{\Tf} \sum_{\sigma' \in \mcs} \E \l[ (\|\sigma'(\foXbrw_s)\|_2^2+1)(1-\rme^{-s})^{-2} \r] \rmd s \eqsp.
  \end{align}
  For any $\sigma' \in \mcs$, we have $\|\sigma'(x)\|_2 \leq \|x\|_2+1$, hence $\|\sigma'(x)\|_2^2 \leq 2\|x\|^2_2+2$ . Consequently, for any $t \in [0,\Tf)$,
  \begin{align}
      A_1^{\sigma^\ell_+} \lesssim 2d \l[ (\m_2(\mubrw_s)+1)\int_{\Tf-t}^{\Tf} (1-\rme^{-s})^2 \rmd s + \m^\star_2\int_{\Tf-t}^{\Tf} (\rme^s-1)^{-2} \rmd s \r] <\infty \eqsp,
  \end{align}
  and
  \begin{align}
      A_1^{\sigma^\ell_-} \lesssim 2d  (\m_2(\mubrw_s)+1)\int_{\Tf-t}^{\Tf} (1-\rme^{-s})^{-2} \rmd s  <\infty \eqsp.
  \end{align}
 Thus we get the desired integrability of $\Mbrw_\sigma(t)$ for any $t  \in [0,\Tf)$ and conclude the proof of \Cref{lem:true_martingale_brw}.
\end{proof}

\section{Additional Proofs of the main paper}

\subsection{Proof of \Cref{lem:bound_fisher_masked}}\label{proof_lem:bound_fisher_masked}

  For $t \in [0,\Tf)$, we have
  \begin{align}
      \Ic(\mum_{\Tf-t}) &= \E \l[ \sum_{i\in [d]}\sum_{j \in \Z_m}\l(\uM_t \1_{i \in \msm_{\Tf-t}} \log \uM_t - \uM_t \1_{i \in \msm_{\Tf-t}} +1 \r) (\foXm_{\Tf-t},\um^{(i)}_j (\foXm_{\Tf-t}))  \r] \notag \\
      &\overset{\log a \leq a-1}{\leq} \E \l[ \sum_{i\in [d]}\sum_{j \in \Z_m}\l(\uM_t \1_{i \in \msm_{\Tf-t}} -1 \r)^2 (\foXm_{\Tf-t},\um^{(i)}_j (\foXm_{\Tf-t}))  \r] \notag \\
      &\quad= \underbrace{\E \l[ \sum_{j \in \Z_m} \sum_{i\in \msm_{\Tf-t}} \l(\uM_t -1 \r)^2 (\foXm_{\Tf-t},\um^{(i)}_j (\foXm_{\Tf-t}))  \r]}_{A_1} + m \E \l[|\msm^{\complement}_{\Tf-t} | \r] \eqsp. \label{eq:fisher_1}
  \end{align}
  From \Cref{prop:condition_expectation_masked_d}, the first term $A_1$ can be written as
  \begin{align}
      A_1 = \E \l[ \sum_{j \in \Z_m} \sum_{i\in \msm_{\Tf-t}} \l( \E \l[ \frac{\fopm_{0,\Tf-t}(\foXm_0,\um^{(i)}_j(x))}{\fopm_{0,\Tf-t}(\foXm_0,x)} -1 \middle| \foXm_{\Tf-t} = x  \r] \r)^2   \r] \eqsp.
  \end{align}
  Applying Jensen inequality gives
  \begin{align}
      A_1 \leq \E \l[ \sum_{j \in \Z_m} \sum_{i\in \msm_{\Tf-t}}  \l( \frac{\fopm_{0,\Tf-t}(\foXm_0,\um^{(i)}_j(\foXm_{\Tf-t}))}{\fopm_{0,\Tf-t}(\foXm_0,\foXm_{\Tf-t})} -1 \r)^2\r] \eqsp.
  \end{align}
  Using the formula of the transition probability given in eq. (14) in the main paper, we deduce
  \begin{align}
      A_1 &\leq \E \l[ \sum_{j \in \Z_m} \sum_{i\in \msm_{\Tf-t}}  \l( \frac{\fopmm_{0,\Tf-t}((\foXm_0)^i, j)}{\fopmm_{0,\Tf-t}((\foXm_0)^i,m)} -1 \r)^2\r] \\
      &= \E \l[ \sum_{i\in \msm_{\Tf-t}}  \l( \frac{\alpha_{\Tf-t}}{1-\alpha_{\Tf-t}} -1 \r)^2\r] = \l( \frac{\alpha_{\Tf-t}}{1-\alpha_{\Tf-t}} -1 \r)^2 \E \l[ |\msm_{\Tf-t}| \r] \eqsp.
  \end{align}
  On the other hand, by the previous computation in eq. (45) of the main paper, we obtain that
  \begin{align}
      \E \l[|\msm_{\Tf-t}| \r] = d (1-\alpha_{\Tf-t}) \quad \text{and} \quad \E \l[ |\msm_{\Tf-t}^{\complement}| \r] = d- \E \l[|\msm_{\Tf-t}| \r] = d \alpha_{\Tf-t} \eqsp.
  \end{align}
  Plugging altogether into \eqref{eq:fisher_1} implies
  \begin{align}
      \Ic(\mum_{\Tf-t}) &\leq d\l( \frac{\alpha_{\Tf-t}}{1-\alpha_{\Tf-t}} -1 \r)^2 (1-\alpha_{\Tf-t}) + md \alpha_{\Tf-t} \\
      &\leq d \l( \frac{\alpha_{\Tf-t}^2}{1-\alpha_{\Tf-t}}-2\alpha_{\Tf-t} +1 - \alpha_{\Tf-t} + m\alpha_{\Tf-t} \r) \\
      &  \overset{\alpha_{\Tf-t} \in (0,1]}{\lesssim} d \l( \frac{\alpha_{\Tf-t}}{1-\alpha_{\Tf-t}} + m \r) \eqsp.
  \end{align}
  In particular, for a constant generator $\beta(t) =1$ for all $t \in [0,\Tf]$, we then obtain the closed-form of $\alpha_t$ as follows
  \begin{align}
      \alpha_t = \rme^{-\int_0^t 1 \rmd s} = \rme^{- t} \quad \text{for $t \in [0,\Tf]$} \eqsp.
  \end{align}
  As a result, we obtain a clear and specific bound on $\Ic(\mum_{\Tf-t})$ for any $t \in [0,\Tf)$:
  \begin{align}
      \Ic(\mum_{\Tf-t}) \lesssim d \l( \frac{1}{\rme^{\Tf-t}-1} +m \r) \overset{\rme^a \geq a+1}{\leq} d \l( \frac{1}{\Tf-t}+m \r) \eqsp,
  \end{align}
  and the proof concludes.

\subsection{Proof of \Cref{theo:start_delta_mask}}\label{sec:proof_theo:start_delta_mask}

Recall that $(\baXmstar_t)_{t\in [0,\Tf-\eta]}$ and $(\baYmstar_t)_{t\in [0,\Tf-\eta]}$ correspond to the same Markov kernel $(\bapmstar_t)_{t\in [0,\Tf-\eta]}$ associated with the estimated generator $(\baqmstar_t)_{t\in [0,\Tf-\eta]}$ but are initialized differently: $\baXmstar_0 \sim \mathrm{Uniform}(\Z^d_m)\fopm_{0,\Tf}$ while $\baYmstar_0 \sim \updelta_{\text{MASK}}^{\otimes d}$.
We begin by proving the following inequality
\begin{align}\label{eq:data_processing_tv}
    \tvnorm{\mathrm{Law}(\baXmstar_{\Tf-\eta}) - \mathrm{Law}(\baYmstar_{\Tf-\eta})} \leq \tvnorm{
    \mathrm{Law}(\baXmstar_0) - \mathrm{Law}( \baYmstar_0) } \eqsp.
\end{align}
Indeed, by definition of total variation, we have
\begin{align}
     &\quad\tvnorm{\mathrm{Law}(\baXmstar_{\Tf-\eta}) - \mathrm{Law}(\baYmstar_{\Tf-\eta})}\\
     &= \frac{1}{2} \sum_{y \in \tZ^d_m} \l| \sum_{x\in \tZ^d_m} \bapmstar_{0,\Tf-\eta}(x,y) \l[ (\mathrm{Law}(\baXmstar_0))(x) - (\mathrm{Law}(\baYmstar_0))(x) \r]  \r| \\
     &\leq \frac{1}{2}  \sum_{x\in \tZ^d_m} \l| (\mathrm{Law}(\baXmstar_0))(x) - (\mathrm{Law}(\baYmstar_0))(x)  \r| \underbrace{\sum_{y \in \tZ^d_m} \bapmstar_{0,\Tf-\eta}(x,y)}_{=1} \\
     &= \tvnorm{
    \mathrm{Law}(\baXmstar_0) - \mathrm{Law}( \baYmstar_0) } \eqsp.
\end{align}
On the other hand, the right hand side of \eqref{eq:data_processing_tv} can be computed explicitly as 
\begin{align}
    \tvnorm{
    \mathrm{Law}(\baXmstar_0) - \mathrm{Law}( \baYmstar_0) } &= \tvnorm{\mathrm{Uniform}(\Z^d_m) \fopm_{0,\Tf} - \updelta_{\text{MASK}}^{\otimes d} } \\
    &= \frac{1}{2} \l| (\mathrm{Uniform}(\Z^d_m) \fopm_{0,\Tf})(m,\ldots,m) -1  \r| \\
    &\leq 1-(1-\alpha_{\Tf})^d \eqsp.
\end{align}
Applying the Bernoulli's inequality for $\alpha_{\Tf} \in [0,1]$ gives us \(
    1- (1-\alpha_{\Tf})^d \leq d\alpha_{\Tf}
\)
, thereby
\begin{align}
    \tvnorm{\mathrm{Law}(\baXmstar_{\Tf-\eta}) - \mathrm{Law}(\baYmstar_{\Tf-\eta})} \leq d\alpha_{\Tf} \eqsp.
\end{align}
Combining this with Theorem 3.1.2 and using the triangle inequality yield the first claim of Theorem 3.1.4. The second claim can be deduced straightforwardly from Theorem 3.1.3 and the closed-form of $\alpha_t = \rme^{-t}$ for any $t \in [0,\Tf]$. With the particular choices of $\Tf, c$ and $\eta$ in Theorem 3.1.4, we then have the error
\begin{align}
    \tvnorm{\mustar - \mathrm{Law}(\baYmstar_{\Tf-\eta})} \lesssim \varepsilon+ \varepsilon\sqrt{\Tf} = \tilde{O}(\vareps)\eqsp,
\end{align}
and the number of iterations is 
\begin{align}
    K &\lesssim \frac{\Tf+ \log(m+\eta^{-1})}{\log \l(1+ \varepsilon^2/[dm\Tf + dm \log (m+d/\varepsilon)] \r)} = \tilde O(dm/\varepsilon^2) \eqsp.
\end{align}

\subsection{Proof of \Cref{lem:2_brw}}\label{proof_lem:2_brw}

For any $t \in [0,\Tf)$ and $\ell \in [d]$,  following Lemma 5.2.1 of the main paper, we have
    \begin{align}
        &\quad\E \l[ \uBRW_t \foqbrw (\baXbrw_t, \sigma^\ell_-(\baXbrw_t)) \r]\\
        &= \sum_{x\in \N^d} \frac{\mubrw_{\Tf-t}(\sigma^\ell_-(x))}{\mubrw_{\Tf-t}(x)}.\frac{\brw(x)}{\brw(\sigma^\ell_-(x))} \foqbrw (x,\sigma^\ell_-(x))\mubrw_{\Tf-t}(x) \\
        &= \sum_{x^\ell \geq 1} \frac{\brw(x)}{\brw(\sigma^\ell_-(x))} \foqbrw (x,\sigma^\ell_-(x))\mubrw_{\Tf-t}(\sigma^\ell_-(x)) \\
        &{=} \sum_{x^\ell \geq 1}  \foqbrw (\sigma^\ell_-(x),x)\mubrw_{\Tf-t}(\sigma^\ell_-(x)) \\
        &= \sum_{x\in \N^d} \foqbrw (x, \sigma^\ell_+(x)) \mubrw_{\Tf-t}(x) \\
        &= \E \l[\foqbrw(\baXbrw_t, \sigma^\ell_+(\baXbrw_t)) \r] \eqsp.
    \end{align}
The second claim is obtained analogously
and we conclude the proof.

\subsection{Proof of \Cref{lem:1_brw}}\label{proof_lem:1_brw}

    The characterization of the discrete score in Lemma 5.2.1 of the main paper yields
    \begin{align}
        & \quad
        \E \l[ \uBRW_t \log \uBRW_t \foqbrw(\baXbrw_t, \sigma^\ell_+(\baXbrw_t))   \r]\\
        &= \sum_{x\in \N^d} \log \l(\frac{\tmubrw_{\Tf-t}(\sigma^\ell_+(x))}{\tmubrw_{\Tf-t}(x)}\r) \frac{\tmubrw_{\Tf-t}(\sigma^\ell_+(x))}{\tmubrw_{\Tf-t}(x)} \foqbrw(x,\sigma^\ell_+(x)) \mubrw_{\Tf-t}(x)\\
        &= - \sum_{x\in \N^d} \log \l(\frac{\tmubrw_{\Tf-t}(x)}{\tmubrw_{\Tf-t}(\sigma^\ell_+(x))}\r) \frac{\brw(x)}{\brw(\sigma^\ell_+(x))} \foqbrw(x,\sigma^\ell_+(x)) \mubrw_{\Tf-t}(\sigma^\ell_+(x))\\
        &{=} -\sum_{x\in\N^d} \log \l(\frac{\tmubrw_{\Tf-t}(x)}{\tmubrw_{\Tf-t}(\sigma^\ell_+(x))}\r)  \foqbrw(\sigma^\ell_+(x),x) \mubrw_{\Tf-t}(\sigma^\ell_+(x)) \\
        &= -\sum_{x \in \N^d} \log \l(\frac{\tmubrw_{\Tf-t}(\sigma^\ell_-(x))}{\tmubrw_{\Tf-t}(x)}\r)  \foqbrw(x,\sigma^\ell_-(x)) \mubrw_{\Tf-t}(x) \\
        &= - \E \l[ \log \uBRW_t \foqbrw (\baXbrw_t, \sigma^\ell_-(\baXbrw_t)) \r] \eqsp.
    \end{align}
    The second claim follows by the same reasoning, which completes the proof.
    
\subsection{Proof of \Cref{prop:4_brw}}\label{sec:proof_prop:4_brw}
Following \cite[Proposition 2.1]{conforti2022probabilistic}, we have
    \begin{equation}\label{estimate:monotone}
        - \frac{\rmd }{\rmd t} \mE (\mathbf{h}'(\tmubrw_t), \tmubrw_t) \geq  \mE (\mathbf{h}'(\tmubrw_t),\tmubrw_t) \eqsp,
    \end{equation}
    where $\tmubrw = \mubrw/\brw$ denotes the relative marginal density of $(\foXbrw_t)_{t\in [0,\Tf]}$, function $\mathbf{h}(a) = a\log a-a+1$ for $a>0$ and the Dirichlet form admits the following explicit formula
    \begin{equation}
        \mE (f,g) = - \sum_{x \in \N^d} f(x) (\foqbrw g)(x)  \brw (x) \eqsp, \quad \text{for any functions $f,g$} \eqsp.
    \end{equation}
  The Dirichlet form in~\eqref{estimate:monotone} can be computed as:
    \begin{align}
        &\quad \mE (\mathbf{h}'(\tmubrw_t),\tmubrw_t) = \mE \l( \log \tmubrw_t, \tmubrw_t \r)\\
        &= -\sum_{x \in \N^d}
        \log \tmubrw_t(x) \l( \foqbrw \tmubrw_t \r)(x)  \brw (x) \\
        &= -\sum_{x \in \N^d}
        \log \tmubrw_t(x) \sum_{\sigma \in \mcs} \foqbrw ( x, \sigma(x))
        \l( \tmubrw_t (\sigma(x)) - \tmubrw_t (x) \r)  \brw (x)\\
        &= -\sum_{x \in \N^d} \sum_{\sigma \in \mcs}
        \log  \tmubrw_t (x) \foqbrw ( x, \sigma(x))
        \l( \frac{\tmubrw_t (\sigma(x))}{ \tmubrw_t (x)}  - 1 \r) \tmubrw_t (x)  \brw (x) \\
        &= -\sum_{x \in \N^d} \sum_{\sigma \in \mcs}
        \foqbrw ( x, \sigma(x))
        \log  \tmubrw_t (x)
        \l( \uBRW_{\Tf-t} (x, \sigma(x))  - 1 \r)   \mubrw_t (x)  \\
        &= \E \l[
        \sum_{\sigma \in \mcs}
        \foqbrw ( \foXbrw_t, \sigma(\foXbrw_t))
        \log \frac{1}{\tmubrw_t (\foXbrw_t)}
        \l( \uBRW_{\Tf-t} (\foXbrw_t, \sigma(\foXbrw_t))  - 1 \r)
        \r] \eqsp.
    \end{align}
    We write the above in terms of $\uBRW_{\Tf-t}$ as
    \begin{align}
        &\quad\mE (\mathbf{h}'(\tmubrw_t),\tmubrw_t) \\
        &= \E \Bigg[
        \sum_{\sigma \in \mcs}
        \foqbrw ( \foXbrw_t, \sigma(\foXbrw_t))\l( \uBRW_{\Tf-t} (\foXbrw_t, \sigma(\foXbrw_t))  - 1 \r)\\
        &\qquad
        \l\{
        \log \uBRW_{\Tf-t} (\foXbrw_t, \sigma(\foXbrw_t))
        - \log \tmubrw_t (\sigma(\foXbrw_t)
        \r\}
        \Bigg]\\
        &= \E \l[
        \sum_{\sigma \in \mcs}
        \uBRW_{\Tf-t}\log \uBRW_{\Tf-t}\foqbrw ( \foXbrw_t, \sigma(\foXbrw_t)) \r] + \E \Big[ \foqbrw ( \foXbrw_t, \sigma(\foXbrw_t))\\
        &\qquad \quad\Big\{
         \log  \tmubrw_t (\foXbrw_t)- \log  \tmubrw_t (\sigma (\foXbrw_t))  \uBRW_{\Tf-t} (\foXbrw_t, \sigma)
        \Big\}
        \Big] \eqsp.
    \end{align}
    The last term can be cancelled out thanks to the following computation
    \begin{align}
        & \quad \E \l[\sum_{\sigma \in \mcs} \log  \tmubrw_t (\sigma (\foXbrw_t))  \uBRW_{\Tf-t} \foqbrw (\foXbrw_t, \sigma(\foXbrw_t))  \r]\\
        &= \sum_{x \in \N^d} \sum_{\sigma \in \mcs} \foqbrw (x, \sigma (x) )
        \log  \tmubrw_t (\sigma(x)) \frac{\brw (x)}{\brw (\sigma(x))} \mubrw_t(\sigma(x)) \\
        &{=} \sum_{x \in \N^d} \sum_{\sigma \in \mcs} \foqbrw (\sigma(x), x)
        \log  \tmubrw_t (\sigma(x))   \mubrw_t(\sigma(x))  \\
        &= \sum_{x \in \N^d} \sum_{\sigma \in \mcs} \foqbrw ( x, \sigma(x))
        \log  \tmubrw_T (x)  \mubrw_t(x) \\
        &= \E \l[ \sum_{\sigma \in \mcs}
        \foqbrw (\foXbrw_t, \sigma(\foXbrw_t))
        \log \tmubrw ( \foXbrw_t)
        \r] \eqsp.
    \end{align}
    Consequently, the Dirichlet form is simplified as
    \begin{align}
       \mE (\mathbf{h}'(\tmubrw_t),\tmubrw_t)  &= \E \l[
        \sum_{\sigma \in \mcs}
        \uBRW_{\Tf-t}\log \uBRW_{\Tf-t}  \foqbrw ( \foXbrw_t, \sigma(\foXbrw_t)) \r]\\
        &= \E \l[
        \sum_{\sigma \in \mcs}
        \foqbrw
        \uBRW_{\Tf-t}\log \uBRW_{\Tf-t} ( \baXbrw_{\Tf-t}, \sigma(\baXbrw_{\Tf-t})) \r] \eqsp.
    \end{align}
For $t \in [0,\Tf]$, denote
\begin{align}
    \Ic_{\brw}(\mubrw_{\Tf-t}) &\eqdef \Ic_{\brw}(\mubrw_{\Tf-t}) = \E \l[\sum_{\sigma \in \mcs}\mathbf{h} (\uBRW_t)\foqbrw (\baXbrw_t, \sigma(\baXbrw_t) ) \r]\\
    &\overset{Lemma 5.2.2}{=}\E \l[
        \sum_{\sigma \in \mcs}
        \uBRW_{t}\log \uBRW_{t} \foqbrw ( \baXbrw_{t}, \sigma(\baXbrw_{t})) \r] \eqsp,
\end{align}
then the computation above shows that $\Ic_{\brw}(\mubrw_{\Tf-t}) = \mE (\mathbf{h}'(\tmubrw_{\Tf-t}), \tmubrw_{\Tf-t})$. Thus~\eqref{estimate:monotone} implies
\begin{equation}
    \frac{\rmd}{ \rmd t} \Ic_{\brw}(\mubrw_{\Tf-t}) \geq  \Ic_{\brw}(\mubrw_{\Tf-t}) \eqsp.
\end{equation}
 Combining this with Gronwall's lemma yields
\begin{equation}
    \Ic_{\brw}(\mubrw_{\Tf-t}) \geq \rme^{t-s} \Ic_{\brw}(\mubrw_{\Tf-s}) \eqsp, \quad \text{for any } 0 \leq s \leq t \leq \Tf \eqsp.
\end{equation}
In particular, $\Ic_{\brw}(\mubrw_{\Tf-t})$ is non-decreasing on $[0,\Tf]$ and the proof concludes.

 \subsection{Proof of \Cref{prop:bound_fisher_rw}}\label{proof_prop:bound_fisher_rw}
Recall that $\Ic(\murw_{\Tf-s}) \geq \Ic(\murw_{\Tf-t})$ for $0 \leq t \leq s \leq \Tf$. Integrating from $t$ to $\Tf$ gives
\begin{align}
   \int_0^{\Tf-t} \Ic(\murw_{s}) \rmd s = \int_t^{\Tf} \Ic(\murw_{\Tf-s}) \rmd s \geq \Ic(\murw_{\Tf-t}) \int_t^{\Tf}1\rmd s = (\Tf-t) \Ic(\murw_{\Tf-t}) \eqsp.\label{eq:bound_fisher_rw_1}
\end{align}
By direct computation, we can show that
  \begin{align}
      \Ic(\murw_{t}) = -2\frac{\rmd }{\rmd t} \KL(\murw_t|\rw) \eqsp,
  \end{align}
 thus \eqref{eq:bound_fisher_rw_1} implies
  \begin{align}
      \Ic(\murw_{\Tf-t}) &\leq
       (\Tf-t)^{-1} \int_0^{\Tf-t} -2\frac{\rmd }{\rmd s} \KL(\murw_s|\rw) \rmd s \notag \\
      &\lesssim (\Tf-t)^{-1} \l[\KL(\mustar|\rw) - \KL (\murw_{\Tf-t}|\rw) \r] \leq (\Tf-t)^{-1} \KL(\mustar|\rw) \eqsp.\label{eq:bound_fisher_rw_2}
  \end{align}
  To this end, note that
  \begin{align}
      \KL(\mustar|\rw) &= \sum_{x\in \Z^d_m} \mustar(x)\underbrace{\log \mustar(x)}_{\leq 0} - \sum_{x\in \Z^d_m} \mustar(x) \log \rw(x) \notag  \\
      &\leq \sum_{x\in \N^d} \mustar(x) \log \l( m^d  \r) \eqsp, \text{ since $\rw = \mathrm{Uniform}(\Z^d_m)$} \notag \\
      &= d\log (m) \sum_{x \in \Z^d_m} \mustar(x) = d\log (m) \eqsp,
  \end{align}
  which completes the proof.

 \subsection{Proof of \Cref{theo:early_stopping_rw}}\label{sec:proof_theo:early_stopping_rw}
 
For $\eta \in (0,\Tf)$, we have
\begin{align}
    \tvnorm{\mustar - \murw_{\eta}} & \leq \P(\foXrw_\eta \neq \foXrw_0) \leq \sum_{x\in \Z^d_m} \mustar(x) (1-\rme^{-d\eta}) = 1-\underbrace{\rme^{-d\eta}}_{\geq -d\eta+1} \leq d\eta  \eqsp.\label{eq:theo_early_stopping_rw_1}
\end{align}
Furthermore, proceeding similarly as in Theorem 3.2.2 for the early-stopped process implies
\begin{align}
     \KL(\murw_\eta|\mathrm{Law}(\baXrwstar_{\Tf-\eta}))
      &\lesssim  \rme^{-\frac{16\pi^2}{25m^2}\Tf}d\log(m) +  \varepsilon^2\Tf + {cd \log(m)} \log(L_\eta) \eqsp,
\end{align}
where $L_\eta= d^{-1}\Ic(\murw_\eta)$. Relying on the upper bound of discrete Fisher information showed in Lemma 5.2.8, we can deduce that
\begin{align}
    \KL(\murw_\eta|\mathrm{Law}(\baXrwstar_{\Tf-\eta}))
      &\lesssim  \rme^{-\frac{16\pi^2}{25m^2}\Tf}d\log(m) +  \varepsilon^2\Tf + {cd \log(m)} \log(\eta^{-1}\log(m)) \eqsp. \label{eq:theo_early_stopping_rw_2}
\end{align}
Combining \eqref{eq:theo_early_stopping_rw_1} and \eqref{eq:theo_early_stopping_rw_2}, by triangle and Pinsker's inequalities, we then have
\begin{align}
    &\quad \tvnorm{\mustar - \mathrm{Law}(\baXrwstar_{\Tf-\eta})}\\
    &\leq \tvnorm{\mustar - \murw_{\eta}} + \tvnorm{\murw_\eta - \mathrm{Law}(\baXrwstar_{\Tf-\eta})} \\
    &\lesssim d\eta + \sqrt{\rme^{-\frac{16\pi^2}{25m^2}\Tf} d\log(m)+  \varepsilon^2\Tf + {cd \log(m)} \log(\eta^{-1}\log(m)) } \eqsp.
\end{align}
Furthermore, choosing $\eta$, $\Tf$ and $c$ as in  eq. (31) of Theorem 3.2.3 imply
  \begin{align}
      \tvnorm{\mustar - \mathrm{Law}(\baXrwstar_{\Tf-\eta})} &\lesssim \varepsilon^2+ \sqrt{\varepsilon^2\Tf} \\
      &\lesssim \varepsilon^2+ m \sqrt{\varepsilon^2\log (d\log(m)/\varepsilon^2)} \eqsp.
  \end{align}
  The number of iterations is calculated as follows
  \begin{align}
       &\quad K = k_0 + k_1 +K-k_0-k_1\\
       &{\lesssim} \dfrac{\Tf -\eta  - 1}{c} + \dfrac{\log(L)}{c} + \dfrac{1}{c} \overset{Lemma 5.2.8}{\lesssim} \dfrac{\Tf + \log(\eta^{-1}\log(m))}{c} \\
       &\lesssim \frac{d \log (m)\log (d \log(m)/\varepsilon)[m^2 \log (d\log(m)/\varepsilon^2) + \log (d\log(m)/\varepsilon)] }{\varepsilon^2} \eqsp.
  \end{align}
 Therefore, our algorithm has the complexity $\tilde O(dm^2/\varepsilon^2)$, where the notation $\tilde O$ indicates that all the logarithms of $d,m,\varepsilon$ have been dropped, and the proof of Theorem 3.2.3 concludes.

\subsection{Proof of \Cref{prop:bound_fisher_brw}}\label{proof_prop:bound_fisher_brw}
  Recall that
  \begin{align}
      \frac{\rmd }{\rmd t} \KL(\mubrw_t|\brw) = -\mE(\mathbf{h}'(\tmubrw_t), \tmubrw_t) \overset{Lemma 5.2.7}{=} -\Ic_{\brw}(\mubrw_{t}) \eqsp.
  \end{align}
  Proceeding as in the proof of  Lemma 5.2.8,
  we then obtain
  \begin{align}
      \Ic_{\brw}(\mubrw_{\Tf-t}) 
      &\leq (\Tf-t)^{-1} \KL(\mustar|\brw) \eqsp.\label{eq:bound_fisher_1_brw}
  \end{align}
  To this end, let us compute $\KL(\mustar|\brw)$:
  \begin{align}
      \KL(\mustar|\brw) &= \sum_{x\in \N^d} \mustar(x)\underbrace{\log \mustar(x)}_{\leq 0} - \sum_{x\in \N^d} \mustar(x) \log \brw(x)  \notag \\
      &\leq \sum_{x\in \N^d} \mustar(x) \log \l(\rme^d \prod_{\ell=1}^d ((x^\ell)!)   \r) \notag\\
      &= d + \sum_{x \in \N^d} \mustar(x) \sum_{\ell=1}^d \sum_{k=1}^{x^\ell} \underbrace{\log(k)}_{\leq k-1 \leq x^\ell} \1_{x^\ell \geq 1} \notag \\
      &\leq d + \sum_{x \in \N^d} \sum_{\ell=1}^d (x^\ell)^2 \mustar(x) = d+ \m^\star_2 \eqsp.
  \end{align}
   Replacing it into \eqref{eq:bound_fisher_1_brw} yields the desired upper bound.
 
 \subsection{Proof of \Cref{theo:early_stopping_brw}}\label{sec:proof_theo:early_stopping_brw}
 
For $\eta \in (0,\Tf)$, the total variation between $\mustar$ and $\mubrw_{\eta}$ is evaluated as follows
\begin{align}
    \tvnorm{\mustar - \mubrw_{\eta}} &= \frac{1}{2}\sum_{x\in \N^d} |\mustar(x) - \mubrw_\eta(x) | \leq \P(\foXbrw_\eta \neq \foXbrw_0) \notag \\
    &\leq \sum_{x\in \N^d} \mustar(x) (1-\rme^{-(d+\sum_{\ell=1}^d x^\ell)\eta}) \notag \\
    &\overset{\rme^a \geq a+1}{\leq} \sum_{x\in \N^d} \mustar(x)\l(d+\sum_{\ell=1}^d x^\ell \r) \eta \notag \\
    &=\eta \E \l[d+\sum_{\ell=1}^d (\foXbrw_0)^\ell \r] = \eta (d + \mstar_1)\eqsp.\label{eq:tv_early_brw}
\end{align}
On the other hand, note that the assumption $\mstar_2$ suffices for all the computation in Theorem 3.3.1. Hence proceeding similarly for stopped process $(\baXbrwstar_t)_{t\in [0,\Tf-\eta]}$ gives
\begin{align}
    \KL(\mubrw_\eta|\mathrm{Law}(\baXbrwstar_{\Tf-\eta} ) ) &\lesssim \rme^{-\Tf} (d+\mstar_2) + \varepsilon^2\Tf \notag \\
    &+ (\rme^c-1)\l[ d \Tf + (d+\mstar_2) \log (L) \r]   \eqsp.\label{eq:early_stopping_brw_1}
\end{align}
Relying on the upper bound of discrete Fisher information given in Lemma 5.2.9 yields
\begin{align}
    \KL(\mubrw_\eta|\mathrm{Law}(\baXbrwstar_{\Tf-\eta} ) ) &\lesssim \rme^{-\Tf} (d+\mstar_2) + \varepsilon^2\Tf \notag \\
      &+ (\rme^c-1)\l[ d\Tf + (d+\mstar_2)\log (\eta^{-1}(1+\mstar_2 d^{-1})) \r]   \eqsp.\label{eq:kl_early_brw}
\end{align}
From \eqref{eq:tv_early_brw} and \eqref{eq:kl_early_brw}, by triangle and Pinsker's inequalities, we arrive at
\begin{align}
    &\quad\tvnorm{\mustar - \mathrm{Law}(\baXbrwstar_{\Tf-\eta})}\\
    &\leq \tvnorm{\mustar - \mubrw_{\eta}} + \tvnorm{\mubrw_\eta - \mathrm{Law}(\baXbrwstar_{\Tf-\eta})} \\
    &\lesssim \tvnorm{\mustar - \mubrw_{\eta}} + \sqrt{\KL(\mubrw_\eta|\mathrm{Law}(\baXbrwstar_{\Tf-\eta} ) )}\\
   &\lesssim \eta (d+\mstar_1) \\
   &+\sqrt{
   \rme^{-\Tf} (d+\mstar_2) + \varepsilon^2\Tf + (\rme^c-1)\l[ d\Tf + (d+\mstar_2)\log (\eta^{-1}(1+\mstar_2 d^{-1})) \r]} \eqsp.
\end{align}
In particular, setting $\eta, \Tf,c$ as in eq. (38), (39) of Theorem 3.3.2 directly implies
\begin{align}
     \tvnorm{\mustar- \mathrm{Law}(\baXbrwstar_{ \Tf-\eta} ) } \lesssim  \varepsilon^2+ \sqrt{\varepsilon^2\log ((d+\mstar_2)/\varepsilon^2)} \tilde O(\vareps) \eqsp,
\end{align}
and the number of iterations is given by
\begin{align}
    K \lesssim \frac{\Tf + \log(L)}{c} = \tilde O((d+\mstar_2)/\varepsilon^2),
\end{align}
since $1/\log(1+h)$ admits the complexity $\tilde O(1/h)$ for $h \approx 0$. Thus, the proof of Theorem 3.3.2 concludes.

\section{Stochastic-calculus viewpoint of CTMCs}\label{sec:stochastic_calculus}

We refer to \cite[Appendix F.1]{pham2025discrete} for a detailed introduction to point processes, stochastic integrals with respect to point processes, and the corresponding It\^o's formula. In this section, we present the stochastic-calculus viewpoint on CTMCs.

 Let $(\msx, \mcx)$ be a measurable space with $\msx \subset \R^d$ and let $(X_t)_{t \in [0,\Tf]}$ be a CTMC on $\msx$ associated with the transition rate function $q: (0,\Tf] \times \msx^2 \to \rset$ satisfying \textbf{H1} and \textbf{H2}
, where $q(X_{t-},x)$ represents the rate of jumping from the current state $X_{t-}$ to the new state $x \in \msx$ at time $t \in (0,\Tf]$. The CTMC $(X_t)_{t \in [0,\Tf]}$ defines a point process $\mathbf{p}_X= (\mathbf{p}_X(t))_{t \in [0,\Tf]}$ on $(\msx, \mcx)$, where
\begin{align}
    \mathbf{p}_X : D_{\mathbf{p}_X} \subset (0,\Tf] \to \msx \eqsp, \quad \mathbf{p}_X (t) = X_t \eqsp, \, t\in D_{\mathbf{p}_X} \eqsp,
\end{align}
with $D_{\mathbf{p}_X}$ is the set of jump times of $(X_t)_{t\in\ccint{0,\Tf}}$. We observe that $\mathbf{p}_X$ describes the new state after jumping at time $t \in (0,\Tf]$ and it constructs a corresponding random measure $N_{\mathbf{p}_X}^{q} (\rmd t \rmd x)$ on $(0,\Tf] \times \msx$ by
\begin{equation}
\begin{aligned}
     N_{\mathbf{p}_X}^{q} ((0,t] \times \mathsf{U}) &= \mathrm{Card} \l\{ s \in D_{\mathbf{p}_X}: s \leq t , \mathbf{p}_X (s) \in \mathsf{U} \r\}\\
     &= \sum_{s \in D_{\mathbf{p}_X}} \updelta_{(s,X_s)} ((0,t] \times \mathsf{U}) \qquad  \text{for } t \in (0,\Tf], \quad \mathsf{U} \in \mcx \eqsp,
\end{aligned}
\end{equation}
that counts the total jumps into $\mathsf{U} \in \mcx$ occurring during $(0,t]$. Then the random compensator $\bar\rmn_{\mathbf{p}_X}^{q}$ of $N_{\mathbf{p}_X}^{q}$ is given by
\begin{equation}
\begin{aligned}
     \bar{\mathrm{n}}_{\mathbf{p}_X}^{q} (\rmd t \rmd x) &= \sum_{y \in \msx}  q(X_{t-}, y) \1_{X_{t-} \neq y } \updelta_y(\rmd x)  \rmd t \eqsp,
\end{aligned}
\end{equation}
since the corresponding compensated measure
\begin{equation}
    \tilde N_{\mathbf{p}_X}^{q} (\rmd t \rmd x) = N_{\mathbf{p}_X}^{q} (\rmd t \rmd x) - \bar\rmn_{\mathbf{p}_X}^{q} (\rmd t \rmd x)
\end{equation}
is an $(\mF_t)$-martingale, where $(\mF_t)_{t \in [0,\Tf]}$ denotes the right-continuous and complete natural filtration generated by the process $(X_t)_{t \in [0,\Tf]}$.
Indeed, we can show the martingale property of $\tilde N_{\mathbf{p}_X}^{q}$ as follows.  For any function $f \in F^1_{\mathbf{p}_X}$, where the class $F^1_{\mathbf{p}_X}$ is given by
\begin{align}
    F^1_{\mathbf{p}} = \Big\{ f(t,x,\omega); f \text{ is } &(\mF_t)\text{-predictable and for each $t\in (0,\Tf]$}\\
    &\qquad \E\l[\int_0^{t} \int_\msx |f(s,x,\cdot)| \bar\rmn_{\mathbf{p}_X}^{q} (\rmd s \rmd x)\r] < \infty  \Big\} \eqsp,
\end{align}
define the following stochastic integrals:
\begin{equation}
    \begin{aligned}
        \int_0^{t+} \int_{\msx} f(s,x, \cdot) N_{\mathbf{p}_X}^{q} (\rmd s \rmd x) &= \sum_{\substack{ 0 < s \leq t \\ s \in D_{\mathbf{p}_X}}} f(s, \mathbf{p}_X(s), \cdot) \eqsp, \\
        \int_0^{t+} \int_{\msx} f(s,x, \cdot) \tilde N_{\mathbf{p}_X}^{q} (\rmd s \rmd x) &= \int_0^{t+} \int_{\msx} f(s,x, \cdot) N_{\mathbf{p}_X}^{q} (\rmd s \rmd x) - \int_0^{t} \int_{\msx} f(s,x, \cdot) \bar\rmn_{\mathbf{p}_X}^{q} (\rmd s \rmd x) \eqsp.
    \end{aligned}
\end{equation}
Then for $0 \leq s < t \leq \Tf$ 
, we have
\begin{equation}
    \begin{aligned}
        &\quad\E \l[ \int_s^{t+} \int_{\msx} f(z,x, \cdot) N_{\mathbf{p}_X}^{q} (\rmd z \rmd x) \middle| \mF_s \r] \\
        &= \E \l[ \sum_{\substack{ s < z \leq t \\ z \in D_{\mathbf{p}_X}}} f(z, X_z, \cdot) \middle| \mF_s \r] \\
        &= \E \l[ \int_{\msx} \int_s^{t} f(z, x, \cdot) q (X_{z-}) \frac{q(X_{z-},x)}{q (X_{z-})} \1_{ X_{z-} \neq x }  \rmd z  \rmd x \middle| \mF_s\r]  \\
        &= \E \l[  \int_s^t \int_{\msx} f(z,x, \cdot) q( X_{z-}, x ) \1_{ X_{z-} \neq x } \rmd x \rmd z \middle| \mF_s\r] \\
        &= \E \l[ \int_s^t \int_{\msx} f(z,x, \cdot) \bar\rmn_{\mathbf{p}_X}^{q} (\rmd z \rmd x) \middle| \mF_s \r] \eqsp,
    \end{aligned}
\end{equation}
meaning that $\bar\rmn_{\mathbf{p}_X}^{q}$ is indeed the compensator of the random measure $N_{\mathbf{p}_X}^{q}$.
With those notations in hand, we can decompose the CTMC $(X_t)_{t \in [0,\Tf]}$ as
\begin{equation}
    X_t = X_0 + \sum_{ \substack{0 < s \leq t \\ s \in D_{\mathbf{p}_X} }} (X_s - X_{s-})   = X_0 + \int_0^{t+} \int_{\msx} (x - X_{s-}) N_{\mathbf{p}_X}^{q} (\rmd s \rmd x) \eqsp, \quad \text{for } t \in [0,\Tf] \eqsp,
\end{equation}
under the assumption $f(s,x,\omega):= x-\omega_{s-} \in F^1_{\mathbf{p}_X}$. Applying It\^o's formula to this process, for any bounded function $F: \msx \to \rset$, we get that
\begin{equation}
\begin{aligned}
    F(X_t) - F(X_0) &= \int_0^{t+} \int_{\msx} \l\{ F(X_{s-} + x - X_{s-}) - F(X_{s-}) \r\} N_{\mathbf{p}_X}^{q} (\rmd s \rmd x) \\
    &= \int_0^{t+} \int_{\msx} \l\{ F(x) - F(X_{s-}) \r\} N_{\mathbf{p}_X}^{q} (\rmd s \rmd x)
\end{aligned}
\end{equation}

Expressing $N_{\mathbf{p}_X}^{q} = \tilde N_{\mathbf{p}_X}^{q} + \bar\rmn_{\mathbf{p}_X}^{q}$ and plugging it into the formula above yield
\begin{equation}
    \begin{aligned}
         &\quad F(X_t) - F(X_0) - \int_0^t \int_{\msx} \l\{ F(x) - F(X_{s-}) \r\} \bar\rmn_{\mathbf{p}_X}^{q} (\rmd s \rmd x)\\
         &= \int_0^{t+} \int_{\msx}  \l\{ F(x) - F(X_{s-}) \r\} \tilde N_{\mathbf{p}_X}^{q} (\rmd s \rmd x) \eqsp.
    \end{aligned}
\end{equation}
In other words, the process

 \begin{equation}
     \begin{aligned}
        \l( F(X_t) - F(X_0) - \int_0^t \int_{\msx } \l\{ F(x) - F(X_{s-}) \r\} \1_{ \{X_{s-} \neq x \}}  q(X_{s-}, \rmd x)   \rmd s  \r)_{ t \in [0,\Tf] }
     \end{aligned}
 \end{equation}
is an $(\mF_t)$-local martingale as the compensated measure $\tilde N_{\mathbf{p}_X}^{q}$ was shown to be an $(\mF_t)$-martingale in the previous computation.
It follows that for the CTMC $(X_t)_{t \in [0,\Tf]}$ with generator $(q_t)_{t\in (0,\Tf]}$, It\^o's formula asserts that the process
\begin{equation}
   \l( F(X_t) - F(X_0) - \int_0^t q F (X_{s-}) \rmd s \r)_{t\in [0,\Tf]}
\end{equation}
is an $(\mF_t)$-local martingale for any bounded function $F : \msx \to \rset$. This result aligns with Dynkin’s formula.

\section{CTMC and their corresponding martingale problem}

 \cite[Theorem 4.1]{ethier2009markov} showed that the stable conservative generator $(q_t)_{t\geq 0}$ defines a well-posed martingale problem, whose unique solution is the distribution of the CTMC $(X_t)_{t\geq 0}$ associated with $(q_t)_{t \geq 0}$. To be more precise, the definition of the martingale problem is given as follows.

Recall that the distribution of a stochastic process $(Y_t)_{t\geq 0}$ is a solution of the martingale associated with a rate matrix $(\canoq_t)_{t\geq 0}$ and initial distribution $\mubrw_0$ and write $\mathrm{Law}((Y_t)_{t\geq 0}) \in \MP(\canoq, \mubrw_0)$ if $Y_0 \sim \mubrw_0$, almost surely $t\mapsto Y_t$ is Borel measurable, and for any bounded function $f: \msx \to \rset$,
 \begin{align}
     \l(f(\canoX_t)-f( \canoX_0)-\int_0^t \canoq(s) f(\canoX_{s-}) \rmd s \r)_{t \geq 0}
    \end{align}
    is an $(\widetilde{\mcf}_t)_{t \geq 0}$-local martingale, where
    \[\widetilde{\mcf}_t = \sigma( Y_s \,: \, s \leq t) \vee
    \sigma(\int_0^t g(\canoX_{s}) \rmd s \, : \, s \leq t \, , \text{
      $g:\msx \to\rset$ bounded}) \eqsp.\]
     We say that uniqueness holds for
    the martingale problem associated with $(\canoq_t)_{t\geq 0}$ and
    initial distribution $\mubrw_0$, if for any two processes
    $(Y_t)_{t\geq 0}$ and $(Y_t^\prime)_{t\geq 0}$ are solutions to
    the associated martingale problem, then their distributions are
    equal. If uniqueness holds and for any $\mubrw_0$, the martingale
    problem associated with associated $(\canoq_t)_{t\geq 0}$ and $\mubrw_0$ admits a solution, we say that the martingale problem is well-posed. Note that since $\msx$ is a discrete
    countable space, \cite[Theorem 4.1]{ethier2009markov} implies the
    following.
    \begin{theorem}
      Assume \textbf{H1} and \textbf{H2}. Then, the
      martingale problem associated with $(q_t)_{t\geq 0}$ is
      well-posed. In addition, the distribution of the CTMC $(X_t)_{t\geq 0}$ associated with $(q_t)_{t\geq 0}$ and starting from the distribution $\mubrw_0$, is the unique solution of the martingale problem associated with $(q_t)_{t\geq 0}$ and $\mubrw_0$.
    \end{theorem}

\section{Girsanov change of measure}\label{sec:girsanov}

The convergence guarantee is established by leveraging Girsanov’s theorem and It\^o’s formula for jump processes, combined with the martingale property of the score function along the time-reversed dynamics established in Section 5 of the main paper. We begin with stating the general Girsanov's theorem for jump processes on a discrete state space $\msx$.

Let $\mathbb{D}_{\Tf} = \D(\ccint{0,\Tf};\msx)$ be the canonical space of all càdlàg (right-continuous with left limits) paths from $\ccint{0,\Tf}$ to $\msx$.
Let $\P^{\genericq}$ be the distribution of the CTMC associated with the generator $(\genericq_t)_{t\in [0,\Tf]}$ satisfying \textbf{H1} and \textbf{H2}, endowed with the right-continuous and complete augmentation of the generated filtration, denoted by $(\mathcal{F}_t)_{t \in [0,\Tf]}$.

We define the corresponding jump kernel for any $(\canoX_t)_{t\in\ccint{0,\Tf}} \in\D_{\Tf}$:
\begin{equation}\label{eq:jump_kernel}
   \bar \rmn_{\bfX}^{\genericq}((\canoX_t)_{t\in [0,\Tf]},\rmd t \rmd x) = \rmn_{\bfX}^{\genericq} (t,\rmd x)  \rmd t \eqsp, \quad \rmn_{\bfX}^{\genericq} (t,\rmd x) = \sum_{y \in \msx} \1_{ \canoX_{t-} \neq y  } \genericq_t( \canoX_{t-}, y) \updelta_{y}(\rmd x) \eqsp.
 \end{equation}
 By convention, we denote $\bar \rmn_{\bfX}^{\genericq}((\canoX_t)_{t\in [0,\Tf]},\rmd t \rmd x)$ by $   \bar \rmn_{\bfX}^{\genericq}(\rmd t \rmd x)$ which corresponds to the compensator of $(\canoX_{t})_{t\in\ccint{0,\Tf}}$ under $\P$, if under this distribution $(\canoX_{t})_{t\in\ccint{0,\Tf}}$ is a CTMC with generator $\genericq : \msx^2 \to \rset$.
 Consequently, the compensated sum of jumps $\tilde N_{\bfX}^{\genericq} =N_{\bfX}^{\genericq} -\bar \rmn_{\bfX}^{\genericq}$ forms a martingale under $\P$, where $\mathbf{p}_{\canoX}$, $N_{\bfX}^{\genericq}$, $\bar \rmn_{\bfX}^{\genericq}$, $\rmn_{\bfX}^{\genericq}$ and
$\tilde N_{\bfX}^{\genericq}$ defined as in Appendix \ref{sec:stochastic_calculus}.


From \cite{leonard2012girsanov} we see that, for jump processes, the relative entropy of two path measures can be decomposed with the help of the Young function $ \varrho(a):=\rme^a-a-1$, for $a\in \R$, and its convex conjugate $\varrho^*(b)=(b+1)\log(b+1)-b$ for $b>-1$ with convention $\varrho^*(-1)=1$ and $\varrho^*(b)=\infty$ for $b<-1$. Note that $\varrho$ and $\varrho^*$ are respectively equivalent to $a^2/2$ and $b^2/2$ near zero. This is proven by the following theorem.

\begin{theorem}[Girsanov theorem]\label{girsanovtheorem}
    Let $\P^{u\canoq}$ and $\P^{\canoq}$ are the distribution of CTMCs associated with $\MP(\canoq,\P^{\canoq}_0)$ and $\MP(u\canoq, \P^{u\canoq}_0)$ with $\P^{u\canoq} \ll \P^{\canoq}$, respectively, where $\canoq$ is a given inhomogeneous generator satisfying \textbf{H1} and \textbf{H2}, $\P^{\canoq}_0 \in \Pc(\msx)$ and $\P^{u\canoq}_0 \in \Pc(\msx)$ are given initial distributions and $u$ is a non-negative function from $[0,\Tf]\times \msx^2$ to $\rset_+$ satisfying
    \begin{align}\label{eq:girsanov_integrability_cond}
        \E_{\P^{\canoq}} \l[ \int_{\ccint{0,\Tf} \times \msx}\varrho(\log u_t(\canoX_{t-},x)) \bar \rmn_{\canoX}^{\canoq} (\rmd t \rmd x) \r] <\infty \eqsp.
    \end{align}
    Then the Radon-Nikodym density of $\P^{u\canoq}$ against $\P^{\canoq}$ is given by
    \begin{align}
         &\dfrac{\rmd\P^{u\canoq}}{\rmd \P^{\canoq}} ((\canoX_t)_{t\in\ccint{0,\Tf}})=\dfrac{\rmd\P^{u\canoq}_0}{\rmd\P^{\canoq}_0}(\canoX_0)\\
         &\exp\l( \int_{\ccint{0,\Tf} \times \msx} \log u_t(\canoX_{t-},x) \tilde N_{\canoX}^{\canoq} (\rmd t \rmd x) -\int_{\ccint{0,\Tf} \times \msx}\varrho(\log u_t(\canoX_{t-},x)) \bar \rmn_{\canoX}^{\canoq} (\rmd t \rmd x)  \r) \eqsp,
    \end{align}
    and the $\KL$ divergence reads as
    \begin{align}
        \KL(\Puq|\Pq) &= \KL(\Puq_0|\Pq_0)+\mathbb \E_{\Puq} \l[\int_{\ccint{0,\Tf}} \sum_{x\in \msx} \mathbf{h}(u_t(\canoX_{t},x)) \1_{ \canoX_{t} \neq x }  \canoq (\canoX_{t}, x) \rmd t \r] \eqsp,
    \end{align}
    where $ \mathbf{h}(a) = \varrho^*(a-1)= a\log a -a+1$ for $a>0$.
\end{theorem}
Notably, the integrability condition in \eqref{eq:girsanov_integrability_cond} is, by a Fenchel duality argument, equivalent to the following:
\begin{align}
    \E_{\Puq} \l[\int_{\ccint{0,\Tf}} \sum_{x\in \msx} \mathbf{h}(u_t(\canoX_{t},x)) \1_{ \canoX_{t} \neq x }  \canoq (\canoX_{t}, x) \rmd t \r] <\infty \eqsp.
\end{align}
This integrability condition holds naturally in the models we study, as it can be directly controlled under our assumptions.
The proof of ~\Cref{girsanovtheorem} is given for completeness and is  based on several technical lemmas, which we introduce in the following framework.
Let $(\chi_t)_{t\in [0,\Tf]}$ be a $\R$-valued process on $\ccint{0,\Tf}\times \msx^2$ such that $\int_{[0,\Tf] \times \msx^2} \varrho(\chi_t(\canoX_{t-},x) )\bar \rmn_{\bfX}^{\genericq}(\rmd t \rmd x) <\infty$, $\PP^{\genericq}$-a.s. Define for $t \in [0,\Tf]$,
\begin{equation}
  \label{eq:def_process_Z_chi}
  \Z_t^{\chi}:=\exp\l(\int_{[0,t] \times \msx} \chi_s(\canoX_{s-},x) \tilde N_{\canoX}^{\genericq} (\rmd s \rmd x) -\int_{\ccint{0,t} \times \msx} \varrho(\chi_s(\canoX_{s-},x)) \bar \rmn_{\bfX}^{\genericq} (\rmd s \rmd x) \r)  \eqsp.
\end{equation}

\begin{lemma}\label{Zmartingale}
 Let $\PP^{\genericq}$ be the distribution of a CTMC with generator $\genericq: [0,\Tf] \times \msx^2 \to \R$ satisfying \textbf{H1} and \textbf{H2}. Let $(\chi_t)_{t\in [0,\Tf]}$  be as above.
    Then $\int_{[0,\Tf] \times \msx} \chi_s(\canoX_{s-},x)  \tilde N_{\canoX}^{\genericq} (\rmd s \rmd x)$ is a local $\PP^{\genericq}$-martingale. Moreover, the process $(\Z_t^{\chi})_{t\in\ccint{0,\Tf}}$ defined in~\eqref{eq:def_process_Z_chi} is a local $\PP^{\genericq}$-martingale and a positive $\PP^{\genericq}$-supermartingale, which satisfies
    \[\rmd \Z_t^{\chi} =\Z_{t-}^\chi \int_{ \msx}
    (\rme^{\chi_t(\canoX_{t-},x)}-1) \tilde N^{{\genericq}}_{\canoX} (\rmd t \rmd x) \eqsp.\]
\end{lemma}
\begin{proof}[Proof of Lemma \ref{Zmartingale}]
  By definition,
the process
    \[M_t^\chi:=\int_{\ccint{0,\Tf} \times \msx} \chi_s(\canoX_{s-},x) \tilde N^{{\genericq}}_{\canoX} (\rmd s \rmd x)\]
    is a local $\PP^{\genericq}$-martingale. Denote 
    \[\Y_t^\chi:=M^\chi_t -\int_{\ccint{0,\Tf}}\beta_s \rmd s \quad \text{with} \quad \beta_s:=\int_{ \msx}\varrho(\chi_s(\canoX_{s-},x))\rmn^{{\genericq}}_{\canoX} (s, \rmd x) \eqsp. \]
    Applying It\^o's formula provided to the jump process $(\Y_t^\chi)_{t\in \ccint{0,\Tf}}$ and for a function $f$ of class $\rmC^2(\R)$ implies 
    \begin{align}
        \rmd f(\Y_t^\chi)&=\l[\int_{ \msx}\l[f(\Y_{t-}^\chi+\chi_t(\canoX_{t-},x))-f(\Y_{t-}^\chi) - f'(\Y_{t-}^\chi) \cdot\chi_t(\canoX_{t-},x) \r]\rmn^{\genericq}_{\canoX} (t, \rmd x)\r]\rmd t \\
        &\qquad \qquad \quad+ f'(\Y_{t-}^\chi)\cdot \beta_t\rmd t+\rmd M_t^f \eqsp, \quad \PP^{\genericq}\text{-a.s.} \eqsp,
    \end{align}
    where $M_t$ is given by
    \begin{equation}
         M_t^f = \int_{[0,\Tf] \times \msx} \l[f(Y^{\chi}_{s-} + \chi_s(\canoX_{s-},x)) - f(Y^\chi_{s-}) \r] \tilde N_{\canoX}^{\genericq} (\rmd s \rmd x)
    \end{equation}
    is a local $\PP^{\genericq}$-martingale, since the integrand is $\R$-valued predictable process and $\tilde N_{\canoX}^{\genericq}$ forms a martingale under $\PP^{\genericq}$.
    Using this formula for $f(y)=\rme^y$, we obtain
    \begin{align}
        \rmd \rme^{\Y_t^\chi}&=\left[\int_{\msx} (\rme^{\Y_{t-}^\chi+\chi_t(\canoX_{t-},x)}-\rme^{\Y_{t-}^\chi} - \rme^{\Y_{t-}^\chi} \cdot\chi_t(\canoX_{t-},x) )\rmn^{\genericq}_{\canoX} (t, \rmd x)  \right]\rmd t-\rme^{\Y_{t-}^\chi} \beta_t \rmd t+\rmd M_t^{\exp} \\
        &=\rme^{\Y_{t-}^\chi}\beta_t \rmd t-\rme^{\Y_{t-}^\chi}\beta_t\rmd t+\rmd M_t^{\exp}=\rmd M_t^{\exp} \eqsp, \quad \PP^{\genericq}\text{-a.s.} \eqsp.
    \end{align}
     This implies $\Z_t^{\chi}=\rme^{\Y_t^\chi}$ is a local $\PP^{\genericq}$-martingale and, since $\Z_t^{\chi}$ is positive, we can conclude that $\Z_t^{\chi}$ is a $\PP^{\genericq}$-supermartingale thanks to Fatou's lemma.
     In addition, we have
      \begin{align}
        \rmd M_t^{\exp} &=\int_{ \msx} \l(\rme^{Y^{\chi}_{t-} + \chi_t(\canoX_{t-},x)} - \rme^{Y^\chi_{t-}} \r) \tilde N_{\canoX}^{\genericq} (\rmd t \rmd x)
        =\rme^{\Y_{t-}^\chi} \int_{\msx} (\rme^{\chi_t(\canoX_{t-},x) }-1 ) \tilde N^{{\genericq}}_{\canoX} (\rmd t \rmd x) \eqsp,
    \end{align}
    \ie, $\rmd \Z_t^{\chi} = \Z_{t-}^\chi \int_{\msx} (\rme^{\chi_t(\canoX_{t-},x) }-1 ) \tilde N^{{\genericq}}_{\canoX} (\rmd t \rmd x) $ and we conclude the proof of~\Cref{Zmartingale}.
\end{proof}

We now define the stopping time for $k, j \geq 1$,
\begin{equation}
  \label{eq:def_sigma_jk}
  \sigma^k_j:=\inf \left\{t\in \ccint{0,\Tf}\, :\, \int_{\ccint{0,t} \times \msx}\varrho(\chi_s(\canoX_{s-},x))\bar \rmn_{\bfX}^{\genericq}(\rmd s \rmd x)\geq k \text{ or }  \chi_t(\canoX_{t-},\canoX_t) \notin [-j,k] \right\}\eqsp.
\end{equation}

For $\P \in \mathcal{P}(\D_{\Tf})$, let us denote $\P^{\sigma^k_j}:= \canoX^{\sigma^k_j}_\# \P$ the law under $\P$ of the process $\canoX^{\sigma^k_j}$ which is stopped at the stopping time $\sigma^k_j$.

\begin{lemma}\label{importantlemma}
   Let $\PP^{\genericq}$ be the distribution of a CTMC with generator $\genericq: [0,\Tf] \times \msx^2 \to \R$ satisfying \textbf{H1} and \textbf{H2} and let $(\chi_t)_{t\in [0,\Tf]}$ be as above.
    Let $(\Z_t^{\chi})_{t\in\ccint{0,\Tf}}$ be defined in~\eqref{eq:def_process_Z_chi} and $\sigma^k_j$ be defined in~\eqref{eq:def_sigma_jk}. For all $j, k\geq 1$, the process $(\Z^{\sigma^k_j}_t)_{t\in\ccint{0,\Tf}}$ defined as $\Z^{\sigma^k_j}_t:= \Z_t^{\chi^k_j}$, is a genuine $\PP^{\genericq}$-martingale with $\chi^k_j = \1_{[0,\sigma^k_j]} \chi$, and the measure $\QQ^k_j$ initialized from $\QQ_0$ and defined for any measurable function $F : \D_{\Tf} \to \rset_+$ by
    \[\PE_{\QQ^k_j}[F((\bfX_{t})_{t\in\ccint{0,\Tf}})]=\PE_{\PP^{\genericq}}[F((\bfX_{t\wedge \sigma_j^k})_{t\in\ccint{0,\Tf}})\Z_{\Tf}^{\sigma^k_j}] \eqsp, \quad \ie \eqsp, \quad \QQ_j^k = \Z^{\sigma^k_j}_{\Tf} (\P^{\genericq})^{\sigma^k_j}\]
    is a probability measure on $\D_{\Tf}$ which satisfies
    \[\QQ^k_j \in \MP(\1_{\l[0,\sigma^k_j\r]}\rme^\chi {\genericq}, \QQ_0) \eqsp.\]
\end{lemma}
\begin{proof}[Proof of Lemma \ref{importantlemma}]
Fix $j, k \geq1$. We have
    \[\Z^{\sigma^k_j}_{t}=\exp \l( \int_{[0,t] \times \msx} \chi^k_j \rmd  \tilde N^{{\genericq}}_{\canoX} -\int_{\ccint{0,t} \times \msx} \varrho(\chi^k_j)\rmd\bar \rmn_{\bfX}^{\genericq}\r) \eqsp,\]
    where $\chi^k_j=\1_{\l[0,\sigma^k_j\r]}\chi$ is predictable since $\chi$ is predictable and $\1_{\l[0,\sigma^k_j\r]}$ is left continuous.
     For simplicity, we drop the subscripts and superscripts and write $\tilde\chi=\chi^k_j$ and $\tilde\Z_t=\Z^{\sigma^k_j}_{t}$ for the rest of the proof. From the definition of $\sigma^k_j$, we obtain that $\PP^{\genericq}$ \as,
    \begin{align}\label{inequality}
        \int_{\ccint{0,t} \times \msx} \varrho(\tilde\chi_s) \rmd\bar \rmn_{\bfX}^{\genericq}\leq k \eqsp, \quad \text{ and  } \tilde\chi_t(\canoX_{t-},\bfX_t) \in [-j,k] \eqsp, \text{ for any } t \in \ccint{0,\Tf} \eqsp.
    \end{align}
    First, we prove that $(\tilde \Z_t)$ is a $\PP^{\genericq}$-martingale. 
    From~\Cref{Zmartingale}, $(\tilde \Z_t)$ is a local martingale. Therefore, it is enough to show that for $t \in [0,\Tf]$,
    \begin{align}
        \mathbb E_{\PP^{\genericq}}[\tilde\Z_t^p] <\infty \eqsp, \quad \text{for some }p>1 \eqsp.
    \end{align}
    For $p >1$, we have
    \begin{align}
        \tilde \Z^p_t=\exp\l( p\int_{[0,t] \times \msx} \tilde \chi_s \rmd \tilde N^{{\genericq}}_{\canoX}-p\int_{\ccint{0,t} \times \msx}\varrho(\tilde \chi_s)\rmd\bar \rmn_{\bfX}^{\genericq}\r)\leq \exp\l(p\int_{[0,t] \times \msx} \tilde\chi_s \rmd \tilde N_{\canoX}^{\genericq}\r) \eqsp,
    \end{align}
    and
    \begin{align}
       \exp\l(p\int_{[0,t] \times \msx} \tilde \chi_s \rmd \tilde N_{\canoX}^{\genericq}-\int_{\ccint{0,t} \times \msx} \varrho(p\tilde \chi_s)\rmd\bar \rmn_{\bfX}^{\genericq}\r)\geq \left. \exp\l(p\int_{[0,t] \times \msx} \tilde \chi_s \rmd \tilde N_{\canoX}^{\genericq}\r)\middle /C(k,p,t) \right. \eqsp,
    \end{align}
   for some finite deterministic constant $0<C(k,p,t)<\infty$. Indeed, $\PP^{\genericq}$-\as, for any $s \in \ccint{0,\Tf}$, it holds that  $\varrho(p\tilde\chi_s)\leq \rme^{k(p-1)}(\varrho(\tilde \chi_s)+k+1)$ since  $ \tilde\chi_s \leq k$ and $p>1$. It yields that $\PP^{\genericq}$-\as, it holds
    \begin{align}
        \exp\l(\int_{\ccint{0,t} \times \msx} \varrho(p\tilde \chi_s)\rmd\bar \rmn_{\bfX}^{\genericq}\r)&\leq \exp\l(\rme^{k(p-1)}\int_{\ccint{0,t} \times \msx}  (\varrho(\tilde\chi_s)+k+1) \rmd\bar \rmn_{\bfX}^{\genericq}\r)\\
        &\overset{~\eqref{inequality}}{\leq} \exp \l(k\rme^{k(p-1)}+(k+1)\rme^{k(p-1)}\int_{[0,t] \times \msx}  1 \rmd \bar \rmn_{\bfX}^{\genericq} \r)\\
        &\leq C(k,p,t) <\infty \eqsp,
    \end{align}
    where the last inequality follows from the formula of $\bar \rmn_{\bfX}^{\genericq}$ given in~\eqref{eq:jump_kernel} and the fact that $\genericq$ satisfies \textbf{H2}.
    This implies $\PP^{\genericq}$-a.s.,
    \begin{align}\label{eq:z}
        \tilde \Z^p_t \leq \exp\l(p \int_{[0,t] \times \msx} \tilde\chi_s \rmd \tilde N^{{\genericq}}_{\canoX}\r)\leq C(k,p,t) \exp\l(p\int_{[0,t] \times \msx} \tilde \chi_s \rmd \tilde N_{\canoX}^{\genericq}-\int_{\ccint{0,t} \times \msx} \varrho(p \tilde \chi_s)\rmd\bar \rmn_{\bfX}^{\genericq}\r) \eqsp.
    \end{align}
    On the other hand, applying Lemma \ref{Zmartingale} for $(p\tilde \chi_t)_{t\in [0,\Tf]}$ yields that \[\exp\l(p\int_{[0,t] \times \msx} \tilde \chi_s \rmd \tilde N_{\canoX}^{\genericq} -\int_{\ccint{0,t} \times \msx} \varrho(p\tilde \chi_s)\rmd\bar \rmn_{\bfX}^{\genericq}\r)\] is a $\PP^{\genericq}$-supermartingale, and we get
    \begin{align}
        &\quad \mathbb E_{\PP^{\genericq}} \l[\exp\l(p\int_{[0,t] \times \msx} \tilde \chi_s \rmd \tilde N_{\canoX}^{\genericq}-\int_{\ccint{0,\Tf} \times \msx} \varrho(p \tilde \chi_s)\rmd\bar \rmn_{\bfX}^{\genericq}\r) \r] \\
        &\leq \E_{\PP^{\genericq}}\l[\exp\l(p\int_{[0,0] \times \msx} \tilde \chi_s \rmd \tilde N_{\canoX}^{\genericq}-\int_{[0,0] \times \msx} \varrho(p\tilde \chi_s)\rmd\bar \rmn_{\bfX}^{\genericq}\r) \r]=1 \eqsp.
    \end{align}
    Plugging this estimate into~\eqref{eq:z} gives
    \begin{align}
       \E_{\PP^{\genericq}} [\tilde \Z^p_t] \leq C(k,p,t)<\infty \eqsp, \quad \text{for any } t \in [0,\Tf] \eqsp,
    \end{align}
    which allow us to conclude that $(\tilde \Z_{t})_{t\in [0,\Tf]}$ is a $\PP^{\genericq}$-martingale \cite[see, $e.g.$,][]{zitkovic2015uniform}. Thereby  $\mathbb E_{\PP^{\genericq}}[\tilde \Z_{t}]=\E_{\PP^{\genericq}}[\tilde \Z_0]=1$ for any $t \in [0,\Tf]$ and it follows  $\QQ^k_j$ is a probability measure on $\D_{\Tf}$.

    Now, we show the second claim of~\Cref{importantlemma}:
    \[\QQ^k_j\in \MP(\1_{\l[0,\sigma^k_j\r]}\rme^\chi {\genericq}, \QQ_0) \eqsp.\]
    Let $\tau$ be a finitely valued stopping time which will be specified later, and for any function $f$,
    we denote
    \begin{align}
        \mathrm{F}_t:=\sum_{0 \leq s\leq t }\{f(\bfX_s)-f(\bfX_{s-})\} = \int_0^{t+}\int_{\msx} \{ f(x) - f(\bfX_{s-}) \} N^{{\genericq}}_{\canoX} (\rmd s \rmd x) \eqsp, \quad \text{for } t \in [0,\Tf] \eqsp.
    \end{align}
   Recall that by~\Cref{Zmartingale}, the martingale $(\tilde \Z_t)$ satisfies the followings for $\PP^{\genericq}$-a.s.,
    \begin{align}
       \rmd \tilde \Z_t=\1_{\l[0,\sigma^k_j\r]}(t)\tilde \Z_{t-}\int_{\msx} (\rme^{\tilde\chi_t (\canoX_{t-},x)}-1)  \tilde N^{{\genericq}}_{\canoX} (\rmd t \rmd x) \eqsp.
    \end{align}

    We have
     \begin{align}
       \mathbb E_{\QQ^k_j} \l[ \sum_{0 \leq s\leq t \wedge \tau} \{ f(\canoX_s)  - f(\canoX_{s-}) \} \middle|\mF_{0} \r] &= \E_{\QQ^k_j} \l[ \mathrm{F}_{t \wedge \tau} |\Fc_0 \r]\nonumber \\
       &=\mathbb E_{\PP^{\genericq}}\l[\tilde\Z_{t \wedge \tau \wedge \sigma^k_j}\mathrm{F}_{t \wedge \tau \wedge \sigma^k_j}-\tilde\Z_{0}\mathrm{F}_{0} \middle| \mF_{0} \r] \eqsp.\label{eq:MP_1}
    \end{align}
    Let us denote the two-dimensional process $(\mri_t)_{t \in [0,\Tf]} = (\mri^1_t, \mri^2_t)_{t \in [0,\Tf]}$, where
    \begin{equation}
        \mri^1_t := \mathrm{F}_t = \int_{[0,t] \times \msx} \underbrace{\l[ f(x) - f(\bfX_{s-}) \r]}_{=:v^1(s,x,\canoX)} \tilde N^{\genericq}_{\bfX} (\rmd s \rmd x) + \underbrace{\int_{[0,t] \times \msx} {\l[ f(x) - f(\bfX_{s-}) \r]} \bar \rmn_{\bfX}^{\genericq}(\rmd s \rmd x)}_{=:A^1_t}\eqsp,
    \end{equation}
    and
    \begin{equation}
        \mri^2_t = \tilde \Z_t = \int_{[0,t] \times \msx} \underbrace{  \tilde \Z_{s-} (\rme^{\tilde\chi_s (\canoX_{s-},x)}-1) }_{:=v^2 (s,x,\canoX)}  \tilde N^{\genericq}_{\canoX} (\rmd s \rmd x) \eqsp.
    \end{equation}
     Let $v = (v^1, v^2)$ and apply Itô's formula to the process $(\mri_t)_{t\in [0,\Tf]}$ using the function given by the product of the coordinates, treating $(A^1_t)_{t \in [0, \Tf]}$ as a continuous, finite variation process adapted to the filtration $(\Fc_t)$,
    \begin{align}
       &\quad \E_{\PP^{\genericq}} \l[\mathrm{F}_{t \wedge \tau \wedge \sigma^k_j} \tilde\Z_{t \wedge \tau \wedge \sigma^k_j} - \mathrm{F}_0 \tilde \Z_0 \middle| \Fc_0 \r] = \E_{\PP^{\genericq}} \l[\mri^1_{t \wedge \tau \wedge \sigma^k_j}\mri^2_{t \wedge \tau \wedge \sigma^k_j} - \mri^1_0 \mri^2_0 \middle| \Fc_0 \r]\\
        &=\E_{\PP^{\genericq}} \l[ \int_{[0,t \wedge \tau \wedge \sigma^k_j] \times \msx} \l\{(\mri^1_{s-}+v^1(s,x,\canoX))(\mri^2_{s-}+v^2(s,x,\canoX)) - \mri^1_{s-}\mri^2_{s-} \r\} \tilde N^{\genericq}_{\bfX}(\rmd s \rmd x) \middle| \Fc_0 \r] \\
        &+ \E_{\PP^{\genericq}} \Bigg[\int_{[0,t\wedge \tau \wedge \sigma^k_j]\times \msx} \{(\mri^1_s+v^1(s,x,\canoX))(\mri^2_s+v^2(s,x,\canoX)) \\
        &\qquad \qquad-\mri^1_s \mri^2_s - \mri^2_s v^1(s,x,\canoX) - \mri^1_s v^2(s,x,\canoX) \} \bar\rmn^{{\genericq}}_{\canoX} (\rmd s\rmd x) \Big| \Fc_0 \Bigg] \\
        &\qquad \qquad+ \E_{\PP^{\genericq}} \l[\int_{[0,t\wedge \tau \wedge \sigma^k_j] \times \msx} \mri^2_s v^1(s,x,\canoX) \bar\rmn^{{\genericq}}_{\canoX} (\rmd s\rmd x)  \middle| \Fc_0\r] \\
        &= \E_{\PP^{\genericq}} \l[ \int_{[0,t\wedge \tau \wedge \sigma^k_j] \times \msx} \l\{ \mri^2_s v^1(s,x,\canoX) + v^1(s,x,\canoX)v^2(s,x,\canoX) \r\} \bar\rmn^{{\genericq}}_{\canoX} (\rmd s \rmd x) \middle| \Fc_0\r] \\
        &\qquad \qquad \quad\text{(as $\tilde N_{\bfX}^{\genericq}$ is a $\PP^{\genericq}$-martingale)} \\
        &= \E_{\PP^{\genericq}} \l[ \int_{[0,t\wedge \tau \wedge \sigma^k_j] \times \msx} \l\{ \tilde \Z_s v^1(s,x,\canoX) + v^1(s,x,\canoX)\tilde \Z_{s-} \l( \rme^{\tilde \chi_s(\canoX_{s-},x)}-1 \r) \r\} \bar\rmn^{{\genericq}}_{\canoX} (\rmd s \rmd x) \middle| \Fc_0\r] \eqsp,
    \end{align}
    where we reduce the stochastic integral w.r.t. $\tilde N_{\bfX}^{\genericq}$ as it is a local $\PP^{\genericq}$-martingale, since the integrand is $\R$-valued predictable process and $\tilde N_{\canoX}^{\genericq}$ forms a martingale under $\PP^{\genericq}$.
    Since $\tilde \Z_s = \tilde \Z_{s-}$ for Lebesgue almost every $s \in [0,t \wedge \tau \wedge \sigma^k_j]$ \cite[Proposition 2.1]{mozumder2009some} and $\bar \rmn_{\bfX}^{\genericq}$ is atomless in time, the calculation follows
    \begin{align}
         &\quad\E_{\PP^{\genericq}} \l[\mathrm{F}_{t \wedge \tau \wedge \sigma^k_j} \tilde\Z_{t \wedge \tau \wedge \sigma^k_j} - \mathrm{F}_0 \tilde \Z_0 \middle| \Fc_0 \r] \nonumber \\
         &= \E_{\PP^{\genericq}} \l[ \int_{[0,t\wedge \tau \wedge \sigma^k_j] \times \msx}  \tilde \Z_{s} v^1(s,x,\canoX) \l(1 + \rme^{\tilde \chi_s(\canoX_{s-},x)}-1 \r)  \bar\rmn^{{\genericq}}_{\canoX} (\rmd s \rmd x) \middle| \Fc_0\r] \nonumber \\
         &= \E_{\PP^{\genericq}} \l[ \int_{[0,t\wedge \tau \wedge \sigma^k_j] \times \msx}  \tilde \Z_{s} v^1(s,x,\canoX) \rme^{\tilde \chi_s(\canoX_{s-},x)}  \bar\rmn^{{\genericq}}_{\canoX} (\rmd s \rmd x) \middle| \Fc_0\r] \eqsp. \label{eq:MP_2}
    \end{align}


Denote $\mathrm{G}_t:= \int_{[0,t] \times \msx}  v^1(s,x,\canoX) \rme^{\tilde \chi_s(\canoX_{s-}, x)}\bar \rmn_{\bfX}^{\genericq}(\rmd s \rmd x) $. Applying It\^o's formula for the process $(\mathrm{G}_t, \tilde \Z_t)$ analogously as argued before, we obtain that
\begin{align}
   &\quad \E_{\PP^{\genericq}} \l[ \tilde \Z_{t \wedge \tau \wedge \sigma^k_j} \mathrm{G}_{t \wedge \tau \wedge \sigma^k_j} - \tilde \Z_0 \mathrm{G}_0 \middle| \Fc_0\r] \nonumber\\
    &= \E_{\PP^{\genericq}} \l[ \int_{[0,t \wedge \tau \wedge \sigma^k_j]\times \msx} \l\{\mathrm{G}_{s-}(\tilde \Z_{s-}+v^2(s,x,\canoX)) - \mathrm{G}_{s-}\tilde \Z_{s-} \r\} \tilde N^{\genericq}_{\bfX} (\rmd s \rmd x) \middle| \Fc_0 \r] \nonumber \\
    &+ \E_{\PP^{\genericq}} \l[\int_{[0,t \wedge \tau \wedge \sigma^k_j] \times \msx} \l\{ (\mathrm{G}_s (\tilde \Z_s+ v^2(s,x,\canoX)) - \mathrm{G}_s \tilde \Z_s - \mathrm{G}_s v^2(s,x,\canoX)  \r\} \bar \rmn_{\bfX}^{\genericq} (\rmd s \rmd x) \middle| \Fc_0 \r] \nonumber\\
    &+ \E_{\PP^{\genericq}} \l[\int_{[0,t \wedge \tau \wedge \sigma^k_j] \times \msx} \tilde \Z_s v^1(s,x,\canoX) \rme^{\tilde\chi_s(\canoX_{s-},x)} \bar \rmn_{\bfX}^{\genericq}(\rmd s \rmd x) \middle| \Fc_0 \r] \nonumber \\
    &= \E_{\PP^{\genericq}} \l[\int_{[0,t \wedge \tau \wedge \sigma^k_j] \times \msx} \tilde \Z_s v^1(s,x,\canoX) \rme^{\tilde\chi_s(\canoX_{s-},x)} \bar \rmn_{\bfX}^{\genericq}(\rmd s \rmd x) \middle| \Fc_0 \r] \eqsp, \label{eq:MP_3}
\end{align}
as the stochastic integral w.r.t. $\tilde N_{\bfX}^{\genericq}$ is a local $\PP^{\genericq}$-martingale, since the integrand is $\R$-valued predictable process and $\tilde N_{\canoX}^{\genericq}$ forms a martingale under $\PP^{\genericq}$.
Combining~\eqref{eq:MP_1},~\eqref{eq:MP_2} and~\eqref{eq:MP_3} implies
\begin{align}
    &\quad \mathbb E_{\QQ^k_j} \l[ \sum_{0 \leq s\leq t \wedge \tau} \{ f(\canoX_s)  - f(\canoX_{s-}) \} \middle|\mF_{0} \r] \\
    &= \E_{\PP^{\genericq}} \l[ \tilde \Z_{t \wedge \tau \wedge \sigma^k_j} \mathrm{G}_{t \wedge \tau \wedge \sigma^k_j}\middle| \Fc_0\r]= \E_{\QQ^k_j} \l[ \mathrm{G}_{t \wedge \tau } \middle| \Fc_0\r] \quad \text{(since $\mathrm{G}_0 =0$)}\\
     &=\mathbb E_{\QQ^k_j}\l[\int_{[0,t \wedge \tau ] \times \msx}\{ f(x) - f(\canoX_{s-}) \} \rme^{\tilde\chi_s(\canoX_{s-},x)}\bar \rmn_{\bfX}^{\genericq}(\rmd s \rmd x)\middle| \mF_{0} \r] \\
        &= \mathbb E_{\QQ^k_j}\l[\int_{[0,t \wedge \tau ] \times \msx} \{ f(x) - f(\canoX_{s-}) \}\rme^{\tilde\chi_s(\canoX_{s-},x)} \sum_{y \in \msx} \1_{\canoX_{s-} \neq y }  {\genericq}_s(\canoX_{s-}, y) \updelta_y(\rmd x)  \rmd s \middle| \mF_{0} \r] \eqsp.
\end{align}

Denote by $\bar \rmn^{\rme^{\tilde\chi} \genericq}_{\bfX}$ the corresponding jump kernel for $ (\canoX_t)_{t \in [0,\Tf]} \in \D_{\Tf}$ and $(t,x) \in [0,\Tf] \times \msx$,
\begin{align}
     \bar \rmn^{\rme^{\tilde\chi} \genericq}_{\bfX} ((\canoX_t)_{t\in [0,\Tf]},\rmd t \rmd x) &:= \sum_{y \in \msx} \1_{\canoX_{t-} \neq y }   (\rme^{\tilde\chi} \genericq)_t(\canoX_{t-},y) \updelta_y(\rmd x) \rmd t \\
     &= \sum_{y \in \msx} \1_{\canoX_{t-} \neq y }   \rme^{\tilde\chi_t(\canoX_{t-},y)} \genericq_t(\canoX_{t-},y) \updelta_y(\rmd x) \rmd t \eqsp,
\end{align}
then the previous equation rewrites
\begin{align}
    \mathbb E_{\QQ^k_j} \l[ f(\canoX_{t \wedge \tau}) - f(\canoX_0) \middle|\mF_{0} \r] &= \mathbb E_{\QQ^k_j}\l[\int_{[0,t \wedge \tau] \times \msx} \{ f(x) - f(\canoX_{s-}) \}  \bar \rmn^{\rme^{\tilde\chi} \genericq}_{\bfX} (\rmd s \rmd x)   \middle| \mF_{0} \r] \\
    &= \E_{\QQ^k_j} \l[\int_{[0,t \wedge \tau]}(\rme^{\tilde\chi} {\genericq})(s) f (\canoX_{s-}) \rmd s \middle| \mF_{0} \r]\eqsp.
\end{align}
    Choosing $\tau$ such that the above terms are meaningful, we conclude that $\QQ^k_j \in \MP(\rme^{\chi^k_j}{\genericq})$ and finish the proof.
\end{proof}


\begin{proof}[Proof of~\Cref{girsanovtheorem}]
  This proof is an adaptation of Theorem 2.6 in \cite{leonard2012girsanov} based on technical lemmas provided above applying on the reference measure $\P^{\canoq} \in \MP (\canoq, \P^{\canoq}_0)$.
  We first show the formulation of the Radon-Nikodym density $\rmd \Puq /\rmd \Pq$ when $\Puq \sim \Pq$. Define the stopping time $\tau_j^k$ as
\begin{align}
    \tau_j^k=\inf\{t \in \ccint{0,\Tf};\int_{\ccint{0,t} \times \msx}\varrho(\log u_s(\canoX_{s-}, x)\bar \rmn^{\canoq}_{\bfX}(\rmd x \rmd s))\geq k \text{ or } \log u_t(\canoX_{t-}, \canoX_t )\notin [-j,k]\} 
\end{align}
    which coincides with the stopping time $\sigma^k_j$ when $\chi = \log u$. Denote $u^{\sigma^k_j}:= \1_{[0,\sigma^k_j]} u$ and for simplicity, we write $u = u^{\sigma^k_j}$.
    By conditioning w.r.t. $\canoX_0$, we can assume without loss of generality that $\Pq_0=\Puq_0$, \ie, ${\rmd \Puq_0}/{\rmd \Pq_0}(\canoX_0)=1$.

    Applying Lemma \ref{importantlemma} for $\Puq \in \MP(u\canoq)$ and $\chi=-\log u$, we have
    \begin{align}\label{Qkj}
        \QQ^{\tau_j^k}
        := \Z^{-\log u, u \canoq}_{\Tf} (\Puq)^{\tau_j^k}
        &\in \MP(\1_{[0,\tau_j^k]}\rme^{-\log u}u \canoq)= \MP(\1_{[0,\tau_j^k]} \canoq ) \eqsp.
    \end{align}

    Furthermore, $\QQ^{\tau^k_j}_0 = (\Puq)^{\tau^k_j}_0 = (\Pq)^{\tau^k_j}_0$, which combined with the equation~\eqref{Qkj} imply $ \QQ^{\tau^k_j}=(\Pq)^{\tau^k_j}$ thanks to the uniqueness of $\MP(\1_{[0,\tau_j^k]}\canoq, \Pq_0)$
     Now, applying again~\Cref{importantlemma} for $\Pq \in \MP(\canoq)$ and $\chi = \log u$ yields
    \begin{align}
        \tilde{\P}^{\tau_j^k}
        := \Z^{\log u, \canoq}_{\Tf}(\Pq)^{\tau_j^k}
        \in \MP(\1_{[0,\tau^k_j]}\rme^{\log u} \canoq)= \MP(\1_{[0,\tau^k_j]} u \canoq)\eqsp.
    \end{align}
    Next reapplying~\Cref{importantlemma} with $(\tilde{\P^{u\canoq}})^{\tau^k_j}\in \MP(\1_{[0,\tau^k_j]}u \canoq)$ and $\chi = -\log u$ implies
    \begin{align}
        \tilde \QQ^{\tau^k_j}
        :=\Z^{-\log u, u\canoq}_{\Tf} \tilde \P^{\tau_j^k}
        \in  \MP(\1_{[0,\tau_j^k]}\rme^{-\log u} u \canoq)= \MP(\1_{[0,\tau_j^k]} \canoq ) \eqsp.
    \end{align}
    Argue as before, the previous equation together with the initial condition $\tilde \QQ^{\tau^k_j}_0 = \tilde \P^{\tau^k_j}_0 = \Pq_0$ yield that $\tilde \QQ^{\tau^k_j}=\overrightarrow{R}^{\tau^k_j}$ thanks to the uniqueness of $\MP(\1_{[0,\tau_j^k]}\canoq, \Pq_0)$. Combining it with $\QQ^{\tau^k_j}=(\Pq)^{\tau^k_j}$ implies
    \[\QQ^{\tau^k_j}=\tilde \QQ^{\tau^k_j} \eqsp,\]
    which means that
    \begin{align}\label{eq:uniq_p}
    \Z^{-\log u, u\canoq}_{\Tf} \P^{\tau_j^k} = \Z^{-\log u, u\canoq}_{\Tf} \tilde\P^{\tau_j^k} \eqsp.
    \end{align}
     Now observe that $\Z_{\Tf}^{-\log u, u\canoq} >0$, therefore, equation~\eqref{eq:uniq_p} yields
    $(\Puq)^{\tau^k_j}=\tilde \P^{\tau^k_j}$. Hence $(\Puq)^{\tau_j^k} = \Z_{\Tf}^{\log u, \canoq} (\Pq)^{\tau^k_j}$, \ie,
    \begin{align}
        &\1_{[0,\tau_j^k\wedge \Tf]}\dfrac{\rmd\Puq}{\rmd\Pq}((\canoX_t)_{t\in [0,\Tf]})=\1_{[0,\tau_j^k\wedge \Tf]}\dfrac{\rmd\Puq_0}{\rmd\Pq_0}(\canoX_0)\\
        &\qquad \exp \l(\int_{[0,\tau^k_j\wedge \Tf] \times \msx}(\1_{[0,\tau^k_j\wedge \Tf]}\log u) \rmd \tilde N^{\canoq}_{\canoX}-\int_{[0,\tau^k_j\wedge \Tf] \times \msx}\varrho(\log u)\rmd\bar \rmn^{\canoq}_{\bfX} \r)\eqsp.
    \end{align}
    Letting $k$ and $j$ tend to infinity and noting that $\tau:=\lim_{k,j\to \infty}\tau_j^k=\infty$, we get $\Puq$-a.s. under condition \eqref{eq:girsanov_integrability_cond} that
\begin{align}
    &\dfrac{\rmd\Puq}{\rmd\Pq}((\canoX_t)_{t\in [0,\Tf]}) =\dfrac{\rmd\Puq_0}{\rmd\Pq_0}(\canoX_0)\\
    &\qquad \exp\l( \int_{\ccint{0,\Tf} \times \msx}\log u_t(\canoX_{t-},x) \tilde N^{\canoq}_{\canoX}(\rmd t \rmd x)-\int_{\ccint{0,\Tf} \times \msx}\varrho(\log u_t(\canoX_{t-},x))\bar \rmn^{\canoq}_{\bfX}(\rmd t \rmd x) \r) \eqsp.
\end{align}
    We now extend the result above to the case when $\Puq$ might not be equivalent to $\Pq$. The idea is to approximate $\Puq$ by a sequence $(\Puq_n)$, which satisfies $\Puq_n \sim \Pq$ for all $n\geq1$. Denoting
    \begin{equation}\label{eq:14}
        \Puq_n=\l(1-\dfrac{1}{n}\r)\Puq+\dfrac{\Pq}{n}  \quad \text{ for } n \geq 1 \eqsp,
    \end{equation}
    we have $\Puq_n \sim \Pq$ and $\lim_{n\to\infty} \KL(\Puq|\Puq_n)=0$. For simplicity, we write $\chi=\log u$ and $\chi^n=\log u^n$, which are well-defined $\Puq$-a.s. From the variational representation of the $\KL$ divergence, using $\Puq\in \MP(u \canoq)$ combined with Lemma \ref{Zmartingale}, we obtain
    \begin{align}
        \KL(\Puq|\Puq_n)&\geq\mathbb E_{\Puq} \l[\int_{\ccint{0,\Tf} \times \msx}(\chi-\chi^n) \rmd \tilde N^{u^n \canoq}_{\canoX}-\int_{\ccint{0,\Tf} \times \msx}\varrho(\chi-\chi^n) \rmd(u^n\bar \rmn^{\canoq}_{\bfX}) \r] \eqsp.
    \end{align}
    By definition, we have \[\tilde N^{u^n \canoq}_{\canoX}=N^{\canoq}_{\bfX}-u^n\bar \rmn^{\canoq}_{\bfX} = N^{\canoq}_{\bfX}-u\bar \rmn^{\canoq}_{\bfX} +(u-u^n)\bar \rmn^{\canoq}_{\bfX} =\tilde N^{u\canoq}_{\canoX}+(u-u^n)\bar \rmn^{\canoq}_{\bfX} \eqsp,\]
    which yields
    \begin{align}
        &\quad\KL(\Puq|\Puq_n)\\
        &\geq \mathbb E_{\Puq} \l[\int_{\ccint{0,\Tf} \times \msx} (\chi-\chi^n) \rmd(\tilde N^{u\canoq}_{\canoX}+\bar \rmn^{\canoq}_{\bfX} (u-u^n))
        -
        \int_{\ccint{0,\Tf} \times \msx}\l(\dfrac{u}{ u^n}-\log \dfrac{u}{u^n}-1 \r)u^n \rmd\bar \rmn^{\canoq}_{\bfX} \r]\\
        &=\mathbb E_{\Puq} \l[ \int_{\ccint{0,\Tf} \times \msx} (\chi-\chi^n)\rmd \tilde N^{u\canoq}_{\canoX}+\int_{\ccint{0,\Tf} \times \msx} u\log \dfrac{u}{u^n} \rmd\bar \rmn^{\canoq}_{\bfX} -\int_{\ccint{0,\Tf} \times \msx}\l(\dfrac{u}{u^n}-1 \r)u^n \rmd\bar \rmn^{\canoq}_{\bfX} \r]\eqsp.
    \end{align}
    Since $\Puq \in \MP(u\canoq)$, we deduce that the stochastic integral $\int_{\ccint{0,\Tf} \times \msx}(\chi-\chi^n)\rmd \tilde N^{u\canoq}_{\canoX}$ is a local $\P$-martingale. Therefore,
    \begin{align}
        \KL( \Puq|\Puq_n)&\geq \mathbb E_{\Puq} \l[ \int_{\ccint{0,\Tf} \times \msx} \l( u^n-u-u\log \dfrac{u^n}{u} \r)\rmd\bar \rmn^{\canoq}_{\bfX} \r]\\
        &=\mathbb E_{\Puq} \l[ \int_{\ccint{0,\Tf} \times \msx} \l( \dfrac{u^n}{u}-\log \dfrac{u^n}{u}-1 \r)u\rmd\bar \rmn^{\canoq}_{\bfX} \r]\\
        &=\mathbb E_{\Puq} \l[ \int_{\ccint{0,\Tf} \times \msx}\varrho(\chi^n-\chi)\rmd(u\bar \rmn^{\canoq}_{\bfX}) \r] \eqsp.
    \end{align}
    Since $\lim_{n \to \infty}\KL(\P|\P_n)=0$, we obtain
    \begin{align}\label{lim}
        \lim_{n\to \infty}\mathbb E_{\Puq} \l[ \int_{\ccint{0,\Tf} \times \msx}\varrho(\chi^n-\chi)\rmd (u\bar \rmn^{\canoq}_{\bfX}) \r]=0 \eqsp.
    \end{align}

    On the other hand, the fact that $\Puq_n\sim \Pq$ yields
    \begin{align}\label{eq:12}
        \dfrac{\rmd \Puq_n}{\rmd \Pq} ((\canoX_t)_{t\in [0,\Tf]})=\dfrac{ \rmd \Puq_{n,0}}{\rmd \Pq_0}(\canoX_0)\exp\l(\int_{[0,\Tf] \times \msx} \chi^n \rmd \tilde N^{\canoq}_{\canoX}- \int_{\ccint{0,\Tf} \times \msx}\varrho(\chi^n)\rmd\bar \rmn^{\canoq}_{\bfX} \r) \eqsp.
    \end{align}

To obtain the desired expression for the Radon–Nikodym density $\frac{\mathrm{d} \Puq}{\mathrm{d} \Pq}$, we represent it as
        \begin{align}
            &\quad\frac{\rmd \Puq}{\rmd \Pq} ((\canoX_t)_{t\in [0,\Tf]}) \notag\\
            &= \l(\frac{\rmd \Puq}{ \rmd \Puq_n}. \frac{\rmd \Puq_n}{\rmd \Pq}\r) ((\canoX_t)_{t\in [0,\Tf]})\nonumber \\
            &\overset{~\eqref{eq:12}}{=} \l(\frac{\rmd \P}{ \rmd \P_n}. \frac{ \rmd \P_{n,0}}{\rmd \Pq_0}\r)(\canoX_0)\exp\l(\int_{[0,\Tf] \times \msx} \chi^n \rmd \tilde N^{\canoq}_{\canoX}- \int_{\ccint{0,\Tf} \times \msx}\varrho(\chi^n)\rmd\bar \rmn^{\canoq}_{\bfX}\r) \nonumber \\
            &= \l(\frac{\rmd \Puq}{ \rmd \Puq_n} . \frac{\rmd \Puq_{n,0}}{\rmd \Puq_0}\r) (\canoX_0) \frac{\rmd \Puq_0}{\rmd \Pq_0} (\canoX_0) \exp\l(\int_{[0,\Tf] \times \msx} \chi \rmd \tilde N^{\canoq}_{\canoX}- \int_{\ccint{0,\Tf} \times \msx}\varrho(\chi)\rmd\bar \rmn^{\canoq}_{\bfX}\r) \nonumber \\
            & \exp\l(\int_{\ccint{0,\Tf} \times \msx}(\chi^n -\chi) \rmd \tilde N^{\canoq}_{\canoX}- \int_{\ccint{0,\Tf} \times \msx} (\varrho(\chi^n) -\varrho (\chi))\rmd\bar \rmn^{\canoq}_{\bfX}\r) \eqsp, \Puq\text{-a.s.} \label{eq:radon_dens}
        \end{align}

  The last part can be rewritten as follows using the relation $\tilde N_{\canoX}^{\canoq} = \tilde N_{\canoX}^{u\canoq}+ (u-1)\bar \rmn^{\canoq}_{\bfX} $,
  \begin{align}
      &\quad\exp\l(\int_{\ccint{0,\Tf} \times \msx}(\chi^n -\chi) \rmd \tilde N^{\canoq}_{\canoX}- \int_{\ccint{0,\Tf} \times \msx} (\varrho(\chi^n) -\varrho (\chi))\rmd\bar \rmn^{\canoq}_{\bfX}\r) \\
            &= \exp\l(\int_{\ccint{0,\Tf} \times \msx}(\chi^n - \chi )  \rmd \tilde N_{\bfX}^{u\canoq} - \int_{\ccint{0,\Tf} \times \msx} \varrho (\chi - \chi^n) \rmd (u \bar \rmn^{\canoq}_{\bfX} ) \r)
  \end{align}

We first handle the second integral above using~\eqref{lim} and the fact that $\varrho$ is a non-negative function to obtain
  \begin{equation}
      \E_{\Puq} \l[ \l|\int_{\ccint{0,\Tf} \times \msx} \varrho (\chi - \chi^n) \rmd (u \bar \rmn^{\canoq}_{\bfX}) \r| \r] \xrightarrow{n \to \infty} 0 \eqsp.
  \end{equation}
  This together with Markov's inequality lead to
  \begin{equation}\label{eq:p_conv_1}
       \int_{\ccint{0,\Tf} \times \msx} \varrho (\chi - \chi^n) \rmd (u \bar \rmn^{\canoq}_{\bfX}) \xrightarrow{\Puq} 0  \eqsp,
  \end{equation}
  and therefore, from \cite[Proposition 1.4]{meliotconvergence}, there is a subsequence $\chi^{n_k}$ such that
  \begin{equation}\label{eq:lim_1}
       \int_{\ccint{0,\Tf} \times \msx} \varrho (\chi - \chi^{n_k}) \rmd (u \bar \rmn^{\canoq}_{\bfX} ) \xrightarrow{n \to \infty} 0 \eqsp, \quad \Puq\text{-a.s.}
  \end{equation}

  Furthermore, recall that $\varrho(a)=\rme^a - a -1 \geq a^2/2$, hence~\eqref{lim} can be used to control the stochastic integral w.r.t. the $\Puq$-martingale $\tilde N_{\bfX}^{u\canoq}$ as follows
  \begin{align}
      \E_{\Puq} \l[ \l(\int_{\ccint{0,\Tf} \times \msx}(\chi^n - \chi )  \rmd \tilde N_{\bfX}^{u\canoq} \r)^2 \r] &\overset{\text{It\^o's isometry}}{=} \E_{\Puq} \l[ \int_{\ccint{0,\Tf} \times \msx}(\chi^n - \chi )^2  \rmd ( u\bar \rmn^{\canoq}_{\bfX} ) \r] \\
      &\quad\leq 2 \E_{\Puq} \l[ \int_{\ccint{0,\Tf} \times \msx} \varrho (\chi - \chi^n) \rmd (u \bar \rmn^{\canoq}_{\bfX} )  \r] \xrightarrow{n \to \infty} 0 \eqsp.
  \end{align}
This, along with Markov’s inequality, results in
\begin{equation}
    \int_{\ccint{0,\Tf} \times \msx}(\chi^n - \chi )  \rmd \tilde N_{\bfX}^{u\canoq}  \xrightarrow{\Puq} 0 \eqsp.
\end{equation}
Combining this with~\eqref{eq:p_conv_1} implies that
\begin{equation}
    \int_{\ccint{0,\Tf} \times \msx}(\chi^n - \chi )  \rmd \tilde N_{\bfX}^{u\canoq} - \int_{\ccint{0,\Tf} \times \msx} \varrho (\chi - \chi^n) \rmd (u\bar \rmn^{\canoq}_{\bfX} ) \xrightarrow{\Puq} 0 \eqsp.
\end{equation}

As a consequence, \cite[Proposition 1.4]{meliotconvergence} asserts that there is a subsequence $(\chi^{n_k})$ such that
\begin{equation}
    \l[\int_{\ccint{0,\Tf} \times \msx}(\chi^{n_k} - \chi )  \rmd \tilde N_{\bfX}^{u\canoq} - \int_{\ccint{0,\Tf} \times \msx} \varrho (\chi - \chi^{n_k}) \rmd (u \bar \rmn^{\canoq}_{\bfX} ) \r] \xrightarrow{k \to \infty} 0 \eqsp, \quad \Puq\text{-a.s.}
\end{equation}
It helps interpreting~\eqref{eq:radon_dens} as
\begin{align}
    & \frac{\rmd \Puq}{\rmd \Pq} ((\canoX_t)_{t\in [0,\Tf]}) \\
    &= \l(\frac{\rmd \Puq}{ \rmd \Puq_{n_k}} . \frac{\rmd \Puq_{n_k,0}}{\rmd \Puq_0} \r)(\canoX_0) \frac{\rmd \Puq_0}{\rmd \Pq_0} (\canoX_0) \exp\l(\int_{[0,\Tf] \times \msx} \chi \rmd \tilde N^{\canoq}_{\canoX}- \int_{\ccint{0,\Tf} \times \msx}\varrho(\chi)\rmd\bar \rmn^{\canoq}_{\bfX} \r)\\
    &\qquad \qquad  \exp\l(\int_{\ccint{0,\Tf} \times \msx}(\chi^{n_k} -\chi) \rmd \tilde N^{\canoq}_{\canoX}- \int_{\ccint{0,\Tf} \times \msx} (\varrho(\chi^{n_k}) -\varrho (\chi))\rmd\bar \rmn^{\canoq}_{\bfX} \r) \\
    & \xrightarrow[~\eqref{eq:14}]{k \to \infty} \frac{\rmd \Puq_0}{\rmd \Pq_0} (\canoX_0) \exp\l(\int_{[0,\Tf] \times \msx} \chi \rmd \tilde N^{\canoq}_{\canoX}- \int_{\ccint{0,\Tf} \times \msx}\varrho(\chi)\rmd\bar \rmn^{\canoq}_{\bfX} \r)\eqsp, \quad \Puq\text{-a.s.}
\end{align}
Replacing $\chi=\log u$, we arrive at our desired claim $\Puq$-a.s.
    \begin{equation}
        \dfrac{\rmd\Puq}{\rmd\Pq} ((\canoX_t)_{t\in [0,\Tf]})=\dfrac{\rmd\Puq_0}{\rmd\Pq_0}(\canoX_0)\exp\l( \int_{\ccint{0,\Tf} \times \msx} \log u \rmd \tilde N^{\canoq}_{\canoX}-\int_{\ccint{0,\Tf} \times \msx}\varrho(\log u)\rmd \bar \rmn^{\canoq}_{\bfX} \r) \eqsp.
    \end{equation}
Consequently, the $\KL$ divergence reads as
\begin{align}
    \KL(\Puq|\Pq) &=\KL(\Puq_0|\Pq_0)+\mathbb E_{\Puq} \Bigg[ \int_{\ccint{0,\Tf} \times \msx}\log u_t(\canoX_{t-},x) \tilde N^{\canoq}_{\canoX}(\rmd t \rmd x) \\
   & \qquad \qquad-\int_{\ccint{0,\Tf} \times \msx}\varrho(\log u_t(\canoX_{t-},x)) \bar \rmn^{\canoq}_{\bfX} (\rmd t \rmd x) \Bigg] \eqsp.
\end{align}
Using the identity $\tilde N_{\canoX}^{\canoq} = \tilde N_{\canoX}^{u\canoq} + (u-1)\bar \rmn^{\canoq}_{\bfX}$ and relying on the fact that $\tilde N_{\canoX}^{u\canoq}$ is a martingale under $\Puq$, we deduce that
\begin{align}
    \KL(\Puq|\Pq)&=\KL(\Puq_0|\Pq_0)+\mathbb E_{\Puq} \l[ \int_{\ccint{0,\Tf} \times \msx}(u-1)\log u \rmd\bar \rmn^{\canoq} -\int_{\ccint{0,\Tf} \times \msx}\varrho(\log u) \rmd\bar \rmn^{\canoq}_{\bfX} \r]\\
    &=\KL(\Puq_0|\Pq_0)+\mathbb E_{\Puq} \l[ \int_{\ccint{0,\Tf} \times \msx}[(u-1)\log u-u+\log u+1]\rmd\bar \rmn^{\canoq}_{\bfX} \r]\\
    &=\KL(\Puq_0|\Pq_0)+\mathbb E_{\Puq} \l[\int_{\ccint{0,\Tf} \times \msx}\mathbf{h}(u_t(\canoX_{t},x)) \bar \rmn^{\canoq}_{\bfX} (\rmd t \rmd x) \r]\eqsp,
\end{align}
since $\canoX_t = \canoX_{t-}$ for Lebesgue almost every $t\in (0,\Tf]$ and $\bar \rmn^{\canoq}_{\bfX}$ is atomless in time, where $\mathbf{h}(a)=a\log a-a+1$ for $a>0$.
In addition, we can simplify the $\KL$ expression above by replacing
\[\bar \rmn^{\canoq}_{\bfX} (\rmd t \rmd x) = \sum_{y\in \msx} \1_{ \{\canoX_{t-} \neq y \}} \canoq (\canoX_{t-},y) \updelta_{y}(\rmd x)  \rmd t\]
 to arrive at
\begin{align}
    \KL(\Puq|\Pq) &= \KL(\Puq_0|\Pq_0)+\mathbb \E_{\Puq} \l[\int_{\ccint{0,\Tf}} \sum_{x\in \msx} \mathbf{h}(u_t(\canoX_{t},x)) \1_{ \canoX_{t} \neq x }  \canoq (\canoX_{t}, x) \rmd t \r] \eqsp.
\end{align}
The proof of~\Cref{girsanovtheorem} is then complete.
\end{proof}

\bibliography{reference}

\end{document}